%% file: main.tex
\documentclass[10pt]{article}

\usepackage{mystyle}
\usepackage{colortbl}
\RequirePackage[top=2.25cm, bottom=2.5cm, left=2.5cm, right=2.5cm, columnsep=0.65cm]{geometry}

\title{\LARGE Policy Iteration Is Not Polynomial for Deterministic Markov Decision Processes: The Price of Algorithmic Anarchy}
\author{
  Han Zhong\thanks{Shanghai Jiao Tong University. Email: \texttt{han.zhong@sjtu.edu.cn}.}
  \and
  Yinyu Ye\thanks{Shanghai Jiao Tong University, SIMIS, and Stanford University. Email: \texttt{yinyu-ye@stanford.edu}.}
}
\date{}

\begin{document}
\maketitle

\input{paper}

\subsection*{AI use statement}
We used generative AI to assist with some derivations, literature searches, proof checking,
language editing, and \LaTeX{} formatting. The authors carefully reviewed all AI-assisted material and
independently verified the mathematical arguments. The authors take full responsibility for this paper.

\bibliographystyle{ims}
\bibliography{references}

\clearpage
\appendix
\input{appendix}

\end{document}

%% file: paper.tex
\usetikzlibrary{arrows.meta}
\tikzset{
  paper state/.style={circle,draw=paperSlate,fill=paperInk!2,text=black,
    minimum size=30pt,inner sep=1pt,outer sep=0.6pt,line width=0.65pt},
  paper arrow/.style={-{Latex[length=3.8pt,width=2.8pt]},line width=0.85pt,
    line cap=round,line join=round,shorten <=0.5pt,shorten >=0.7pt},
  paper label/.style={fill=white,inner sep=1pt,font=\small,text=black},
  paper panel/.style={font=\small\bfseries,text=paperInk,inner sep=0pt}
}

\begin{abstract}
We establish exponential iteration lower bounds in the number of states
for Howard's policy iteration on deterministic discounted and
average-reward Markov decision processes (MDPs), with a constant
number of actions per state. This rules out polynomial-time complexity of
Howard's policy iteration in both settings.
We further prove that a Dantzig counterpart of Howard's policy
iteration runs in strongly polynomial time on deterministic average-reward
MDPs. Together
with the known strongly polynomial bound for Dantzig's rule on
deterministic discounted MDPs, these results
establish an exponential separation between the two algorithms in both
settings. This separation between Howard's decentralized and simultaneous
selfish improvements and Dantzig's coordinated selection of a single
action with the largest improvement across all states reveals a ``price'' of
algorithmic anarchy.
\end{abstract}

\paragraph{Keywords.} deterministic Markov decision processes;
policy iteration; strongly polynomial algorithms.

\section{Introduction}
\label{sec:introduction}
Markov decision processes (MDPs) provide a standard framework for
sequential decision-making and reinforcement learning
\citep{Puterman1994,SuttonBarto2018}.
The objective is to maximize the expected discounted sum of rewards in
discounted MDPs or the expected long-run average reward (gain) in
average-reward MDPs. Finding an optimal policy efficiently is a central
algorithmic problem in both settings.

Policy iteration, also known as Howard's policy iteration
\citep{Howard1960}, is a foundational algorithm
for solving MDPs.
At each iteration, it evaluates the current policy and updates the policy
simultaneously across states, with each state independently selecting a greedy action
using the resulting values.
An alternative for discounted MDPs is Dantzig's rule, originally introduced
for the simplex method \citep{Dantzig1963}. It selects a state--action pair
with the largest improvement across all states, updates the policy only
at that state, and then re-evaluates the policy \citep{Ye2011}.

For discounted MDPs with a fixed discount factor, policy iteration and
Dantzig's rule are strongly polynomial, with arithmetic-operation bounds
polynomial in the numbers of states and actions \citep{Ye2011,Scherrer2016}.
As with the earlier interior-point algorithm of \citet{Ye2005}, these bounds
depend on the discount and therefore do not establish strong polynomiality
when the discount is part of the input.
For deterministic discounted MDPs, \citet{PostYe2015} prove that Dantzig's
rule is strongly polynomial independently of the discount, with an improved
iteration bound given by~\citet{HansenKaplanZwick2014}.
For deterministic average-reward MDPs, \citet{Karp1978,Madani2002}
give strongly polynomial algorithms based on dynamic programming.
A natural question is whether Howard's policy iteration is also strongly
polynomial on deterministic discounted and average-reward MDPs
\citep{AsadiEtAl2025,GoenkaEtAl2026}.

Exponential iteration lower bounds for Howard's policy iteration are known
for general average-reward MDPs \citep{Fearnley2010} and discounted MDPs
\citep{HollandersEtAl2012}.
These constructions use stochastic transitions and therefore do not
apply to deterministic MDPs.
For deterministic MDPs, the quadratic iteration lower bounds
in the average-reward setting \citep{HansenZwick2010}
and the discounted setting \citep{Hansen2012} are compatible with
a strongly polynomial upper bound.

\paragraph{Our contributions.}
We establish exponential iteration lower bounds for Howard's policy
iteration on deterministic discounted and average-reward MDPs, ruling out
polynomial-time complexity in both settings. Our constructions use a constant
number of actions per state. Under this restriction, no superpolynomial
iteration lower bound was previously known, even with stochastic transitions
\citep{MukherjeeKalyanakrishnan2025}.
Let \(N\) be the number of states and \(M\) the number of
state--action pairs. Our contributions are fourfold.
\begin{itemize}[leftmargin=*]
\item \textbf{Exponential lower bound for discounted MDPs.}
Howard's policy iteration can require \(\exp(\Omega(N))\)
iterations to find an optimal policy on deterministic discounted MDPs, even with at most two
actions per state (Theorem~\ref{thm:exponential}).
The lower bound is attained on instances of encoding length \(O(N^2)\),
ruling out polynomial running time for Howard's policy iteration.
In particular, this gives a negative answer to the long-standing question
of its strong polynomiality highlighted by \citet{PostYe2015}.
\item \textbf{Exponential lower bound for average-reward MDPs.}
We construct deterministic average-reward MDPs with at most four
actions per state on which Howard's policy iteration requires
\(\exp(\Omega(N))\) iterations before any state attains its optimal
gain (Theorem~\ref{thm:average}).
The instances have strongly connected transition graphs and
encoding length \(O(N^2)\).
Since \(M\le4N\) for these instances, the lower bound disproves
the conjectured \(M\)-iteration upper bound
\citep{Hansen2012,AsadiEtAl2025}.
\item \textbf{A strongly polynomial Dantzig counterpart.}
The discount-independent bound of
\citet{PostYe2015} also yields a strongly polynomial method
for deterministic average-reward MDPs by choosing a discount factor
sufficiently close to one.
We consider a Dantzig counterpart~\eqref{eq:av-single-switch}
that works directly with the gains and biases
used by Howard's policy iteration.
It selects a state--action pair with the largest improvement across all states
and updates only that state.
We prove that this rule finds an optimal policy
on any deterministic average-reward MDP
within \(O(N^3M^2\log^2(2N))\) iterations
(Theorem~\ref{thm:av-upper}).
\item \textbf{Small integer rewards.}
The two exponential constructions above use \(O(N)\) bits per reward.
Even when nonnegative integer rewards are restricted to logarithmic bit length,
Howard's policy iteration can require \(\exp(\Omega(N^{1/4}))\)
iterations on deterministic discounted MDPs
(Theorem~\ref{thm:small-rewards}).
We further obtain an \(\exp(\Omega((N/\log N)^{1/4}))\) iteration lower bound
when each reward has constant bit length
(Theorem~\ref{thm:constant-rewards}).
Both families have encoding length \(O(N\log N)\), so Howard's policy iteration
is not a polynomial-time algorithm even when each reward has constant bit length.
By contrast, Howard's policy iteration is strongly polynomial on deterministic
average-reward MDPs with nonnegative integer rewards of at most logarithmic
bit length (Theorem~\ref{thm:av-bounded-rewards}), giving a qualitative
complexity gap between the two settings.
\end{itemize}

Together with the strongly polynomial bound for Dantzig's rule on
deterministic discounted MDPs \citep{PostYe2015}, our results establish
an exponential separation between Howard's policy iteration and Dantzig's rule
in both settings.
This separation reveals a ``price'' of algorithmic anarchy:
Howard's policy iteration makes simultaneous selfish improvements at individual states
using the same policy values and can require exponentially more iterations than
Dantzig's rule, which coordinates updates across states and re-evaluates the policy
after each switch.
Table~\ref{tab:iteration-summary} summarizes the results.

\begin{table}[t]
\centering
\setlength{\abovecaptionskip}{0pt}
\setlength{\belowcaptionskip}{7pt}
\caption{Iteration bounds for deterministic MDPs.
Bit restrictions apply to nonnegative integer rewards.}
\label{tab:iteration-summary}
\begingroup
\papertablestyle
\normalsize
\hypersetup{linkcolor=redd,citecolor=bluee}
\setlength{\tabcolsep}{7pt}
\setlength{\extrarowheight}{0pt}
\setlength{\heavyrulewidth}{.6pt}
\setlength{\lightrulewidth}{.35pt}
\renewcommand{\arraystretch}{1.7}
\begin{tabular}{
  >{\centering\arraybackslash}m{\dimexpr.19\linewidth-2\tabcolsep\relax}
  >{\centering\arraybackslash}m{\dimexpr.17\linewidth-2\tabcolsep\relax}
  >{\centering\arraybackslash}m{\dimexpr.39\linewidth-2\tabcolsep\relax}
  >{\centering\arraybackslash}m{\dimexpr.25\linewidth-2\tabcolsep\relax}}
\toprule
\rowcolor{paperBand}
\papertableheading{Setting} &
\papertableheading{Reward bits} &
\papertableheading{Howard's policy iteration} &
\papertableheading{Dantzig's rule} \\
\midrule
\multirow{3}{=}{\centering\makecell[c]{Discounted\\MDP}} &
Unrestricted &
\(\exp(\Omega(N))\)\enspace{\footnotesize (Theorem~\ref{thm:exponential})} &
\multirow{3}{=}{\centering\makecell[c]{Strongly polynomial\\
{\footnotesize\citep{PostYe2015}}}} \\
&
\(O(\log N)\) bits &
\(\exp(\Omega(N^{1/4}))\)\enspace{\footnotesize (Theorem~\ref{thm:small-rewards})} & \\
&
\(O(1)\) bits &
\(\exp(\Omega((N/\log N)^{1/4}))\)\enspace{\footnotesize (Theorem~\ref{thm:constant-rewards})} & \\
\midrule
\multirow{3}{=}{\centering\makecell[c]{Average-reward\\MDP}} &
Unrestricted &
\(\exp(\Omega(N))\)\enspace{\footnotesize (Theorem~\ref{thm:average})} &
\multirow{3}{=}{\centering\makecell[c]{Strongly polynomial\\
{\footnotesize (Theorem~\ref{thm:av-upper})}}} \\
&
\(O(\log N)\) bits &
\multirow{2}{=}{\centering\makecell[c]{Strongly polynomial\\
{\footnotesize (Theorem~\ref{thm:av-bounded-rewards})}}} & \\
& \(O(1)\) bits & & \\
\bottomrule
\end{tabular}
\endgroup
\end{table}

\noeqref{proof:1,proof:6,proof:7,proof:14,proof:17,proof:20,proof:29,proof:31,proof:32,proof:33,proof:35,proof:37,proof:41,proof:38,eq:policy-value}
\section{Lower bounds for discounted MDPs}
\label{sec:construction}

A deterministic discounted MDP has a finite state space \(\mathcal S\)
of size \(N\), a nonempty finite action set \(\mathcal{A}_s\)
at each state \(s\), and a discount factor \(\gamma\in(0,1)\).
Taking action \(a\in\mathcal{A}_s\) at state \(s\)
yields a reward \(r(s,a)\in\mathbb Q\) and moves
to the next state \(f(s,a)\in\mathcal S\).
A policy \(\pi\) selects an action \(\pi(s)\in\mathcal{A}_s\) at each
state. For an initial state \(s\), let \(s_{(0)}=s\) and
\(s_{(\tau+1)}=f(s_{(\tau)},\pi(s_{(\tau)}))\) for \(\tau\ge0\).
The value and action-value functions are
\begin{equation}
V^\pi(s):=\sum_{\tau=0}^{\infty}\gamma^\tau r(s_{(\tau)},\pi(s_{(\tau)})),\qquad
Q^\pi(s,a):=r(s,a)+\gamma V^\pi(f(s,a)).
\label{eq:policy-value}
\end{equation}
In particular, \(V^\pi(s)=Q^\pi(s,\pi(s))\).
The objective is to find a policy \(\pi^\star\) satisfying
\(V^{\pi^\star}(s)=\max_\pi V^\pi(s)\) for every \(s\in\mathcal S\).
Starting from an initial policy \(\pi_0\), let \(\pi_t\) denote
the policy after \(t\) iterations.
At each iteration, both algorithms evaluate \(\pi_t\) and update the policy
using the action values in~\eqref{eq:policy-value}:
\begin{itemize}[leftmargin=*]
\item \emph{Howard's policy iteration.}
The algorithm updates all states simultaneously according to
\begin{equation}
\pi_{t+1}(s)\in\operatorname*{arg\,max}_{a\in\mathcal{A}_s}
Q^{\pi_t}(s,a),\qquad\forall\,s\in\mathcal S.
\label{eq:policy-update}
\end{equation}
If \(\pi_t(s)\) attains the maximum in~\eqref{eq:policy-update}, we set
\(\pi_{t+1}(s)=\pi_t(s)\). Otherwise, we choose the first maximizer
in a fixed ordering of \(\mathcal A_s\). The algorithm terminates when
\(\pi_{t+1}=\pi_t\), which holds if and only if \(\pi_t\) is optimal.

\item \emph{Dantzig's rule.}
The algorithm selects a state--action pair
\begin{equation}
(s_t,a_t)\in\operatorname*{arg\,max}_{s\in\mathcal S,\,a\in\mathcal A_s}
\bigl(Q^{\pi_t}(s,a)-V^{\pi_t}(s)\bigr).
\label{eq:discounted-dantzig}
\end{equation}
If the maximum in~\eqref{eq:discounted-dantzig} is zero, the algorithm terminates.
Otherwise, it sets \(\pi_{t+1}(s_t)=a_t\) and
\(\pi_{t+1}(s)=\pi_t(s)\) for every \(s\ne s_t\).
\end{itemize}

\citet{PostYe2015} prove that Dantzig's rule is strongly polynomial
on deterministic discounted MDPs.
The following theorem shows that Howard's policy iteration can require exponentially
many iterations.

\begin{theorem}
\label{thm:exponential}
There is a family of deterministic discounted MDPs with \(N\) states
and at most two actions per state, on which Howard's policy iteration
performs at least \(\exp(\Omega(N))\) iterations from a specified
initial policy.
Each instance has encoding length \(O(N^2)\).
\end{theorem}

For deterministic discounted MDPs with at most two actions per state
and nonnegative integer rewards of logarithmic bit length,
\citet{MukherjeeKalyanakrishnan2025} give a subexponential iteration upper
bound of \(\exp(O(\sqrt N(\log N)^{3/2}))\) for Howard's policy iteration,
independently of the discount.
Thus, the exponential lower bound in Theorem~\ref{thm:exponential}
cannot persist under this reward restriction.
The next theorem nevertheless gives a stretched-exponential lower bound.

\begin{theorem}
\label{thm:small-rewards}
There is a family of deterministic discounted MDPs with \(N\) states,
at most two actions per state,
and nonnegative integer rewards smaller than \(4N\), on which
Howard's policy iteration performs at least \(\exp(\Omega(N^{1/4}))\)
iterations from a specified initial policy.
Each instance has encoding length \(O(N\log N)\).
\end{theorem}

The iteration count remains superpolynomial even when each reward
has constant bit length.

\begin{theorem}
\label{thm:constant-rewards}
There is a family of deterministic discounted MDPs with \(N\) states,
at most two actions per state, and rewards in \(\{0,\ldots,88\}\),
on which Howard's policy iteration performs at least
\(\exp(\Omega((N/\log N)^{1/4}))\) iterations from a specified
initial policy. Each instance has encoding length \(O(N\log N)\).
\end{theorem}

All three theorems use the construction below.
Section~\ref{sec:analysis} proves that policy iteration visits distinct policies
in binary-counting order, yielding an iteration lower bound exponential
in the number of counter bits.
Sections~\ref{sec:exponential-family} and~\ref{sec:bounded-rewards}
complete the proofs by relating the number of counter bits to the number
of states and verifying the reward and encoding bounds.

\subsection{Formal construction}
\label{sec:formal-construction}

For the construction below, fix an integer \(d\ge3\) and let
\(F_i,G_i\in\mathbb Z[x]\) for \(0\le i<d\).
The choice \(F_i(x)=2^i\) and \(G_i(x)=2^{d+i+1}\) yields
Theorem~\ref{thm:exponential}. We allow general integer polynomials
to obtain the small-reward constructions in Section~\ref{sec:bounded-rewards}.
Assume that \(F_0(x)>0\) for \(x\) sufficiently close to \(1\) from below.
Define
\begin{equation}
\mathsf f_i(\gamma):=\frac{F_i(\gamma)}{F_0(\gamma)},\quad
\mathsf f_i:=\lim_{\gamma\to1^-}\mathsf f_i(\gamma),\qquad
\mathsf g_i(\gamma):=\frac{G_i(\gamma)}{F_0(\gamma)},\quad
\mathsf g_i:=\lim_{\gamma\to1^-}\mathsf g_i(\gamma),
\label{proof:29}
\end{equation}
where the limits are assumed to exist and satisfy
\begin{equation}
\mathsf{f}_0=1,\quad \mathsf{f}_i\ge2\mathsf{f}_{i-1},\qquad
\mathsf{g}_0\ge4\mathsf{f}_{d-1},\quad \mathsf{g}_i\ge2\mathsf{g}_{i-1}
\qquad\forall\,1\le i<d.
\label{proof:1}
\end{equation}
Set \(L=2+\max_{0\le i<d}\{\deg F_i,\deg G_i\}\ge2\),
where \(\deg\) denotes polynomial degree.
Write
\[
F_i(x)=\sum_{j=1}^{L-1}F_{i,j}x^{j-1},\qquad
G_i(x)=\sum_{j=1}^{L-1}G_{i,j}x^{j-1},
\]
and set \(F_{i,j}=G_{i,j}=0\) for
\(j\notin\{1,\ldots,L-1\}\).

\paragraph{States and actions.}
The state space \(\mathcal S\) is partitioned into four sets:
\begin{enumerate}[leftmargin=*]
\item \(\mathcal S_1=\{s_i,s_{i,1},\ldots,s_{i,5},\bar s_i,\bar s_{i,1},\ldots,\bar s_{i,4}\}_{i=0}^{d-1}
\cup\{s_d,\bar s_d\}\), with \(|\mathcal S_1|=11d+2\).

\item \(\mathcal S_2=\{s_i^{\mathrm c},s_i^{\mathrm z}\}_{i=0}^{d-1}
\cup\{s_{-1}^{\mathrm c},s_{-1}^{\mathrm z}\}\), with
\(|\mathcal S_2|=2d+2\).

\item \(\mathcal S_3=\{s_{i,0}^{\mathrm r},s_{i,1}^{\mathrm r},s_{i,2}^{\mathrm r}\}_{i=0}^{d-1}\),
with \(|\mathcal S_3|=3d\).

\item \(\mathcal S_4\) consists of the \((67d+5)L-32d-5\) intermediate
states introduced in the transition construction below.
\end{enumerate}
The action sets are
\[
\mathcal A_s=
\begin{cases}
\{a_0\},&s\in\mathcal S_4\cup\{\bar s_d,s_{-1}^{\mathrm c},s_{-1}^{\mathrm z}\},\\
\{a_0,a_1\},&\text{otherwise}.
\end{cases}
\]
For the total number of states, we obtain
\[
N=\sum_{j=1}^4|\mathcal S_j|=(67d+5)L-16d-1=\Theta(dL).
\label{proof:33}
\]

\paragraph{Transitions.}
For each \(s\in\mathcal S_1\cup\mathcal S_2\cup\mathcal S_3\) and
\(a\in\mathcal A_s\), Table~\ref{tab:paths-s1} specifies an endpoint
\(s'\in\mathcal S_1\cup\mathcal S_2\cup\mathcal S_3\) and a path
length \(m\), a positive multiple of \(L\).
Introduce \(m-1\) intermediate states
\(\{s^\circ_{a,j}\}_{j=1}^{m-1}\) and set
\[
f(s,a)=s^\circ_{a,1},\qquad
f(s^\circ_{a,j},a_0)=s^\circ_{a,j+1}\quad\forall\,1\le j<m-1,\qquad
f(s^\circ_{a,m-1},a_0)=s'.
\]
Thus, starting from \(s\), we reach \(s^\circ_{a,j}\) after \(j\) transitions.
The intermediate states are distinct for different pairs \((s,a)\)
and together constitute \(\mathcal S_4\).
\begin{table}[t]
\centering
\setlength{\abovecaptionskip}{0pt}
\setlength{\belowcaptionskip}{7pt}
\caption{Paths from \(\mathcal S_1\), \(\mathcal S_2\), and \(\mathcal S_3\).
Here \(0\le i<d\) except in the row explicitly indexed by \(-1\),
and \(1\le j\le3\) in the second row of the \(\mathcal S_1\) block.
Slashes separate corresponding states, as in \(s_i/\bar s_i\).
A~dash indicates an unavailable action.}
\label{tab:paths-s1}
\label{tab:paths-s2}
\label{tab:paths-s3}
\begingroup
\papertablestyle
\begin{tabular}{@{}>{\centering\arraybackslash}p{0.075\linewidth}@{\hspace{0.015\linewidth}}
  p{0.19\linewidth}@{}
  >{\centering\arraybackslash}p{0.225\linewidth}@{}
  >{\centering\arraybackslash}p{0.135\linewidth}@{}
  >{\centering\arraybackslash}p{0.225\linewidth}@{}
  >{\centering\arraybackslash}p{0.135\linewidth}@{}}
\toprule
\rowcolor{paperBand}
\multicolumn{1}{@{}>{\columncolor{paperBand}[0pt][\tabcolsep]}c@{\hspace{0.015\linewidth}}}{} & & \multicolumn{2}{c}{\papertableheading{Action \(a_0\)}} & \multicolumn{2}{c}{\papertableheading{Action \(a_1\)}}\\
\cmidrule{1-6}
\rowcolor{paperBand}
\multicolumn{1}{@{}>{\columncolor{paperBand}[0pt][\tabcolsep]}c@{\hspace{0.015\linewidth}}}{\papertableheading{Set}} &
\papertableheading{Starting state} & \papertableheading{Endpoint} &
\papertableheading{Length} & \papertableheading{Endpoint} &
\multicolumn{1}{@{}>{\columncolor{paperBand}[\tabcolsep][0pt]\centering\arraybackslash}p{0.135\linewidth}@{}}{\papertableheading{Length}}\\
\midrule
 & \(s_i/\bar s_i\) & \(s_{i+1}/\bar s_{i+1}\) & \(6L\) & \(s_{i,1}/\bar s_{i,1}\) & \(L\)\\
 & \(s_{i,j}/\bar s_{i,j}\) & \(s_{i+1}/\bar s_{i+1}\) & \((6-j)L\) & \(s_{i,j+1}/\bar s_{i,j+1}\) & \(L\)\\
 & \(s_{i,4}\) & \(s_{i,5}\) & \(L\) & \(s_i^{\mathrm c}\) & \(2L\)\\
 & \(s_{i,5}\) & \(s_{i+1}\) & \(L\) & \(s_i^{\mathrm z}\) & \(L\)\\
 & \(\bar s_{i,4}\) & \(\bar s_{i+1}\) & \(2L\) & \(s_i\) & \(L\)\\
 & \(s_d\) & \(\bar s_0\) & \(L\) & \(s_0\) & \(L\)\\
\multirow{-7}{*}{\(\mathcal S_1\)} & \(\bar s_d\) & \(\bar s_0\) & \(L\) & \textemdash & \textemdash\\
\midrule
 & \(s_i^{\mathrm c}\) & \(s_{i,0}^{\mathrm r}\) & \(L\) & \(s_{i,1}^{\mathrm r}\) & \(L\)\\
 & \(s_i^{\mathrm z}\) & \(s_{i,0}^{\mathrm r}\) & \(L\) & \(s_{i,2}^{\mathrm r}\) & \(L\)\\
\multirow{-3}{*}{\(\mathcal S_2\)} & \(s_{-1}^{\mathrm c},s_{-1}^{\mathrm z}\) & \(\bar s_0\) & \(L\) & \textemdash & \textemdash\\
\midrule
 & \(s_{i,0}^{\mathrm r}\) & \(s_{i-1}^{\mathrm z}\) & \(L\) & \(s_{i+1}\) & \(6L\)\\
 & \(s_{i,1}^{\mathrm r}\) & \(s_{i-1}^{\mathrm c}\) & \(L\) & \(s_{i,1}\) & \(L\)\\
\multirow{-3}{*}{\(\mathcal S_3\)} & \(s_{i,2}^{\mathrm r}\) & \(s_{i-1}^{\mathrm z}\) & \(L\) & \(s_{i,1}\) & \(L\)\\
\bottomrule
\end{tabular}
\endgroup
\end{table}

\begin{enumerate}[leftmargin=*]
\item \emph{Paths starting in \(\mathcal S_1\).}
Figure~\ref{fig:bit-paths} shows the \(6L\)-transition paths from
\(s_i\) to \(s_{i+1}\) and from \(\bar s_i\) to \(\bar s_{i+1}\),
together with the branches at \(s_{i,4},s_{i,5},\bar s_{i,4}\).
The paths from \(s_d\) have length \(L\) and end at \(\bar s_0\)
under \(a_0\) and at \(s_0\) under \(a_1\).
The sole action at \(\bar s_d\) gives an \(L\)-transition path to \(\bar s_0\).

\item \emph{Paths starting in \(\mathcal S_2\).}
For \(0\le i<d\), the \(a_0\) paths from \(s_i^{\mathrm c}\) and
\(s_i^{\mathrm z}\) both end at \(s_{i,0}^{\mathrm r}\).
The \(a_1\) paths end at \(s_{i,1}^{\mathrm r}\) and
\(s_{i,2}^{\mathrm r}\), respectively.
The sole actions at \(s_{-1}^{\mathrm c},s_{-1}^{\mathrm z}\) give
paths to \(\bar s_0\). Every path from \(\mathcal S_2\) has length \(L\).

\item \emph{Paths starting in \(\mathcal S_3\).}
Under \(a_0\), the path from \(s_{i,1}^{\mathrm r}\) ends at
\(s_{i-1}^{\mathrm c}\), and those from \(s_{i,0}^{\mathrm r}\)
and \(s_{i,2}^{\mathrm r}\) end at \(s_{i-1}^{\mathrm z}\).
Under \(a_1\), the path from \(s_{i,0}^{\mathrm r}\) ends at
\(s_{i+1}\), and those from \(s_{i,1}^{\mathrm r}\) and
\(s_{i,2}^{\mathrm r}\) end at \(s_{i,1}\).
All paths have length \(L\) except the \(a_1\) path from
\(s_{i,0}^{\mathrm r}\), which has length \(6L\).
\end{enumerate}

\begin{figure}[t]
\centering
\begingroup
\begin{tikzpicture}[
  x=1pt,y=1pt,font=\small,text=black,
  state/.style={paper state,text height=1.8ex,text depth=0.65ex},
  paired state/.style={paper state,minimum size=34pt},
  endpoint/.style={state,fill=white},
  path arrow/.style={paper arrow},
  teal path/.style={path arrow,draw=paperTeal,text=black},
  bronze path/.style={path arrow,draw=paperBronze,text=black},
  path label/.style={paper label}
]
\path[use as bounding box] (0,-45) rectangle (456,180);
\node[paired state] (s0) at (25,116)
  {\shortstack{$s_i$\\[-1pt]$\bar s_i$}};
\node[paired state] (s1) at (106,116)
  {\shortstack{$s_{i,1}$\\[-1pt]$\bar s_{i,1}$}};
\node[paired state] (s2) at (187,116)
  {\shortstack{$s_{i,2}$\\[-1pt]$\bar s_{i,2}$}};
\node[paired state] (s3) at (268,116)
  {\shortstack{$s_{i,3}$\\[-1pt]$\bar s_{i,3}$}};
\node[paired state] (s4) at (349,116)
  {\shortstack{$s_{i,4}$\\[-1pt]$\bar s_{i,4}$}};
\node[paired state] (s5) at (430,116)
  {\shortstack{$s_{i+1}$\\[-1pt]$\bar s_{i+1}$}};
\draw[teal path] (s0.east) -- node[path label,above=4pt] {$a_1$} (s1.west);
\draw[teal path] (s1.east) -- node[path label,above=4pt] {$a_1$} (s2.west);
\draw[teal path] (s2.east) -- node[path label,above=4pt] {$a_1$} (s3.west);
\draw[teal path] (s3.east) -- node[path label,above=4pt] {$a_1$} (s4.west);
\draw[bronze path] (s4.east) --
  node[path label,above=4pt] {$2L$} (s5.west);
\draw[bronze path] (s0.35)
  .. controls (130,165) and (325,165) ..
  node[path label,pos=0.5,above=5pt] {$a_0:\ 6L$} (s5.145);

\node[state,fill=white] (u) at (24,20) {$s_{i,4}$};
\node[state,fill=white] (v) at (112,20) {$s_{i,5}$};
\node[endpoint] (c) at (218,67) {$s_i^{\mathrm c}$};
\node[endpoint] (next) at (218,20) {$s_{i+1}$};
\node[endpoint] (z) at (218,-27) {$s_i^{\mathrm z}$};
\draw[bronze path] (u.east) --
  node[path label,below=4pt] {$a_0:\ L$} (v.west);
\draw[bronze path] (v.east) --
  node[path label,below=4pt] {$a_0:\ L$} (next.west);
\draw[teal path] (u) --
  node[path label,pos=0.58,above=8pt] {$a_1:\ 2L$} (c);
\draw[teal path] (v) --
  node[path label,pos=0.57,below=9pt] {$a_1:\ L$} (z);

\node[state,fill=white] (b) at (326,20) {$\bar s_{i,4}$};
\node[endpoint] (b0) at (430,67) {$\bar s_{i+1}$};
\node[endpoint] (b1) at (430,-27) {$s_i$};
\draw[bronze path] (b) --
  node[path label,pos=0.5,above=10pt] {$a_0:\ 2L$} (b0);
\draw[teal path] (b) --
  node[path label,pos=0.5,below=9pt] {$a_1:\ L$} (b1);
\end{tikzpicture}
\endgroup
\caption{Paths and action choices in the discounted construction.
The top panel shows the \(6L\)-transition paths from \(s_i\) to
\(s_{i+1}\) and from \(\bar s_i\) to \(\bar s_{i+1}\), with the upper
and lower labels in each circle referring to unbarred and barred states,
respectively. Arrows without a length label represent paths of length \(L\).
The upper \(2L\) arrow represents either the path through \(s_{i,5}\)
with \(a_0\) at both \(s_{i,4}\) and \(s_{i,5}\), or the \(a_0\)
path from \(\bar s_{i,4}\) to \(\bar s_{i+1}\).
The bottom panels show the action choices and path lengths at
\(s_{i,4},s_{i,5},\bar s_{i,4}\).}
\label{fig:bit-paths}
\end{figure}

The \(32d+5\) paths in Table~\ref{tab:paths-s1} have total length \((67d+5)L\).
Each path of length \(m\) adds \(m-1\) states, so
\(|\mathcal S_4|=(67d+5)L-32d-5\).

\paragraph{Rewards.}
We first define integer rewards, allowing negative values.
Before the uniform shift below, nonzero rewards occur only at the
intermediate states in \(\mathcal S_4\).
For \(s\in\mathcal S_1\cup\mathcal S_2\cup\mathcal S_3\),
each action \(a\in\mathcal A_s\) has immediate reward \(r(s,a)=0\)
and selects a path through \(\mathcal S_4\).
The following rules specify the intermediate rewards according to
the set containing the starting state of each path.
For a path of length \(m\) in Table~\ref{tab:paths-s1},
the formulas apply to \(1\le j<m\).
All unspecified rewards are zero.
\begin{enumerate}[leftmargin=*]
\item \emph{Paths starting in \(\mathcal S_1\).}
For \(s\in\mathcal S_1\), set
\begin{subequations}
\label{eq:rewards-s1}
\noeqref{eq:reward-s14-a0,eq:reward-s14-a1,eq:reward-s15-a0,eq:reward-s15-a1,
eq:reward-bar-s14-a0,eq:reward-bar-s14-a1,eq:reward-sd-a0}
\begin{empheq}[left={r(s^\circ_{a,j},a_0)=\empheqlbrace}]{alignat=2}
&(96d+16)G_{i,j},\quad &&s=s_{i,4},\ a=a_0,
\label{eq:reward-s14-a0}\\
&4F_{i,j}-16F_{d-1,j}-2F_{0,j},\quad &&s=s_{i,4},\ a=a_1,
\label{eq:reward-s14-a1}\\
&-(96d+16)G_{i,j},\quad &&s=s_{i,5},\ a=a_0,
\label{eq:reward-s15-a0}\\
&-(96d+16)G_{i,j}-4F_{i,j}-F_{0,j},\quad &&s=s_{i,5},\ a=a_1,
\label{eq:reward-s15-a1}\\
&(96d+16)(G_{i,j}-G_{i,j-L}),\quad &&s=\bar s_{i,4},\ a=a_0,
\label{eq:reward-bar-s14-a0}\\
&-32F_{d-1,j},\quad &&s=\bar s_{i,4},\ a=a_1,
\label{eq:reward-bar-s14-a1}\\
&16F_{d-1,j},\quad &&s=s_d,\ a=a_0,
\label{eq:reward-sd-a0}\\
&0,\quad &&\text{otherwise}.\notag
\end{empheq}
\end{subequations}
\item \emph{Paths starting in \(\mathcal S_2\).}
For \(s\in\mathcal S_2\), set
\[
r(s^\circ_{a,j},a_0)=
\begin{cases}
16F_{d-1,j},&s=s_{-1}^{\mathrm c},\ a=a_0,\\
0,&\text{otherwise}.
\end{cases}
\label{eq:rewards-s2}
\]
\item \emph{Paths starting in \(\mathcal S_3\).}
Set
\begin{subequations}
\label{eq:rewards-s3}
\noeqref{eq:reward-sr0-a0,eq:reward-sr0-a1,eq:reward-sr12-a1}
\begin{empheq}[left={r(s^\circ_{a,j},a_0)=\empheqlbrace}]{alignat=2}
&4F_{i,j},\quad &&s=s_{i,0}^{\mathrm r},\ a=a_0,
\label{eq:reward-sr0-a0}\\
&-32F_{d-1,j}-4F_{i,j}+4F_{0,j}-4F_{0,j-L},\quad &&s=s_{i,0}^{\mathrm r},\ a=a_1,
\label{eq:reward-sr0-a1}\\
&-32F_{d-1,j}-4F_{i,j},\quad &&s\in\{s_{i,1}^{\mathrm r},s_{i,2}^{\mathrm r}\},\ a=a_1,
\label{eq:reward-sr12-a1}\\
&0,\quad &&\text{otherwise}.\notag
\end{empheq}
\end{subequations}
\end{enumerate}
Let \(r_{\max}=\max_{s\in\mathcal S,\,a\in \mathcal{A}_s}|r(s,a)|\)
be the maximum absolute unshifted immediate reward.
Add \(r_{\max}\) to every reward, so the final rewards lie in \([0,2r_{\max}]\).
This adds \(r_{\max}/(1-\gamma)\) to every \(V^\pi(s)\) and \(Q^\pi(s,a)\),
leaving all comparisons and optimality gaps unchanged. The proof uses the
unshifted rewards.

\paragraph{Discount and initial policy.}
We choose the discount factor \(\gamma\) and the initial policy
\(\pi_0\) as follows:
\begin{equation}
\gamma=1-\frac{1}{4^N r_{\max}},\qquad
\pi_0(s)=
\begin{cases}
a_1,&s\in\{s_d\}\cup\{s_{i,j}:0\le i<d,\ 1\le j\le3\},\\
a_0,&\text{otherwise}.
\end{cases}
\label{proof:32}
\end{equation}

\subsection{Binary encoding and policy values}
\label{sec:counter-interpretation}

A policy \(\pi\) encodes the same vector \(\mathbf b\in\{0,1\}^d\)
in the first group \(\{s_i\}_{i=0}^{d-1}\) and the second group
\(\{\bar s_i\}_{i=0}^{d-1}\) if
\[
\pi(s_i)=\pi(\bar s_i)=a_{[\mathbf b]_i},\qquad \forall\,0\le i<d.
\]
The states \(s_i\) and \(\bar s_i\) represent the $i$-th bit,
with \(i=0\) denoting the least significant bit.
Write \(\mathbf0=(0,\ldots,0)\) and \(\mathbf1=(1,\ldots,1)\).
For \(\mathbf b\ne\mathbf1\), let \(\mathbf b^+\) denote the result
of adding one in binary. For example, when \(d=5\),
\(\mathbf b=01011\) represents \(11\), and \(\mathbf b^+=01100\)
represents \(12\).
For each \(\mathbf b\in\{0,1\}^d\), define the policy
\(\pi[\mathbf b]\) by
\begin{equation}
\pi[\mathbf b](s)=
\begin{cases}
a_{[\mathbf b]_i},
&s\in\{s_i,\bar s_i,s_i^{\mathrm c},s_i^{\mathrm z},
\bar s_{i,1},\bar s_{i,2},\bar s_{i,3}\},\quad 0\le i<d,\\
\pi_0(s),&\text{otherwise},
\end{cases}
\label{eq:encoding-policy}
\end{equation}
where \(\pi_0\) is defined
in~\eqref{proof:32}. In particular, \(\pi[\mathbf0]=\pi_0\).
For \(0\le i<d\), \(\pi[\mathbf b]\) selects the following paths
from \(s_i\) to \(s_{i+1}\) and from \(\bar s_i\) to \(\bar s_{i+1}\):

\begin{itemize}[leftmargin=*,topsep=3pt,itemsep=2pt,parsep=0pt]
\item If \([\mathbf b]_i=0\), the policy selects \(a_0\) at \(s_i\)
and \(\bar s_i\). Each path reaches its endpoint in \(6L\) transitions.
\item If \([\mathbf b]_i=1\), the policy follows the horizontal paths
in the top panel of Figure~\ref{fig:bit-paths}.
From \(s_i\), it passes through \(s_{i,1},\ldots,s_{i,5}\) to
\(s_{i+1}\): four \(a_1\) paths followed by two \(a_0\) paths,
all of length \(L\).
From \(\bar s_i\), it passes through \(\bar s_{i,1},\ldots,\bar s_{i,4}\)
to \(\bar s_{i+1}\): four \(a_1\) paths of length \(L\), followed
by one \(a_0\) path of length \(2L\). Both paths therefore have length \(6L\).
\end{itemize}
The policy selects \(a_1\) at \(s_d\). The selected paths form a cycle
from \(s_0\) through \(s_1,\ldots,s_d\) and back to \(s_0\), and
another from \(\bar s_0\) through \(\bar s_1,\ldots,\bar s_d\)
and back to \(\bar s_0\).
Each cycle consists of \(d\) paths of length \(6L\) followed by an
\(L\)-transition path back to its starting state, giving total length
\((6d+1)L\).

Our construction is designed to make policy iteration follow
an exponentially long sequence of binary encodings,
advancing one increment at a time.
The rewards involving \(F_i\) make the auxiliary action comparisons
detect whether all lower bits are one, which determines whether
bit \(i\) flips in the next increment.
For \(\gamma\) sufficiently close to one, the rewards involving
\(G_i\) ensure that the value gained by setting the least significant
zero bit to one exceeds the value lost by resetting the lower
bits to zero.

\begin{lemma}[Value ordering]
\label{lem:encoding-values}
There exists \(\gamma_0\in(0,1)\) such that, for all
\(\gamma\in(\gamma_0,1)\) and
\(\mathbf b\in\{0,1\}^d\setminus\{\mathbf1\}\),
\[
V^{\pi[\mathbf b^+]}(s_0)>V^{\pi[\mathbf b]}(s_0),
\qquad
V^{\pi[\mathbf b^+]}(\bar s_0)>V^{\pi[\mathbf b]}(\bar s_0).
\]
\end{lemma}

\begin{proof}
We use the unshifted rewards, which give the same value differences.
Fix any \(\mathbf b\in\{0,1\}^d\). For \(0\le i<d\), consider
the paths from \(s_i\) to \(s_{i+1}\) and from \(\bar s_i\) to
\(\bar s_{i+1}\) selected by \(\pi[\mathbf b]\).
If \([\mathbf b]_i=0\), the policy selects \(a_0\) at \(s_i\)
and \(\bar s_i\), and every reward along both paths is zero by the
last case of~\eqref{eq:rewards-s1}.
If \([\mathbf b]_i=1\), both paths have zero rewards for the
first \(4L\) transitions.
By~\eqref{eq:reward-s14-a0},~\eqref{eq:reward-s15-a0},
and~\eqref{eq:reward-bar-s14-a0}, the remaining \(2L\) transitions on each path carry
the following two \(L\)-term reward sequences, in order:
\[
\bigl(0,\ (96d+16)G_{i,1},\ \ldots,\ (96d+16)G_{i,L-1}\bigr),
\qquad
\bigl(0,\ -(96d+16)G_{i,1},\ \ldots,\ -(96d+16)G_{i,L-1}\bigr).
\]
On each \(6L\)-transition path, the second sequence starts \(L\)
transitions after the first and contributes \(-\gamma^L\) times the
discounted reward of the first. For the total discounted reward along
each path, we therefore obtain
\[
(96d+16)\gamma^{4L}(1-\gamma^L)
\sum_{j=1}^{L-1}\gamma^jG_{i,j}
=(96d+16)\gamma^{4L+1}(1-\gamma^L)G_i(\gamma).
\label{proof:31}
\]
Each bit \(i\) with \([\mathbf b]_i=1\) contributes \(\gamma^{6iL}\) times~\eqref{proof:31}
to the discounted reward over one cycle starting from \(s_0\) or
\(\bar s_0\). The paths selected at \(s_d\) and \(\bar s_d\) have
zero rewards. Summing over \(i\) with \([\mathbf b]_i=1\) and repeating the cycle
every \((6d+1)L\) transitions, we obtain
\[
V^{\pi[\mathbf b]}(s_0)=V^{\pi[\mathbf b]}(\bar s_0)
=\frac{(96d+16)\gamma^{4L+1}(1-\gamma^L)}
{1-\gamma^{(6d+1)L}}
\sum_{i=0}^{d-1}\gamma^{6iL}G_i(\gamma)[\mathbf b]_i.
\label{eq:encoded-cycle-value}
\]
For \(\mathbf b\ne\mathbf1\), we compare the cycle values
in~\eqref{eq:encoded-cycle-value} for \(\mathbf b\) and \(\mathbf b^+\).
Let \(\kappa=\min\{i:[\mathbf b]_i=0\}\) be the index of the least
significant zero bit of \(\mathbf b\).
Adding one changes bit \(\kappa\) from zero to one and every lower bit
from one to zero, leaving higher bits unchanged.
For \(s\in\{s_0,\bar s_0\}\), we combine the polynomial limits
in~\eqref{proof:29} with
\((1-\gamma^L)/(1-\gamma^{(6d+1)L})\to1/(6d+1)\) and
the condition \(\mathsf g_i\ge2\mathsf g_{i-1}\) in~\eqref{proof:1}
to obtain
\[
\begin{aligned}
\lim_{\gamma\to1^{-}}
\frac{V^{\pi[\mathbf b^+]}(s)-V^{\pi[\mathbf b]}(s)}
{4\gamma F_0(\gamma)}
&=4\sum_{i=0}^{d-1}\mathsf g_i
\bigl([\mathbf b^+]_i-[\mathbf b]_i\bigr)\\
&=4\Bigl(\mathsf g_\kappa-\sum_{i=0}^{\kappa-1}\mathsf g_i\Bigr)
\ge4\Bigl(\mathsf g_\kappa-\sum_{i=0}^{\kappa-1}(\mathsf g_{i+1}-\mathsf g_i)\Bigr)
=4\mathsf g_0>0.
\end{aligned}
\label{eq:increment-limit}
\]
We conclude that the two value inequalities in Lemma~\ref{lem:encoding-values}
hold for \(\gamma<1\) sufficiently close to one.
Since there are finitely many encodings, a common threshold
\(\gamma_0\in(0,1)\) suffices for all \(\mathbf b\ne\mathbf1\).
\end{proof}

By Lemma~\ref{lem:encoding-values}, for
\(\gamma<1\) sufficiently close to one, the \(2^d\) policies
\(\pi[\mathbf b]\), listed in binary order, have strictly increasing
values at \(s_0\) and \(\bar s_0\).
In Section~\ref{sec:analysis}, we show that policy iteration, starting
from \(\pi[\mathbf0]\), performs the first \(2^d-2\) increments in
this order, using five iterations per increment.
This yields \(\Omega(2^d)\) iterations.
Sections~\ref{sec:exponential-family}--\ref{sec:explicit-construction} derive the corresponding lower bounds
in the number of states.

\section{Proofs of the discounted lower bounds}
\label{sec:families}

\subsection{Binary counting under policy iteration}
\label{sec:analysis}

\begin{lemma}[Binary counting]
\label{prop:counter}
For the construction in Section~\ref{sec:construction}, with
\(\gamma,\pi_0\) in~\eqref{proof:32}, Howard's policy iteration
in~\eqref{eq:policy-update} produces \(\pi[\mathbf b]\) after
\(5\sum_{i=0}^{d-1}2^i[\mathbf b]_i\) iterations
for every \(\mathbf b\in\{0,1\}^d\setminus\{\mathbf1\}\).
\end{lemma}

\begin{figure}[t]
\centering
\begin{tikzpicture}[x=1cm,y=1cm,font=\small,text=black,
 state/.style={paper state},
 path/.style={paper arrow,draw=paperSlate},
 old/.style={path,draw=paperBronze},new/.style={path,draw=paperTeal},
 return/.style={rounded corners=2mm},
 lab/.style={paper label},
 title/.style={paper panel,anchor=west}]
\begin{scope}[xshift=0.00cm,yshift=-0.00cm]
\node[title] at (0,1.55) {(a) Before Iteration~1};
\node[state] (a0) at (0.55,0) {$s_0$};
\node[state] (ad) at (4.15,0) {$s_d$};
\node[state] (ab0) at (0.55,-2.25) {$\bar s_0$};
\node[state] (abd) at (4.15,-2.25) {$\bar s_d$};
\node[lab] at (2.35,.39) {Bits at $s_i$: $\mathbf b$};
\node[lab] at (2.35,-2.64) {Bits at $\bar s_i$: $\mathbf b$};
\draw[old] (a0)--(ad);
\draw[old,return] (ad.north)--(4.15,.90)--(.55,.90)--(a0.north);
\node[lab] at (2.35,1.10) {$a_1$};
\draw[old] (ab0)--(abd);
\draw[old,return] (abd.south)--(4.15,-3.15)--(.55,-3.15)--(ab0.south);
\end{scope}

\begin{scope}[xshift=5.18cm,yshift=-0.00cm]
\node[title] at (0,1.55) {(b) After Iteration~1};
\node[state] (b0) at (0.55,0) {$s_0$};
\node[state] (bd) at (4.15,0) {$s_d$};
\node[state] (bb0) at (0.55,-2.25) {$\bar s_0$};
\node[state] (bbd) at (4.15,-2.25) {$\bar s_d$};
\node[lab] at (2.35,.39) {Bits at $s_i$: $\mathbf1$};
\node[lab] at (2.35,-2.64) {Bits at $\bar s_i$: $\mathbf b$};
\draw[old] (bd)--node[lab,pos=.47,above=5pt] {$a_0$} (bb0);
\draw[old] (b0.south)--(bb0.north);
\node[lab,anchor=west] at (.73,-1.04) {via $\mathcal S_2$};
\draw[old] (bb0)--(bbd);
\draw[old,return] (bbd.south)--(4.15,-3.15)--(.55,-3.15)--(bb0.south);
\end{scope}

\begin{scope}[xshift=10.36cm,yshift=-0.00cm]
\node[title] at (0,1.55) {(c) After Iteration~2};
\node[state] (c0) at (0.55,0) {$s_0$};
\node[state] (cd) at (4.15,0) {$s_d$};
\node[state] (cb0) at (0.55,-2.25) {$\bar s_0$};
\node[state] (cbd) at (4.15,-2.25) {$\bar s_d$};
\node[lab] at (2.35,.39) {Bits at $s_i$: $\mathbf b^+$};
\node[lab] at (2.35,-2.64) {Bits at $\bar s_i$: $\mathbf b$};
\draw[path] (c0)--(cd);
\draw[old] (cd)--node[lab,pos=.47,above=5pt] {$a_0$} (cb0);
\draw[old] (cb0)--(cbd);
\draw[old,return] (cbd.south)--(4.15,-3.15)--(.55,-3.15)--(cb0.south);
\end{scope}

\begin{scope}[xshift=0.00cm,yshift=-5.25cm]
\node[title] at (0,1.55) {(d) After Iteration~3};
\node[state,draw=paperTeal] (d0) at (0.55,0) {$s_0$};
\node[state,draw=paperTeal] (dd) at (4.15,0) {$s_d$};
\node[state] (db0) at (0.55,-2.25) {$\bar s_0$};
\node[state] (dbd) at (4.15,-2.25) {$\bar s_d$};
\node[lab] at (2.35,.39) {Bits at $s_i$: $\mathbf b^+$};
\node[lab] at (2.35,-2.64) {Bits at $\bar s_i$: $\mathbf b$};
\draw[new] (d0)--(dd);
\draw[new,return] (dd.north)--(4.15,.90)--(.55,.90)--(d0.north);
\node[lab] at (2.35,1.10) {$a_1$};
\draw[old] (db0)--(dbd);
\draw[old,return] (dbd.south)--(4.15,-3.15)--(.55,-3.15)--(db0.south);
\end{scope}

\begin{scope}[xshift=5.18cm,yshift=-5.25cm]
\node[title] at (0,1.55) {(e) After Iteration~4};
\node[state,draw=paperTeal] (e0) at (0.55,0) {$s_0$};
\node[state,draw=paperTeal] (ed) at (4.15,0) {$s_d$};
\node[state] (eb0) at (0.55,-2.25) {$\bar s_0$};
\node[state] (ebd) at (4.15,-2.25) {$\bar s_{0,4}$};
\node[lab] at (2.35,.39) {Bits at $s_i$: $\mathbf b^+$};
\node[lab] at (2.35,-2.64) {Bits at $\bar s_i$: $\mathbf1$};
\draw[new] (e0)--(ed);
\draw[new,return] (ed.north)--(4.15,.90)--(.55,.90)--(e0.north);
\node[lab] at (2.35,1.10) {$a_1$};
\draw[new] (eb0)--(ebd);
\node[lab] at (2.35,-1.98) {$a_1$ paths};
\draw[new] (ebd)--node[lab,pos=.47,above=5pt] {$a_1$} (e0);
\end{scope}

\begin{scope}[xshift=10.36cm,yshift=-5.25cm]
\node[title] at (0,1.55) {(f) After Iteration~5};
\node[state,draw=paperTeal] (f0) at (0.55,0) {$s_0$};
\node[state,draw=paperTeal] (fd) at (4.15,0) {$s_d$};
\node[state,draw=paperTeal] (fb0) at (0.55,-2.25) {$\bar s_0$};
\node[state,draw=paperTeal] (fbd) at (4.15,-2.25) {$\bar s_d$};
\node[lab] at (2.35,.39) {Bits at $s_i$: $\mathbf b^+$};
\node[lab] at (2.35,-2.64) {Bits at $\bar s_i$: $\mathbf b^+$};
\draw[new] (f0)--(fd);
\draw[new,return] (fd.north)--(4.15,.90)--(.55,.90)--(f0.north);
\node[lab] at (2.35,1.10) {$a_1$};
\draw[new] (fb0)--(fbd);
\draw[new,return] (fbd.south)--(4.15,-3.15)--(.55,-3.15)--(fb0.south);
\end{scope}

\end{tikzpicture}
\caption{Selected paths and bit choices during the five iterations from
\(\pi[\mathbf b]\) to \(\pi[\mathbf b^+]\).
Circles denote states, and arrows represent paths selected by the current policy.
The upper and lower rows in each panel show the first and second groups,
respectively.}
\label{fig:five-updates}
\end{figure}

\looseness=-1 \noindent\textbf{Proof overview.}
Starting from \(\pi[\mathbf b]\), where \(\mathbf b\) represents an integer
in \(0,\ldots,2^d-3\), five iterations produce \(\pi[\mathbf b^+]\).
Figure~\ref{fig:five-updates} tracks the selected paths and bit choices.
Each update is simultaneous and uses action values under the preceding policy.

\begin{enumerate}[label=\textbf{Iteration \arabic*.},wide=0pt,
 labelsep=.6em,itemsep=5pt,topsep=5pt,parsep=0pt]
\item
Every \(s_i\) selects \(a_1\), while the second group keeps
\(\mathbf b\) and its cycle.
The comparisons under \(\pi[\mathbf b]\) at \(s_{i,4},s_{i,5}\)
select branches through \(\mathcal S_2\) at exactly the zero positions
of \(\mathbf b^+\). These branches determine which \(s_i\) will select
\(a_0\) in Iteration~2.
As Figure~\ref{fig:five-updates}(b) shows, the path from \(s_0\) now
enters the old cycle through \(\mathcal S_2\), and \(s_d\) selects
\(a_0\) toward \(\bar s_0\).
The first group therefore has no cycle of its own.

\item
The states \(s_i\) select \(a_0\) at the positions chosen in Iteration~1,
bypassing the branches through \(\mathcal S_2\), and keep \(a_1\) elsewhere.
The first group therefore encodes \(\mathbf b^+\).
The selected path from \(s_0\) reaches \(s_d\) and then enters the old
cycle through \(\bar s_0\), traversing the new encoding only once
(Figure~\ref{fig:five-updates}(c)).

\item
The state \(s_d\) selects \(a_1\), returning to \(s_0\).
The first-group path now repeats, forming a cycle for \(\mathbf b^+\),
while the second group keeps its cycle for \(\mathbf b\)
(Figure~\ref{fig:five-updates}(d)).
The new cycle's value advantage over the old cycle outweighs the reward
losses along the connecting paths, making them preferable in Iteration~4.

\item
The states \(\bar s_{i,4}\) and those in \(\mathcal S_3\) select
\(a_1\), establishing the connections into the first group.
All \(\bar s_i\) also select \(a_1\), but their paths enter the
first-group cycle through \(\bar s_{i,4}\), as in
Figure~\ref{fig:five-updates}(e).

\item
Through these connections, the action comparisons at
\(\bar s_i,s_i^{\mathrm c},s_i^{\mathrm z}\) reduce to the comparison
at \(s_i\), with a small advantage for \(a_0\).
When \([\mathbf b^+]_i=0\), the actions at \(s_i\) tie, so these states
select \(a_0\).
When \([\mathbf b^+]_i=1\), the extra reward from \(a_1\) outweighs
this advantage, so they select \(a_1\).
The states \(\bar s_{i,4}\) and those in \(\mathcal S_3\) return to
\(a_0\). Together with the updates at \(\{s_{i,j}\}_{j=1}^3\) and
\(\{\bar s_{i,j}\}_{j=1}^3\), these changes give the full policy
\(\pi[\mathbf b^+]\).
Both groups again have their own cycles, now encoding \(\mathbf b^+\)
(Figure~\ref{fig:five-updates}(f)).
\end{enumerate}
\looseness=-1 The states \(\{s_{i,j}\}_{j=1}^3\) and
\(\{\bar s_{i,j}\}_{j=1}^3\) delay changes at \(s_i\) and
\(\bar s_i\), respectively.
Within each set of three states, switches from \(a_0\) to \(a_1\)
proceed in the order \(j=3,2,1\), one per iteration, while ties under
the preceding policy keep the corresponding bit at zero.
This preserves the old second-group zeros through Iteration~3
and the new first-group zeros through Iteration~5.
Table~\ref{tab:auxiliary-updates} records the action choices during the increment.

\begin{table}[t]
\centering
\setlength{\abovecaptionskip}{0pt}
\setlength{\belowcaptionskip}{7pt}
\caption{Policy choices at all states with two actions during one binary increment.
Entries \(0,1\) denote \(a_0,a_1\).
Here \(0\le i<d\), \(0\le j\le2\), \(\kappa=\min\{i:[\mathbf b]_i=0\}\),
and \(\mathds{1}\{\cdot\}\) denotes the indicator.
All remaining states have only action \(a_0\).}
\label{tab:auxiliary-updates}
\begingroup
\papertablestyle
\begin{tabular}{
  >{\raggedright\arraybackslash}p{\dimexpr.12\linewidth-2\tabcolsep\relax}
  >{\centering\arraybackslash}p{\dimexpr.12\linewidth-2\tabcolsep\relax}
  >{\centering\arraybackslash}p{\dimexpr.22\linewidth-2\tabcolsep\relax}
  *{3}{>{\centering\arraybackslash}p{\dimexpr.14\linewidth-2\tabcolsep\relax}}
  >{\centering\arraybackslash}p{\dimexpr.12\linewidth-2\tabcolsep\relax}}
\toprule
\rowcolor{paperBand}
\papertableheading{State} &
\papertableheading{\(\pi_t\)} & \papertableheading{\(\pi_{t+1}\)} &
\papertableheading{\(\pi_{t+2}\)} & \papertableheading{\(\pi_{t+3}\)} &
\papertableheading{\(\pi_{t+4}\)} & \papertableheading{\(\pi_{t+5}\)}\\
\midrule
\(s_i\) & \([\mathbf b]_i\) & 1 & \([\mathbf b^+]_i\) & \([\mathbf b^+]_i\) & \([\mathbf b^+]_i\) & \([\mathbf b^+]_i\)\\
\(s_{i,1}\) & 1 & 1 & \([\mathbf b^+]_i\) & \([\mathbf b^+]_i\) & \([\mathbf b^+]_i\) & 1\\
\(s_{i,2}\) & 1 & 1 & \([\mathbf b^+]_i\) & \([\mathbf b^+]_i\) & 1 & 1\\
\(s_{i,3}\) & 1 & 1 & \([\mathbf b^+]_i\) & 1 & 1 & 1\\
\(s_{i,4}\) & 0 & \(\mathds{1}\{i<\kappa\}\) & 0 & 0 & 0 & 0\\
\(s_{i,5}\) & 0 & \(\mathds{1}\{i>\kappa,\,[\mathbf b]_i=0\}\) & 0 & 0 & 0 & 0\\
\(s_d\) & 1 & 0 & 0 & 1 & 1 & 1\\
\addlinespace[2.5pt]
\(\bar s_i\) & \([\mathbf b]_i\) & \([\mathbf b]_i\) & \([\mathbf b]_i\) & \([\mathbf b]_i\) & 1 & \([\mathbf b^+]_i\)\\
\(\bar s_{i,1}\) & \([\mathbf b]_i\) & \([\mathbf b]_i\) & \([\mathbf b]_i\) & 1 & 1 & \([\mathbf b^+]_i\)\\
\(\bar s_{i,2}\) & \([\mathbf b]_i\) & \([\mathbf b]_i\) & 1 & 1 & 1 & \([\mathbf b^+]_i\)\\
\(\bar s_{i,3}\) & \([\mathbf b]_i\) & 1 & 1 & 1 & 1 & \([\mathbf b^+]_i\)\\
\(\bar s_{i,4}\) & 0 & 0 & 0 & 0 & 1 & 0\\
\addlinespace[2.5pt]
\(s_i^{\mathrm c}\) & \([\mathbf b]_i\) & \(\mathds{1}\{i\le \kappa\}\) & \(\mathds{1}\{i\le \kappa+1\}\) &
\(\mathds{1}\{i\le \kappa+2\}\) & \(\mathds{1}\{i\le \kappa+3\}\) & \([\mathbf b^+]_i\)\\
\(s_i^{\mathrm z}\) & \([\mathbf b]_i\) & 0 & 0 & 0 & 0 & \([\mathbf b^+]_i\)\\
\(s_{i,j}^{\mathrm r}\) & 0 & 0 & 0 & 0 & 1 & 0\\
\bottomrule
\end{tabular}
\endgroup
\end{table}

Appendices~\ref{sec:update-one}--\ref{sec:update-five} derive the five updates
for discounts sufficiently close to one.
Each action-value difference has the form
\(H(\gamma)/(1-\gamma^{(6d+1)L})\), with a positive denominator and an integer
polynomial \(H\) of degree at most \(N\) and coefficient magnitudes at most
\(4r_{\max}\). Appendix~\ref{sec:complete-counter-proof} shows that for every
polynomial satisfying these bounds,
\(\operatorname{sgn}H(\gamma)=\lim_{x\to1^-}\operatorname{sgn}H(x)\) at the
discount in~\eqref{proof:32}. Thus the five updates, including ties, hold at
this discount.
Induction from \(\pi_0=\pi[\mathbf0]\) gives at least
\(5(2^d-2)\) iterations.
The same estimate gives the inequalities in
Lemma~\ref{lem:encoding-values} at this discount, whose encoding length
is \(O(N+\log r_{\max})\).

\begin{remark}[Comparison with Dantzig's rule]
\label{rem:dantzig-comparison}
\normalfont
Howard's policy iteration performs all improving switches using the same policy values,
while the simplex method with Dantzig's pivoting rule re-evaluates after
each switch. Re-evaluation can change which switches are selected,
so the two algorithms can produce different sequences of bit changes.
For example, one can verify that, on the \(d=4\) instance in
Section~\ref{sec:exponential-family}, Dantzig's rule starting from
\(\pi[\mathbf0]\) can change the bits in each group in the order
\(0000\to1000\to1100\to1110\to1111\),
with updates at other states omitted. Each bit changes only once
along this trajectory, whereas Howard's policy iteration repeatedly changes bits
to implement successive binary increments.
\end{remark}

We choose the reward polynomials to express the \(2^d\) iteration bound
from Lemma~\ref{prop:counter} in terms of \(N\) and to meet the stated
bounds on reward magnitudes 
and encoding lengths.

\subsection{\texorpdfstring{Proof of Theorem~\ref{thm:exponential}}{Proof of Theorem \getrefnumber{thm:exponential}}}
\label{sec:exponential-family}

\begin{proof}
Take \(F_i(x)=2^i\) and \(G_i(x)=2^{d+i+1}\). Conditions
\eqref{proof:29} and~\eqref{proof:1} hold with \(\mathsf{f}_i=2^i\) and
\(\mathsf{g}_i=2^{d+i+1}\). Here \(L=2\), so \eqref{proof:33} gives
\(N=\Theta(d)\).
Lemma~\ref{prop:counter} gives at least \(2^d\) iterations,
which proves the \(\exp(\Omega(N))\) lower bound.
Since \(r_{\max}=O(d\,4^d)\), each reward and the common discount
use \(O(d)\) bits, by~\eqref{proof:32}.
Each instance has encoding length \(O(N^2)\) and can be constructed
in time polynomial in \(N\).
\end{proof}

\subsection{\texorpdfstring{Proofs of Theorems~\ref{thm:small-rewards} and~\ref{thm:constant-rewards}}{Proofs of Theorems \getrefnumber{thm:small-rewards} and \getrefnumber{thm:constant-rewards}}}
\label{sec:bounded-rewards}

A high-order zero at one allows us to obtain large limiting ratios
by substituting powers of \(x\), without increasing coefficient magnitudes.
\begin{proof}[Proof of Theorem~\ref{thm:small-rewards}]
We first find a nonzero integer polynomial \(F_0\) such that
\[ \label{eq:requirement}
\deg F_0\le3d^2,\qquad
\max_{j\ge1}|F_{0,j}|\le3d^2,\qquad
F_0^{(r)}(1)=0\qquad\forall\,0\le r<2d,
\]
where \(F_{0,j}\) denotes the coefficient of \(x^{j-1}\),
as defined in Section~\ref{sec:formal-construction}.
Here \(F_0^{(r)}\) denotes its \(r\)-th derivative, with
\(F_0^{(0)}=F_0\).
Let
\[
\mathcal F=\bigg\{\sum_{j=0}^{3d^2}a_jx^j:a_j\in\{0,\ldots,3d^2\}\bigg\} \quad \text{with} \quad |\mathcal F|=(3d^2+1)^{3d^2+1}.
\]
For \(\widetilde F\in\mathcal F\), we define $\mathcal T(\widetilde F)=(\widetilde F(1),\widetilde F'(1),\ldots,
\widetilde F^{(2d-1)}(1))$. 
The entries of \(\mathcal T(\widetilde F)\) are integers satisfying
\[
0\le\widetilde F^{(r)}(1)
=\sum_{j=r}^{3d^2}a_j\,j(j-1)\cdots(j-r+1)
\le(3d^2+1)(3d^2)^{r+1},\qquad 0\le r<2d.
\]
Since each coordinate takes at most
\((3d^2+1)(3d^2)^{r+1}+1<(3d^2+1)^{r+2}\) values and \(d\ge3\), we obtain
\[
|\mathcal T(\mathcal F)|
&<\prod_{r=0}^{2d-1}(3d^2+1)^{r+2}
=(3d^2+1)^{2d^2+3d}<(3d^2+1)^{3d^2+1}=|\mathcal F|.
\]
Hence, two distinct polynomials \(\widetilde F_1,\widetilde F_2\in\mathcal F\)
satisfy \(\mathcal T(\widetilde F_1)=\mathcal T(\widetilde F_2)\).
Both \(\widetilde F_1-\widetilde F_2\) and \(\widetilde F_2-\widetilde F_1\)
are nonzero and satisfy~\eqref{eq:requirement}.
Taylor expansion at \(x=1\) shows that one of them is positive on
\((1-\varepsilon,1)\) for some \(\varepsilon>0\). Choose this difference as \(F_0\).
Set
\[
F_i(x)=F_0(x^{i+1}),\qquad G_i(x)=F_0(x^{d+i+1})
\qquad\forall\,0\le i<d.
\label{proof:41}
\]
Let \(r^\star=\min\{r\ge0:F_0^{(r)}(1)\ne0\}\ge2d\).
Taylor expansion at \(x=1\) gives
\(\lim_{x\to1^-}F_0(x^q)/F_0(x)=q^{r^\star}\)
for every positive integer \(q\).
The limiting weights in~\eqref{proof:29} are therefore
\[
\mathsf f_i=(i+1)^{r^\star},\qquad
\mathsf g_i=(d+i+1)^{r^\star}\qquad\forall\,0\le i<d,\qquad
\mathsf f_0=1.
\label{eq:small-reward-weights}
\]
For \(1\le j\le2d-1\), the binomial theorem and \(r^\star\ge2d\) give
\((1+1/j)^{r^\star}\ge(1+1/j)^j\ge2\).
It follows that
\[
\mathsf f_i\ge2\mathsf f_{i-1},\qquad
\mathsf g_i\ge2\mathsf g_{i-1}\qquad\forall\,1\le i<d,\qquad
\frac{\mathsf g_0}{\mathsf f_{d-1}}
=\Bigl(1+\frac1d\Bigr)^{r^\star}
\ge\biggl[\Bigl(1+\frac1d\Bigr)^d\biggr]^2\ge4.
\label{eq:small-reward-growth}
\]
Combining~\eqref{eq:small-reward-weights} and~\eqref{eq:small-reward-growth}
yields~\eqref{proof:1}. Lemma~\ref{prop:counter} then gives at least \(2^d\) iterations.

We next bound the immediate rewards.
Since \(2d\le r^\star\le\deg F_0\le3d^2\), the definition of \(L\)
and~\eqref{proof:41} give
\[
4d^2+2\le L=2d\deg F_0+2\le6d^3+2.
\label{eq:small-reward-length}
\]
By~\eqref{eq:requirement}, \eqref{proof:41}, and~\eqref{eq:small-reward-length},
every coefficient of \(F_i,G_i\) has magnitude at most \(3d^2<L\).
The reward formulas~\eqref{eq:rewards-s1}--\eqref{eq:rewards-s3}
therefore give \(r_{\max}\le(96d+21)L\), with the largest bound
coming from~\eqref{eq:reward-s15-a1}.
Combining this bound with \(N=(67d+5)L-16d-1\) from~\eqref{proof:33}, we obtain
\[
0\le r(s,a)\le2r_{\max}<4N
\qquad\forall\,s\in\mathcal S,\ a\in\mathcal A_s.
\label{proof:37}
\]
Each reward and successor index uses \(O(\log N)\) bits, while the
common discount in~\eqref{proof:32} uses \(O(N)\) bits.
With at most two actions per state, the total encoding length is \(O(N\log N)\). Finally, combining~\eqref{eq:small-reward-length} with~\eqref{proof:33}
gives \(N=O(d^4)\). The lower bound of \(2^d\) iterations therefore
implies the claimed \(\exp(\Omega(N^{1/4}))\) bound.
\end{proof}

\begin{proof}[Proof of Theorem~\ref{thm:constant-rewards}]
We set \(\ell=\lceil\log_2(2d)\rceil\).
We first find a nonzero integer polynomial \(\widetilde F_0\) such that
\begin{equation}
\deg\widetilde F_0\le16d^2\ell,\qquad
\max_{j\ge1}|\widetilde F_{0,j}|\le1,\qquad
\widetilde F_0^{(r)}(1)=0\qquad\forall\,0\le r<2d.
\label{eq:constant-reward-base}
\end{equation}
As in the proof of Theorem~\ref{thm:small-rewards},
we use the map \(\mathcal T(\widetilde F)=(\widetilde F(1),\widetilde F'(1),
\ldots,\widetilde F^{(2d-1)}(1))\) with
\[
\mathcal F=\bigg\{\sum_{j=0}^{16d^2\ell}a_jx^j:a_j\in\{0,1\}\bigg\}
\quad\text{with}\quad |\mathcal F|=2^{16d^2\ell+1}.
\]
For \(\widetilde F\in\mathcal F\) and \(0\le r<2d\), we have
\(0\le\widetilde F^{(r)}(1)\le(16d^2\ell+1)(16d^2\ell)^r\),
so the \(r\)-th coordinate of \(\mathcal T(\widetilde F)\) takes at most
\(2(16d^2\ell+1)^{r+1}\) values.
Since \(d\le2^{\ell-1}\) and \(\ell\ge3\), we have
\(16d^2\ell+1\le4\ell\,2^{2\ell}+1<2^{4\ell}\).
We obtain
\[
|\mathcal T(\mathcal F)|
\le2^{2d}(16d^2\ell+1)^{d(2d+1)}
<2^{2d+4\ell d(2d+1)}
<2^{16d^2\ell+1}=|\mathcal F|.
\]
As in the proof of Theorem~\ref{thm:small-rewards},
we take the difference of two distinct candidates with the same image
under \(\mathcal T\) and choose its sign to obtain a nonzero
\(\widetilde F_0\) satisfying~\eqref{eq:constant-reward-base}
and \(\widetilde F_0(x)>0\) near \(1\) from below.
We set
\begin{equation}
F_i(x)=(96d+16)\widetilde F_0(x^{96(i+1)}),\qquad
G_i(x)=\biggl(\sum_{h=0}^{96d+15}x^h\biggr)
\widetilde F_0(x^{96(d+i+1)})\qquad\forall\,0\le i<d.
\label{eq:constant-reward-polynomials}
\end{equation}
Let \(r^\star=\min\{r\ge0:\widetilde F_0^{(r)}(1)\ne0\}\ge2d\).
Taylor expansion at \(x=1\) gives
\(\lim_{x\to1^-}\widetilde F_0(x^{96q})/\widetilde F_0(x^{96})
=q^{r^\star}\) for every positive integer \(q\).
Since \((1+x+\cdots+x^{96d+15})/(96d+16)\to1\),
the limits in~\eqref{proof:29} are given
by~\eqref{eq:small-reward-weights}.
As \(r^\star\ge2d\), the inequalities
in~\eqref{eq:small-reward-growth} also hold,
so~\eqref{proof:1} is satisfied.
Lemma~\ref{prop:counter} then gives at least \(2^d\) iterations.

By~\eqref{eq:constant-reward-base}, we have
\(\widetilde F_{0,j}\in\{0,\pm1\}\), so~\eqref{eq:constant-reward-polynomials}
gives \(F_{i,j}\in\{0,\pm(96d+16)\}\) and \(G_{i,j}\in\{0,\pm1\}\).
The reward formulas~\eqref{eq:rewards-s1}, \eqref{eq:rewards-s2},
and~\eqref{eq:rewards-s3} give \(r_{\max}\le44(96d+16)\).
We shift and scale the rewards as follows:
\begin{equation}
r(s,a)\longmapsto\frac{r(s,a)+r_{\max}}{96d+16}
\in\{0,1,\ldots,88\}
\qquad\forall\,s\in\mathcal S,\ a\in\mathcal A_s.
\label{eq:constant-reward-range}
\end{equation}
This preserves the policy sequence.
The encoding length is \(O(N\log N)\), as in the proof of
Theorem~\ref{thm:small-rewards}.
By~\eqref{eq:constant-reward-base}, \eqref{eq:constant-reward-polynomials},
and~\eqref{proof:33}, we have
\[
L=O(d^3\log d),\qquad N=\Theta(dL)=O(d^4\log d).
\]
Since \(N\ge d\), the \(2^d\) iterations imply the claimed
\(\exp(\Omega((N/\log N)^{1/4}))\) lower bound.
\end{proof}

\subsection{An explicit construction with constant reward bit length}
\label{sec:explicit-construction}

The instances in Theorem~\ref{thm:exponential} can be generated in time
polynomial in \(N\). The proofs of Theorems~\ref{thm:small-rewards}
and~\ref{thm:constant-rewards} use a pigeonhole argument to establish
the existence of the reward polynomial.
The following proposition gives an explicit construction with constant
bit length per reward and a weaker iteration lower bound.

\begin{proposition}[An explicit family]
\label{prop:explicit}
There is a family of deterministic discounted MDPs with \(N\) states,
at most two actions per state, and rewards in \(\{0,\ldots,88\}\), on which Howard's policy iteration performs at least
\(N^{\Omega(\log N)}\) iterations from a specified initial policy.
Each instance has encoding length \(O(N\log N)\) and can be constructed
in time polynomial in \(N\).
\end{proposition}

\begin{proof}
We take an integer \(m\ge2\), set \(d=m^2\), and define
\[
F_0(x)=(96d+16)\prod_{j=0}^{m-1}(1-x^{96\cdot2^j}).
\]
Here \(F_0(x)>0\) for \(0<x<1\), and its nonzero coefficients
are \(\pm(96d+16)\) by uniqueness of binary representation.
Let \(q_0=1\) and \(q_{j+1}=\lceil(1+1/m)q_j\rceil\) for \(0\le j<2d\).
We define
\[
F_i(x)=F_0(x^{q_i}),\qquad
G_i(x)=\biggl(\sum_{h=0}^{96d+15}x^h\biggr)
\frac{F_0(x^{q_{d+i+1}})}{96d+16}\qquad\forall\,0\le i<d.
\label{proof:38}
\]
The limit \((1-x^{q_j})/(1-x)\to q_j\) as \(x\to1^-\) gives
\(F_0(x^{q_j})/F_0(x)\to q_j^m\) for \(0\le j\le2d\).
Since \(\sum_{h=0}^{96d+15}x^h\to96d+16\), the limiting weights
in~\eqref{proof:29} are
\[
\mathsf f_i=q_i^m,\qquad \mathsf g_i=q_{d+i+1}^m,
\qquad \mathsf f_0=1.
\]
Since \(q_{j+1}/q_j\ge1+1/m\) and \((1+1/m)^m\ge2\), we obtain
\[
\mathsf f_i\ge2\mathsf f_{i-1},\qquad
\mathsf g_i\ge2\mathsf g_{i-1}\qquad\forall\,1\le i<d,\qquad
\frac{\mathsf g_0}{\mathsf f_{d-1}}
=\left(\frac{q_{d+1}}{q_{d-1}}\right)^m
\ge\left(1+\frac1m\right)^{2m}\ge4.
\]
These weights satisfy~\eqref{proof:1}.
We shift and scale the rewards as in~\eqref{eq:constant-reward-range}.
The polynomial coefficients, transition paths, and discount
in~\eqref{proof:32} can all be generated in time polynomial
in \(N\), and the encoding length is \(O(N\log N)\).
The recurrence gives
\(q_{j+1}+m\le(1+1/m)(q_j+m)\), hence
\(q_{2d}\le(m+1)(1+1/m)^{2d}\le(m+1)e^{2m}\).
By~\eqref{proof:38} and~\eqref{proof:33}, we have
\[
L=2^{\Theta(m)},\qquad
\log N=\log(dL)+O(1)=\Theta(m).
\label{eq:explicit-size}
\]
Lemma~\ref{prop:counter} therefore gives at least
\(2^d=2^{m^2}=N^{\Theta(\log N)}\) iterations,
proving the claim.
\end{proof}

\subsection{Halving the initial optimality gap}
\label{sec:discounted-progress}

The same iteration lower bounds apply to halving the initial optimality
gap at \(s_0\).

\begin{corollary}
\label{cor:progress}
Let \(V^\star\) denote the optimal value function.
For the families constructed in the proofs of Theorems~\ref{thm:exponential}, \ref{thm:small-rewards},
and~\ref{thm:constant-rewards} and Proposition~\ref{prop:explicit},
the corresponding iteration lower bound also holds for attaining
\begin{equation}
V^\star(s_0)-V^{\pi_t}(s_0)
\le\frac12\bigl(V^\star(s_0)-V^{\pi_0}(s_0)\bigr)
\label{eq:half-gap-target}
\end{equation}
from the specified initial policy.
\end{corollary}

\begin{proof}
The common reward shift and positive scaling preserve the policy sequence
and the condition in~\eqref{eq:half-gap-target}.
We therefore work with the rewards before these transformations.
Then \(V^{\pi_0}(s_0)=V^{\pi[\mathbf0]}(s_0)=0\).
Let \(\mathbf b\in\{0,1\}^d\) have \([\mathbf b]_{d-2}=1\) and
all other entries zero.
Lemma~\ref{prop:counter} gives policy \(\pi[\mathbf b]\) after
\(5\cdot2^{d-2}\) iterations. We show that
\(V^{\pi[\mathbf1]}(s_0)>2V^{\pi[\mathbf b]}(s_0)\).
Using~\eqref{proof:29},~\eqref{proof:1}, and~\eqref{eq:encoded-cycle-value},
we obtain
\[
\lim_{\gamma\to1^-}
\frac{V^{\pi[\mathbf1]}(s_0)-2V^{\pi[\mathbf b]}(s_0)}
{4\gamma F_0(\gamma)}
=4\Bigl(\sum_{i=0}^{d-1}\mathsf g_i-2\mathsf g_{d-2}\Bigr)
\ge4\sum_{i=0}^{d-2}\mathsf g_i>0.
\label{eq:gap-comparison-limit}
\]
By~\eqref{eq:encoded-cycle-value},
\(V^{\pi[\mathbf1]}(s_0)-2V^{\pi[\mathbf b]}(s_0)\) has denominator
\(1-\gamma^{(6d+1)L}\) and an integer polynomial numerator of
degree below \(N\), with coefficients of magnitude at most
\(4r_{\max}\). Applying the sign-preservation estimate in
Appendix~\ref{sec:complete-counter-proof} to this numerator and
using~\eqref{eq:gap-comparison-limit}, we obtain
\(V^{\pi[\mathbf1]}(s_0)>2V^{\pi[\mathbf b]}(s_0)\) for
\(\gamma\) in~\eqref{proof:32}.
By monotonicity of policy iteration,
\[
V^{\pi_t}(s_0)
\le V^{\pi[\mathbf b]}(s_0)
<\frac12V^{\pi[\mathbf1]}(s_0)
\le\frac12V^\star(s_0),
\qquad 0\le t\le5\cdot2^{d-2}.
\]
Since \(V^{\pi_0}(s_0)=0\), the target in~\eqref{eq:half-gap-target}
has not been reached at any of these iterations.
\end{proof}

The price of algorithmic anarchy also affects progress toward optimality.
By Corollary~\ref{cor:progress}, Howard's policy iteration can require exponentially
many iterations even to halve the initial optimality gap at \(s_0\),
while Dantzig's rule reaches an optimal policy in polynomially many iterations.

\section{Average-reward MDPs and policy iteration}
\label{sec:average}

For deterministic average-reward MDPs, we use the notation of
Section~\ref{sec:construction} for the state space \(\mathcal S\),
action sets \(\{\mathcal A_s\}_{s\in\mathcal S}\), immediate rewards
\(r(s,a)\), and next states \(f(s,a)\).
We write \(N=|\mathcal S|\) and \({M=\sum_{s\in\mathcal S}|\mathcal A_s|}\).
The transition graph consists of all directed edges \(s\xrightarrow{a}f(s,a)\)
with \(s\in\mathcal S\) and \(a\in\mathcal A_s\).
A cycle \(C\) is a nonempty closed path in this graph with no repeated
state except at its two ends.
We say that a policy \(\pi\) contains \(C\) if \(\pi(s)=a\)
for every edge \((s,a)\in C\).
For a policy \(\pi\), let \(s_{(0)}=s\) and
\(s_{(\tau+1)}=f(s_{(\tau)},\pi(s_{(\tau)}))\) for \(\tau\ge0\).
The \emph{gain} of policy \(\pi\) from state \(s\) is the long-run average reward
\[
J^\pi(s)=\lim_{K\to\infty}\frac1K
\sum_{\tau=0}^{K-1}r(s_{(\tau)},\pi(s_{(\tau)})).
\]
The trajectory eventually repeats a cycle, and \(J^\pi(s)\) is
that cycle's mean reward. Write
\(J^\star(s)=\max_\pi J^\pi(s)\).
The \emph{bias} (relative value) \(V^\pi\) and the \emph{state--action values}
\(Q^\pi\) satisfy
\begin{equation}
V^\pi(s)=r(s,\pi(s))-J^\pi(s)+V^\pi(f(s,\pi(s))),\qquad
Q^\pi(s,a)=r(s,a)-J^\pi(s)+V^\pi(f(s,a)).
\label{eq:av-bias}
\end{equation}
The first identity in~\eqref{eq:av-bias} does not determine \(V^\pi\) uniquely:
adding the same constant to its values at all states that reach a given cycle
leaves the identity unchanged.
To determine the bias uniquely, we fix an ordering of \(\mathcal S\)
and let \(s_C\) denote the first state of each cycle \(C\) in this ordering.
For each policy \(\pi\), we set \(V^\pi(s_C)=0\) for every cycle \(C\)
contained in \(\pi\).
Using the first identity in~\eqref{eq:av-bias}, we then determine the remaining
values by working backward around each cycle and along the paths leading to it.
Under policy \(\pi\), the \emph{residual} of action \(a\) at state \(s\) is
\(Q^\pi(s,a)-V^\pi(s)\), which is zero for \(a=\pi(s)\)
by~\eqref{eq:av-bias}.
We consider Howard's policy iteration \citep{Hansen2012} and a Dantzig counterpart
that selects a single state--action pair with the largest improvement.
Starting from an initial policy \(\pi_0\), let \(\pi_t\) denote
the policy after \(t\) iterations.
\begin{itemize}[leftmargin=*]
\item \emph{Howard's policy iteration.}
Given \(\pi_t\), the algorithm updates all states
simultaneously according to
\begin{equation}
\pi_{t+1}(s)\in\operatorname*{arg\,max}_{a\in\mathcal A_s}
\bigl(J^{\pi_t}(f(s,a))-J^{\pi_t}(s),\,Q^{\pi_t}(s,a)-V^{\pi_t}(s)\bigr)
=\operatorname*{arg\,max}_{a\in\mathcal A_s}
\bigl(J^{\pi_t}(f(s,a)),Q^{\pi_t}(s,a)\bigr).
\label{eq:av-update}
\end{equation}
The maximization is lexicographic: we first maximize the gain difference
between the next state and the current state, then maximize the state--action
value difference relative to the current action among actions attaining
the largest gain difference.
If \(\pi_t(s)\) attains the maximum in~\eqref{eq:av-update},
we set \(\pi_{t+1}(s)=\pi_t(s)\). Otherwise, any maximizer may be chosen.
Policy iteration terminates when \(\pi_{t+1}=\pi_t\).

\item \emph{Dantzig's rule.}
Given \(\pi_t\), the algorithm compares gain differences and bias residuals
across all states and selects any state--action pair attaining the
lexicographic maximum:
\begin{equation}
(s_t,a_t)\in\operatorname*{arg\,max}_{s\in\mathcal S,\,a\in\mathcal A_s}
\bigl(J^{\pi_t}(f(s,a))-J^{\pi_t}(s),\ Q^{\pi_t}(s,a)-V^{\pi_t}(s)\bigr).
\label{eq:av-single-switch}
\end{equation}
If the maximum objective value in~\eqref{eq:av-single-switch}
is \((0,0)\), the algorithm terminates. Otherwise, it sets
\(\pi_{t+1}(s_t)=a_t\) and keeps all other actions unchanged.
\end{itemize}

We first establish monotonicity and optimality properties of the two update rules.
\begin{lemma}[Monotonicity and optimality]
\label{lem:av-monotonicity}
For consecutive policies \(\pi_t,\pi_{t+1}\) generated by either
Howard's policy iteration with update rule~\eqref{eq:av-update} or Dantzig's rule~\eqref{eq:av-single-switch},
we have
\begin{equation}
J^{\pi_{t+1}}(s)\ge J^{\pi_t}(s),\qquad\forall\,s\in\mathcal S.
\label{eq:av-gain-monotonicity}
\end{equation}
For each state \(s\) with \(J^{\pi_{t+1}}(s)=J^{\pi_t}(s)\), we also have
\begin{equation}
V^{\pi_{t+1}}(s)\ge V^{\pi_t}(s),
\label{eq:av-bias-monotonicity}
\end{equation}
with strict inequality if \(\pi_{t+1}(s)\ne\pi_t(s)\).
Both algorithms terminate with an optimal policy.
\end{lemma}
\begin{proof}
Fix \(s\in\mathcal S\),
and let \(C\) be the cycle reached by following \(\pi_{t+1}\) from \(s\).
By~\eqref{eq:av-update} and~\eqref{eq:av-single-switch}, we have
\begin{equation}
J^{\pi_t}(f(s',\pi_{t+1}(s')))\ge J^{\pi_t}(s'),
\qquad\forall\,s'\in\mathcal S.
\label{eq:av-edge-gain}
\end{equation}
Applying~\eqref{eq:av-edge-gain} around \(C\) shows that
\(J^{\pi_t}(f(s',\pi_{t+1}(s')))=J^{\pi_t}(s')\)
for every \(s'\in C\).
By~\eqref{eq:av-update} and~\eqref{eq:av-single-switch}, we therefore have
\(Q^{\pi_t}(s',\pi_{t+1}(s'))\ge V^{\pi_t}(s')\),
with strict inequality if \(\pi_{t+1}(s')\ne\pi_t(s')\).
Summing this inequality over \(s'\in C\) and using the second identity
in~\eqref{eq:av-bias}, we obtain
\[
J^{\pi_{t+1}}(s)
=\frac1{|C|}\sum_{s'\in C}r(s',\pi_{t+1}(s'))
\ge\frac1{|C|}\sum_{s'\in C}J^{\pi_t}(s')
\ge J^{\pi_t}(s).
\]
The last inequality follows by applying~\eqref{eq:av-edge-gain}
along the path under \(\pi_{t+1}\) from \(s\) to each \(s'\in C\).

Now suppose that \(J^{\pi_{t+1}}(s)=J^{\pi_t}(s)\).
For every state \(s'\) on the trajectory from \(s\) under \(\pi_{t+1}\),
by~\eqref{eq:av-gain-monotonicity} and~\eqref{eq:av-edge-gain}, we have
\(J^{\pi_t}(s)
\le J^{\pi_t}(s')
\le J^{\pi_{t+1}}(s')
=J^{\pi_{t+1}}(s)
=J^{\pi_t}(s)\).
Since \(J^{\pi_t}(s')=J^{\pi_t}(s)\) for every \(s'\in C\),
summing the second identity in~\eqref{eq:av-bias} around \(C\) gives
\[
\sum_{s'\in C}
\bigl(Q^{\pi_t}(s',\pi_{t+1}(s'))-V^{\pi_t}(s')\bigr)
=|C|\bigl(J^{\pi_{t+1}}(s)-J^{\pi_t}(s)\bigr)=0.
\]
Since each summand is nonnegative and is strictly positive if
\(\pi_{t+1}(s')\ne\pi_t(s')\), we have
\(\pi_{t+1}(s')=\pi_t(s')\) for every \(s'\in C\).
By the first identity in~\eqref{eq:av-bias} and
\(V^{\pi_t}(s_C)=V^{\pi_{t+1}}(s_C)=0\), we then obtain
\(V^{\pi_{t+1}}(s')=V^{\pi_t}(s')\) for every \(s'\in C\).
For \(s'\notin C\) on this trajectory, suppose that
\(V^{\pi_{t+1}}(f(s',\pi_{t+1}(s')))\ge V^{\pi_t}(f(s',\pi_{t+1}(s')))\).
By~\eqref{eq:av-bias} and \(J^{\pi_{t+1}}(s')=J^{\pi_t}(s')\), we obtain
\[
V^{\pi_{t+1}}(s')
\ge Q^{\pi_t}(s',\pi_{t+1}(s'))
\ge V^{\pi_t}(s'),
\]
where the second inequality follows from~\eqref{eq:av-update}
and~\eqref{eq:av-single-switch} because
\(J^{\pi_t}(f(s',\pi_{t+1}(s')))=J^{\pi_t}(s')\),
with strict inequality if \(\pi_{t+1}(s')\ne\pi_t(s')\).
Applying this argument backward from \(C\) to \(s\)
proves~\eqref{eq:av-bias-monotonicity}.

By~\eqref{eq:av-gain-monotonicity} and~\eqref{eq:av-bias-monotonicity},
no policy can recur, so finiteness of the policy set implies termination.
At the final policy \(\pi_t\), by~\eqref{eq:av-update}
and~\eqref{eq:av-single-switch}, we have
\(J^{\pi_t}(f(s,a))\le J^{\pi_t}(s)\) for every state--action pair,
and \(Q^{\pi_t}(s,a)\le V^{\pi_t}(s)\) whenever
\(J^{\pi_t}(f(s,a))=J^{\pi_t}(s)\).
For any \(s\in\mathcal S\) and any cycle \(C\) reachable from \(s\),
applying the gain inequality along paths to \(C\) and around \(C\) gives
\(J^{\pi_t}(f(s',a))=J^{\pi_t}(s')\le J^{\pi_t}(s)\)
for every \((s',a)\in C\).
Hence \(Q^{\pi_t}(s',a)\le V^{\pi_t}(s')\) for every \((s',a)\in C\).
Summing the second identity in~\eqref{eq:av-bias} around \(C\), we obtain
\[
\frac1{|C|}\sum_{(s',a)\in C}r(s',a)
\le\frac1{|C|}\sum_{s'\in C}J^{\pi_t}(s')
\le J^{\pi_t}(s).
\]
Every policy starting from \(s\) eventually follows a cycle reachable from \(s\).
Hence \(J^{\pi_t}(s)=J^\star(s)\) for every \(s\in\mathcal S\),
so \(\pi_t\) is optimal.
\end{proof}

\section{Lower bounds for average-reward MDPs}
\label{sec:av-proof}

We next establish an exponential lower bound for Howard's policy iteration
with update rule~\eqref{eq:av-update}.

\begin{theorem}
\label{thm:average}
There is a family of deterministic average-reward MDPs with \(N\) states
and at most four actions per state, on which Howard's policy iteration
performs at least \(\exp(\Omega(N))\) iterations from a specified
initial policy before any state attains its optimal gain.
Each instance has a strongly connected transition graph and encoding
length \(O(N^2)\).
\end{theorem}

Theorem~\ref{thm:average} shows that Howard's policy iteration is not
a polynomial-time algorithm for deterministic average-reward MDPs, even with
at most four actions per state. It can require exponentially many
iterations even to produce a policy containing an optimal cycle.
It also refutes the conjecture that Howard's policy iteration terminates within
\(M\) iterations on every deterministic average-reward MDP with
\(M\) state--action pairs or, equivalently, on every weighted directed
graph with \(M\) edges
(\citealp[Conjecture~6.1.1]{Hansen2012};
\citealp[Conjecture~3.1]{AsadiEtAl2025}).

\subsection{Formal construction}
\label{sec:av-construction}

We fix an integer \(d\ge1\) and use the ranges
\(0\le i<d\), \(0\le k<3d\), and \(0\le j\le4\)
throughout the construction, unless stated otherwise.

\paragraph{States and actions.}
The state space is \(\mathcal S=\mathcal S_1\cup\mathcal S_2\), where
\begin{enumerate}[leftmargin=*]
\item \(\mathcal S_1=\{s_{i,j},s'_{i,j}:0\le i<d,\ 0\le j\le4\}
\cup\{s_{i,5}:0\le i<d\}\), with \(|\mathcal S_1|=11d\).

\item \(\mathcal S_2=\{\bar s_{k,j},\bar s'_{k,j}:0\le k<3d,\ 0\le j\le4\}
\cup\{\bar s_{k,5}:0\le k<3d\}\), with \(|\mathcal S_2|=33d\).
\end{enumerate}
We may omit the second index when it is zero, writing
\(s_i=s_{i,0}\), \(s'_i=s'_{i,0}\), \(\bar s_k=\bar s_{k,0}\),
and \(\bar s'_k=\bar s'_{k,0}\).
We set \(\mathcal A_s=\{a_0,a_1,a_2,a_3\}\) for
\(s\in\mathcal S_1\) and \(\mathcal A_s=\{a_0,a_1,a_2\}\)
for \(s\in\mathcal S_2\), giving
\begin{equation}
N=44d,\qquad M=143d.
\label{eq:av-state-count}
\end{equation}

\paragraph{Transitions.}
Actions \(a_0,a_1\) connect states within \(\mathcal S_1\) or
\(\mathcal S_2\), while \(a_2,a_3\) in \(\mathcal S_1\) and
\(a_2\) in \(\mathcal S_2\) connect the two sets.
We specify the transitions by action below.
Figure~\ref{fig:av-branches} illustrates the transitions in \(\mathcal S_1\).

\suppressfloats[t]
\begin{figure}[t]
\centering
\begingroup
\tikzset{
 state/.style={paper state,text height=1.8ex,text depth=0.65ex},
 path arrow/.style={paper arrow},
 teal path/.style={path arrow,draw=paperTeal},
 bronze path/.style={path arrow,draw=paperBronze},
 path label/.style={paper label}
}
\begin{minipage}[t]{0.55\linewidth}
\centering
\begin{tikzpicture}[x=1pt,y=1pt,font=\small]
\path[use as bounding box] (-18,-56) rectangle (234,80);
\node[paper panel] at (108,70)
 {(a) Transitions within \(\mathcal S_1\)};
\node[state] (zeroone) at (0,35) {$s_i$};
\node[state] (zerolast) at (90,35) {$s_{i,4}$};
\node[state] (oneone) at (0,-35) {$s'_i$};
\node[state] (onelast) at (90,-35) {$s'_{i,4}$};
\node[state] (shared) at (150,0) {$s_{i,5}$};
\node[state] (nextzero) at (216,35) {$s_{i+1}$};
\node[state] (nextone) at (216,-35) {$s'_{i+1}$};
\draw[teal path] (zeroone.east) --
 node[path label,above=3pt] {$a_1$}
 node[path label] {$\cdots$} (zerolast.west);
\draw[teal path] (oneone.east) --
 node[path label,below=3pt] {$a_1$}
 node[path label] {$\cdots$} (onelast.west);
\draw[teal path] (zerolast) --
 node[path label,pos=0.43,above right=2pt] {$a_1$} (shared);
\draw[teal path] (onelast) --
 node[path label,pos=0.43,below right=2pt] {$a_1$} (shared);
\draw[bronze path] (zeroone) --
 node[path label,pos=0.43,below=2pt] {$a_0$} (shared);
\draw[bronze path] (oneone) --
 node[path label,pos=0.43,above=2pt] {$a_0$} (shared);
\draw[bronze path] (shared) --
 node[path label,above=3pt] {$a_0$} (nextzero);
\draw[teal path] (shared) --
 node[path label,below=3pt] {$a_1$} (nextone);
\end{tikzpicture}
\end{minipage}\hfill
\begin{minipage}[t]{0.42\linewidth}
\centering
\begin{tikzpicture}[x=1pt,y=1pt,font=\small]
\path[use as bounding box] (-18,-56) rectangle (176,80);
\node[paper panel] at (79,70)
 {(b) Actions from \(\mathcal S_1\) to \(\mathcal S_2\)};
\node[state] (zeroj) at (26,35) {$s_{i,j}$};
\node[state] (onej) at (132,35) {$s'_{i,j}$};
\node[state] (zeromiddle) at (0,-35) {$\bar s_{d+i}$};
\node[state] (onefirst) at (52,-35) {$\bar s'_i$};
\node[state] (onemiddle) at (106,-35) {$\bar s'_{d+i}$};
\node[state] (zerofirst) at (158,-35) {$\bar s_i$};
\draw[bronze path] (zeroj) --
 node[path label,pos=0.5,left=2pt] {$a_2$} (zeromiddle);
\draw[teal path] (zeroj) --
 node[path label,pos=0.5,right=2pt] {$a_3$} (onefirst);
\draw[bronze path] (onej) --
 node[path label,pos=0.5,left=2pt] {$a_2$} (onemiddle);
\draw[teal path] (onej) --
 node[path label,pos=0.5,right=2pt] {$a_3$} (zerofirst);
\end{tikzpicture}
\end{minipage}
\endgroup
\caption{Transitions in the average-reward construction.
In (a), the dots omit the states with \(j=1,2,3\).
The \(a_0\) arrows from \(s_{i,j},s'_{i,j}\) to \(s_{i,5}\)
for \(0\le j\le4\) are shown only for \(j=0\).
Here \(i+1\) is taken modulo \(d\).
In (b), the arrows show the \(a_2,a_3\) transitions
for \(0\le j\le4\).
The corresponding transitions at \(s_{i,5}\)
are given in~\eqref{eq:av-shared-exits}.
The \(a_0,a_1\) transitions in \(\mathcal S_2\)
have the same form as in (a),
with position indices taken modulo \(3d\).}
\label{fig:av-branches}
\end{figure}

\begin{itemize}[leftmargin=*,topsep=3pt,itemsep=4pt,parsep=0pt]
\item \emph{\(a_0\).}
Action \(a_0\) leads directly from the states with \(0\le j\le4\)
to \(s_{i,5}\) or \(\bar s_{k,5}\), i.e., we set
\begin{equation*}
f(s_{i,j},a_0)=f(s'_{i,j},a_0)=s_{i,5},\qquad
f(\bar s_{k,j},a_0)=f(\bar s'_{k,j},a_0)=\bar s_{k,5}.
\end{equation*}
At \(s_{i,5}\) and \(\bar s_{k,5}\), action \(a_0\) leads to the next
position, with \(f(s_{i,5},a_0)=s_{i+1}\) and
\(f(\bar s_{k,5},a_0)=\bar s_{k+1}\). Here \(i+1\) and \(k+1\)
are taken modulo \(d\) and \(3d\), respectively.

\item \emph{\(a_1\).}
Action \(a_1\) follows the two paths at each position, which both end
at \(s_{i,5}\) or \(\bar s_{k,5}\). For \(0\le j<4\), we set
\begin{equation}
f(s_{i,j},a_1)=s_{i,j+1},\quad
f(s'_{i,j},a_1)=s'_{i,j+1},\quad
f(\bar s_{k,j},a_1)=\bar s_{k,j+1},\quad
f(\bar s'_{k,j},a_1)=\bar s'_{k,j+1}.
\label{eq:av-branch-transitions}
\end{equation}
At states with \(j=4\), we set
\begin{equation*}
f(s_{i,4},a_1)=f(s'_{i,4},a_1)=s_{i,5},\qquad
f(\bar s_{k,4},a_1)=f(\bar s'_{k,4},a_1)=\bar s_{k,5}.
\end{equation*}
At \(s_{i,5}\) and \(\bar s_{k,5}\), action \(a_1\) leads to the next
position, with \(f(s_{i,5},a_1)=s'_{i+1}\) and
\(f(\bar s_{k,5},a_1)=\bar s'_{k+1}\).

\item \emph{\(a_2,a_3\).}
Actions \(a_2,a_3\) in \(\mathcal S_1\), and \(a_2\) in
\(\mathcal S_2\), lead directly to the other state set.
In \(\mathcal S_1\), for \(0\le j\le4\), we set
\begin{subequations}
\label{eq:av-exit-transitions}
\noeqref{eq:av-path-exits,eq:av-shared-exits}
\begin{empheq}[left=\empheqlbrace]{align}
f(s_{i,j},a_2)&=\bar s_{d+i},\quad
f(s'_{i,j},a_2)=\bar s'_{d+i},\quad
f(s_{i,j},a_3)=\bar s'_i,\quad
f(s'_{i,j},a_3)=\bar s_i,
\label{eq:av-path-exits}\\
f(s_{i,5},a_2)&=\bar s_{d+i-1,5},\quad
f(s_{i,5},a_3)=\bar s_{i-1,5}.
\label{eq:av-shared-exits}
\end{empheq}
\end{subequations}
Here \(d+i-1\) and \(i-1\) in~\eqref{eq:av-shared-exits}
are taken modulo \(3d\).
In \(\mathcal S_2\), for \(k\in\{i,d+i,2d+i\}\) and
\(0\le j\le4\), we set
\begin{equation*}
f(\bar s_{k,j},a_2)=s_i,\qquad
f(\bar s'_{k,j},a_2)=s'_i,\qquad
f(\bar s_{k,5},a_2)=s_{i-1,5}.
\end{equation*}
At \(\bar s_{k,5}\), action \(a_2\) leads to the preceding position
in \(\mathcal S_1\).
\(i-1\) is taken modulo \(d\).
\end{itemize}

\paragraph{Rewards.}
All unspecified rewards are zero.
At every state with \(0\le j\le4\), we set
\begin{equation}
r(s,a_0)=-24,\qquad
\forall\,s\in\{s_{i,j},s'_{i,j},\bar s_{k,j},\bar s'_{k,j}\}.
\label{eq:av-private-return}
\end{equation}
Action \(a_1\) has zero reward except at states with \(j=4\). We set
\begin{subequations}
\label{eq:av-branch-rewards}
\noeqref{eq:av-reward-j4-s1,eq:av-reward-j4-s2-negative,
eq:av-reward-j4-s2-positive,eq:av-reward-j4-s2-bit,eq:av-reward-j4-zero}
\begin{empheq}[left={r(s,a_1)=\empheqlbrace}]{alignat=2}
&2^{i-d},\quad &&s=s'_{i,4},
\label{eq:av-reward-j4-s1}\\
&-16,\quad &&s=\bar s_{d+i,4},
\label{eq:av-reward-j4-s2-negative}\\
&16+\dfrac{3}{d2^{d+1}},\quad &&s=\bar s_{2d+i,4},
\label{eq:av-reward-j4-s2-positive}\\
&3\cdot2^{i-d}+\dfrac{3}{d2^{d+1}},\quad &&s=\bar s'_{2d+i,4},
\label{eq:av-reward-j4-s2-bit}\\
&0,\quad &&\text{otherwise}.
\label{eq:av-reward-j4-zero}
\end{empheq}
\end{subequations}
Following the \(a_1\) path to \(s_{i,5}\) or \(\bar s_{k,5}\)
gives total reward at least \(-16\), whereas taking \(a_0\) reaches
the same state with reward \(-24\).
We set the rewards of \(a_2,a_3\) as follows.
\begin{enumerate}[leftmargin=*]
\item \emph{States in \(\mathcal S_1\).}
For \(0\le j\le4\), we set
\begin{subequations}
\label{eq:av-rewards-s1}
\noeqref{eq:av-reward-s1-a23,eq:av-reward-s1-shared}
\begin{empheq}[left=\empheqlbrace]{align}
r(s,a_2)&=64(i-4d)+j,\quad
r(s,a_3)=64(i-4d)+j+8,
\quad s\in\{s_{i,j},s'_{i,j}\},
\label{eq:av-reward-s1-a23}\\
r(s_{i,5},a_2)&=64(i-4d)+24,\quad
r(s_{i,5},a_3)=64(i-4d)+32.
\label{eq:av-reward-s1-shared}
\end{empheq}
\end{subequations}

\item \emph{States in \(\mathcal S_2\).}
For \(0\le j\le4\), we set
\begin{equation}
r(\bar s_{k,j},a_2)=r(\bar s'_{k,j},a_2)=64(k-4d)+j,\qquad
r(\bar s_{k,5},a_2)=64(k-4d)+24.
\label{eq:av-exit-rewards}
\end{equation}
\end{enumerate}

\paragraph{Initial policy.}
We start from the policy
\begin{equation}
\pi_0(s)=
\begin{cases}
a_1,&s=s_{i,j},\quad 0\le j\le4,\\
a_2,&s\in\mathcal S_2,\\
a_0,&\text{otherwise}.
\end{cases}
\label{eq:av-initial-policy}
\end{equation}
Under \(\pi_0\), the paths
\(s_i\to s_{i,1}\to\cdots\to s_{i,4}\to s_{i,5}\)
form a cycle with zero rewards.
Every \(s'_{i,j}\) takes \(a_0\) into this cycle, and every state in
\(\mathcal S_2\) takes \(a_2\) into \(\mathcal S_1\) and then reaches
the cycle. Thus the cycle is unique and \(J^{\pi_0}(s)=0\)
for every \(s\in\mathcal S\).

\paragraph{Binary encoding.}
We use \(\mathbf b\in\{0,1\}^d\) and its binary successor
\(\mathbf b^+\) as in Section~\ref{sec:counter-interpretation}.
For \(0\le n<2^{d+1}\), let \(\mathbf b\) be the binary representation of
\(\lfloor n/2\rfloor\).
We set \(\ell(n)=1\) when \(n\) is even and \(\ell(n)=2\) when \(n\) is odd.
We define the \emph{encoded cycle} \(C_n\) in \(\mathcal S_{\ell(n)}\) as follows:
\begin{itemize}[leftmargin=*,topsep=3pt,itemsep=2pt,parsep=0pt]
\item If \(n\) is even (\(\ell(n)=1\)), \(C_n\) is the cycle of length \(6d\) in
\(\mathcal S_1\) that follows
\(s_i\to\cdots\to s_{i,4}\to s_{i,5}\) when \([\mathbf b]_i=0\), and
\(s'_i\to\cdots\to s'_{i,4}\to s_{i,5}\) when \([\mathbf b]_i=1\),
for \(0\le i<d\).
\item If \(n\) is odd (\(\ell(n)=2\)), \(C_n\) is the cycle of length \(18d\) in
\(\mathcal S_2\) that follows
\(\bar s_k\to\cdots\to\bar s_{k,4}\to\bar s_{k,5}\) when \([\mathbf b]_i=0\),
and \(\bar s'_k\to\cdots\to\bar s'_{k,4}\to\bar s_{k,5}\) when
\([\mathbf b]_i=1\), for \(0\le i<d\) and \(k\in\{i,d+i,2d+i\}\).
\end{itemize}
In both cases, the cycle uses \(a_1\) at states with \(0\le j\le4\), and
\(a_{[\mathbf b]_{i+1}}\) at states with \(j=5\),
with bit indices taken modulo \(d\).
The initial policy \(\pi_0\) has \(C_0\) as its unique cycle.
If a policy \(\pi\) contains \(C_n\),
averaging the rewards in~\eqref{eq:av-branch-rewards} over \(C_n\) gives
\begin{equation}
J^\pi(s)=\frac1{6d}\sum_{i=0}^{d-1}2^{i-d}[\mathbf b]_i+
\frac{\ell(n)-1}{6d\,2^{d+1}}
=\frac{n}{6d\,2^{d+1}}\in\left[0,\frac1{6d}\right),
\label{eq:av-cycle-means}
\end{equation}
for every state \(s\) reaching \(C_n\) under \(\pi\), where the last equality uses
\(n=2\sum_{i=0}^{d-1}2^i[\mathbf b]_i+\ell(n)-1\).
Thus the encoded cycles have strictly increasing mean rewards,
with consecutive differences \(2^{-d-1}/(6d)\).

As in the discounted setting, our construction is designed to make policy iteration follow
an exponentially long sequence of encoded cycles.
The transitions and rewards are arranged to make policy iteration copy
the current encoding \(\mathbf b\) from \(\mathcal S_1\) to \(\mathcal S_2\).
By~\eqref{eq:av-cycle-means}, the copied cycle has a larger mean reward,
so states in \(\mathcal S_1\) improve their gains by switching
to the cycle in \(\mathcal S_2\).
While the cycle encoding \(\mathbf b\) in \(\mathcal S_2\) remains unchanged,
policy iteration uses the rewards along paths through \(\mathcal S_2\)
to determine which bits must flip to form the cycle encoding
\(\mathbf b^+\) in \(\mathcal S_1\).
The larger mean reward of the new cycle gives states in \(\mathcal S_2\)
a gain improvement by switching to the cycle in \(\mathcal S_1\).
The next subsection states this counting property and outlines its proof.

\subsection{Binary counting under policy iteration}
\label{sec:av-successive-cycles}

\begin{lemma}[Binary counting]
\label{lem:av-counting}
For the construction in Section~\ref{sec:av-construction}, Howard's
policy iteration with update rule~\eqref{eq:av-update}, starting from \(\pi_0\) in~\eqref{eq:av-initial-policy},
produces a policy \(\pi_{3n-1}\) containing the cycle \(C_n\)  for every
\(1\le n\le2^{d+1}-2\).
\end{lemma}

\begingroup
\noindent\textbf{Proof overview.}
For \(0\le n\le2^{d+1}-3\), suppose inductively that \(C_n\)
is the unique cycle~of \(\pi_{3n}\).
The first two updates below create \(C_{n+1}\) in the other group.
When \(n<2^{d+1}-3\), the third restores the induction hypothesis at
\(\pi_{3(n+1)}\).

\begin{enumerate}[label=\textbf{Iteration \arabic*.},wide=0pt,
 labelsep=.6em,itemsep=5pt,topsep=5pt,parsep=0pt]
\item \emph{Select the next paths.}
All states have the gain of \(C_n\), so the bias comparisons determine the
action choices. The update retains \(C_n\) and selects all but the closing
edge of \(C_{n+1}\), copying the encoding for even \(n\) and incrementing it
for odd \(n\).
\item \emph{Close the new cycle.}
The higher mean of \(C_{n+1}\) gives its closing edge a positive residual.
The update selects this edge, leaving exactly the two cycles \(C_n,C_{n+1}\).
\item \emph{Make the new cycle unique.}
For \(n<2^{d+1}-3\), the higher gain of \(C_{n+1}\) makes states in the
old group switch to the new group. The resulting policy has \(C_{n+1}\)
as its unique cycle and restores the induction hypothesis.
\end{enumerate}

Appendix~\ref{app:average-counting} verifies these updates and the induction
step. Thus \(C_{n+1}\) first appears at iteration \(3n+2\), and every
policy cycle up to the creation of \(C_{2^{d+1}-2}\) is encoded.
\endgroup

\subsection{\texorpdfstring{Completing the proof of Theorem~\ref{thm:average}}{Completing the proof of Theorem \getrefnumber{thm:average}}}
\label{sec:av-optimal-gain}

\begin{proof}[Proof of Theorem~\ref{thm:average}]
The \(a_0,a_1\) transitions allow every state in each set to reach
every other state of that set, and the \(a_2\) transitions connect the
two sets, so the transition graph is strongly connected.
We consider the cycle in \(\mathcal S_2\) formed by the \(a_1\) paths through
\(\bar s_{i,4}\), \(\bar s'_{d+i,4}\), and \(\bar s_{2d+i,4}\),
for \(0\le i<d\).
By~\eqref{eq:av-reward-j4-s2-positive}
and~\eqref{eq:av-reward-j4-zero}, this cycle receives reward
at least \(16\) at each \(\bar s_{2d+i,4}\) and zero elsewhere.
Its mean reward is therefore at least \(16d/(18d)=8/9\).
We choose a policy that follows this cycle and, at each state outside it,
takes an action decreasing the shortest distance to it, so every
state reaches the cycle. Hence, we have
\begin{equation}
J^\star(s)\ge\frac{8}{9},\qquad\forall\,s\in\mathcal S.
\label{eq:av-optimal-gain}
\end{equation}
For \(0\le t\le6\cdot2^d-7\), the proof of Lemma~\ref{lem:av-counting}
shows that every cycle of \(\pi_t\) is encoded.
Combining~\eqref{eq:av-cycle-means} and~\eqref{eq:av-optimal-gain}, we obtain
\begin{equation}
J^{\pi_t}(s)<\frac1{6d}<\frac12J^\star(s),\qquad
\forall\,s\in\mathcal S,\quad 0\le t\le6\cdot2^d-7.
\label{eq:av-before-optimum}
\end{equation}
Since \(N=44d\) by~\eqref{eq:av-state-count},
\eqref{eq:av-before-optimum} gives the claimed
\(\exp(\Omega(N))\) lower bound.

\vspace{5pt}\noindent\textbf{Integer rewards and encoding length.}
We replace \(r(s,a)\) by \(d2^{d+1}(r(s,a)+256d)\).
Since \(-256d\le r(s,a)<17\) and every reward denominator divides
\(d2^{d+1}\), the transformed rewards are nonnegative integers
with \(O(N)\) bits.
This transformation preserves the policy sequence and scales
all optimality gaps by the same positive factor.
With \(M=143d=O(N)\) and \(O(\log N)\) bits per successor index,
the total encoding length is \(O(N^2)\).
The instance and initial policy can be constructed in time
polynomial in \(N\).
\end{proof}

Since \(J^{\pi_0}(s_0)=0\) for the original rewards,
\eqref{eq:av-before-optimum} also gives an exponential lower bound
for halving the initial optimality gap, as in Corollary~\ref{cor:progress}.
\begin{corollary}
\label{cor:av-progress}
\looseness=-1 For the family in Theorem~\ref{thm:average}, policy iteration requires \(\exp(\Omega(N))\) iterations to attain
\[
J^\star(s_0)-J^{\pi_t}(s_0)
\le\frac12\bigl(J^\star(s_0)-J^{\pi_0}(s_0)\bigr).
\]
\end{corollary}
\section{Upper bounds for average-reward MDPs}
\label{sec:av-upper}

We first prove that the Dantzig counterpart~\eqref{eq:av-single-switch}
is strongly polynomial.
We then give an iteration bound for Howard's policy iteration with nonnegative
integer rewards in terms of their bit length.

\subsection{Strong polynomiality of Dantzig's rule}
\label{sec:av-dantzig-upper}

\begin{theorem}
\label{thm:av-upper}
On any deterministic average-reward MDP, Dantzig's rule~\eqref{eq:av-single-switch}
finds an optimal policy within \(O(N^3M^2\log^2(2N))\) iterations
from every initial policy.
\end{theorem}

The instances in Theorem~\ref{thm:average} have \(M\le4N\),
so the bound in Theorem~\ref{thm:av-upper} is polynomial in \(N\).
The two theorems therefore establish an exponential separation
between Howard's policy iteration and its Dantzig counterpart,
showing that the price of algorithmic anarchy also arises
in the average-reward setting.

\begin{proof}[Proof of Theorem~\ref{thm:av-upper}]
By~\eqref{eq:av-single-switch}, every update is either
(i) a \emph{gain update}, with \(J^{\pi_t}(f(s_t,a_t))>J^{\pi_t}(s_t)\),
or (ii) a \emph{bias update}, with
\(J^{\pi_t}(f(s_t,a_t))=J^{\pi_t}(s_t)\) and
\(Q^{\pi_t}(s_t,a_t)>V^{\pi_t}(s_t)\).
We first bound gain updates, then adapt the two action-elimination
arguments of \citet{PostYe2015}
to bound bias updates with unchanged cycles and the number of new cycles.

\vspace{5pt}\noindent\textbf{Gain updates.}
The update in~\eqref{eq:av-single-switch} sets \(\pi_{t+1}(s_t)=a_t\)
and leaves all other actions unchanged.
If the path under \(\pi_t\) from \(f(s_t,a_t)\) reaches \(s_t\),
then this path and the new transition \(s_t\to f(s_t,a_t)\) form a cycle.
In this case we have
\(J^{\pi_t}(f(s_t,a_t))=J^{\pi_t}(s_t)\).
Otherwise the path avoids \(s_t\) and remains unchanged.
Policy \(\pi_{t+1}\) takes \(a_t\) at \(s_t\)
and then follows this path from \(f(s_t,a_t)\), so
\(J^{\pi_{t+1}}(s_t)=J^{\pi_t}(f(s_t,a_t))\).
A gain update therefore creates no cycle and strictly increases the gain
at \(s_t\).
By~\eqref{eq:av-gain-monotonicity}, the gain at each state never decreases.
Thus successive gain updates at a fixed state give distinct gain values.
Each value is the mean reward of a cycle present initially or created
during the updates.
There are at most \(N\) initial cycles, so summing over all \(N\) states gives
\begin{equation}
\#\{\text{gain updates}\}
\le N\bigl(N+\#\{\text{cycle creations}\}\bigr).
\label{eq:av-gain-update-count}
\end{equation}

\vspace{5pt}\noindent\textbf{Bias updates with unchanged cycles.}
For a bias update that creates no cycle, we have
\(J^{\pi_{t+1}}(s_t)=J^{\pi_t}(f(s_t,a_t))=J^{\pi_t}(s_t)\).
Only the action at \(s_t\) changes, so states that reach \(s_t\)
under \(\pi_t\) still reach it under \(\pi_{t+1}\)
and all other states follow unchanged paths.
Thus we have \(J^{\pi_{t+1}}(s)=J^{\pi_t}(s)\) for every \(s\in\mathcal S\).
By~\eqref{eq:av-single-switch}, we therefore have
\(J^{\pi_{t+1}}(f(s,a))=J^{\pi_t}(f(s,a))
\le J^{\pi_t}(s)=J^{\pi_{t+1}}(s)\) for every state--action pair.
Thus bias updates continue until a new cycle is created
or the algorithm terminates.

Now consider consecutive bias updates that neither create nor destroy
a cycle, and let \(\pi_T\) be the policy after the last update.
For every policy \(\pi_t\) in this sequence, we have
\(J^{\pi_t}=J^{\pi_T}\).
By~\eqref{eq:av-bias-monotonicity}, we therefore have
\(V^{\pi_t}(s)\le V^{\pi_T}(s)\) for every \(s\in\mathcal S\).
Let \(C\) be any cycle shared by \(\pi_t\) and \(\pi_T\).
We have \(V^{\pi_t}(s)=V^{\pi_T}(s)\) for every \(s\in C\) because
both policies use the same actions on \(C\).
For each \(s\in\mathcal S\), let \(s_{(0)}=s,\ldots,s_{(\ell(s))}\)
be the path under \(\pi_T\) from \(s\) to its first state on a cycle,
where \(\ell(s)<N\).
Since \(J^{\pi_t}=J^{\pi_T}\), we have
\(J^{\pi_t}(s_{(\tau+1)})=J^{\pi_t}(s_{(\tau)})\) along this path.
For \(0\le \tau<\ell(s)\), we obtain from both identities in~\eqref{eq:av-bias} that
\[
\begin{aligned}
Q^{\pi_t}(s_{(\tau)},\pi_T(s_{(\tau)}))-V^{\pi_t}(s_{(\tau)})
&=r(s_{(\tau)},\pi_T(s_{(\tau)}))-J^{\pi_t}(s_{(\tau)})
  +V^{\pi_t}(s_{(\tau+1)})-V^{\pi_t}(s_{(\tau)})\\
&=V^{\pi_T}(s_{(\tau)})-V^{\pi_T}(s_{(\tau+1)})
  +V^{\pi_t}(s_{(\tau+1)})-V^{\pi_t}(s_{(\tau)}).
\end{aligned}
\]
Using \(V^{\pi_t}(s_{(\ell(s))})=V^{\pi_T}(s_{(\ell(s))})\),
we sum over \(s\) and \(\tau\) to obtain
\begin{equation}
\sum_{s\in\mathcal S}\bigl(V^{\pi_T}(s)-V^{\pi_t}(s)\bigr)
=\sum_{s\in\mathcal S}\sum_{\tau=0}^{\ell(s)-1}
\bigl(Q^{\pi_t}(s_{(\tau)},\pi_T(s_{(\tau)}))-V^{\pi_t}(s_{(\tau)})\bigr)
\le N^2\bigl(Q^{\pi_t}(s_t,a_t)-V^{\pi_t}(s_t)\bigr),
\label{eq:av-path-residual}
\end{equation}
where the inequality follows from~\eqref{eq:av-single-switch} and \(\ell(s)<N\).
Since \(\pi_{t+1}(s_t)=a_t\) and \(J^{\pi_{t+1}}(s_t)=J^{\pi_t}(s_t)\),
we obtain from~\eqref{eq:av-bias} that
\begin{equation}
\begin{aligned}
V^{\pi_{t+1}}(s_t)
&=r(s_t,a_t)-J^{\pi_t}(s_t)+V^{\pi_{t+1}}(f(s_t,a_t))\\
&\ge r(s_t,a_t)-J^{\pi_t}(s_t)+V^{\pi_t}(f(s_t,a_t))
=Q^{\pi_t}(s_t,a_t),
\end{aligned}
\label{eq:av-selected-bias-increase}
\end{equation}
where the inequality follows from~\eqref{eq:av-bias-monotonicity}
at \(f(s_t,a_t)\).
By~\eqref{eq:av-bias-monotonicity}, we also have
\(V^{\pi_{t+1}}(s)-V^{\pi_t}(s)\ge0\) for every \(s\in\mathcal S\).
Thus, for consecutive policies \(\pi_t,\pi_{t+1}\) in the sequence, we obtain
\begin{equation}
\begin{aligned}
\sum_s\bigl(V^{\pi_T}(s)-V^{\pi_{t+1}}(s)\bigr)
&=\sum_s\bigl(V^{\pi_T}(s)-V^{\pi_t}(s)\bigr)
-\sum_s\bigl(V^{\pi_{t+1}}(s)-V^{\pi_t}(s)\bigr)\\
&\le\sum_s\bigl(V^{\pi_T}(s)-V^{\pi_t}(s)\bigr)
-\bigl(Q^{\pi_t}(s_t,a_t)-V^{\pi_t}(s_t)\bigr)\\
&\le\left(1-\frac1{N^2}\right)
\sum_s\bigl(V^{\pi_T}(s)-V^{\pi_t}(s)\bigr),
\end{aligned}
\label{eq:av-transient-contraction}
\end{equation}
where the first inequality follows from~\eqref{eq:av-selected-bias-increase}
and the second from~\eqref{eq:av-path-residual}.

Fix an iteration \(t\) with \(T-t>2N^2\log N\).
By~\eqref{eq:av-bias-monotonicity} and~\eqref{eq:av-selected-bias-increase},
we have \(\sum_s\bigl(V^{\pi_T}(s)-V^{\pi_t}(s)\bigr)>0\),
since the update at \(t\) has positive residual.
Choose \((s,a)\) with \(a=\pi_t(s)\) maximizing
\(V^{\pi_T}(s)-Q^{\pi_T}(s,a)\).
For each \(s'\in\mathcal S\), let \(s'_{(0)}=s',\ldots,s'_{(\ell(s'))}\)
be the path under \(\pi_t\) from \(s'\) to its first state on a cycle.
By the same calculation as for~\eqref{eq:av-path-residual}, we obtain
\begin{equation}
\sum_{s'\in\mathcal S}\bigl(V^{\pi_T}(s')-V^{\pi_t}(s')\bigr)
=\sum_{s'\in\mathcal S}\sum_{\tau=0}^{\ell(s')-1}
\bigl(V^{\pi_T}(s'_{(\tau)})-Q^{\pi_T}(s'_{(\tau)},\pi_t(s'_{(\tau)}))\bigr)
\le N^2\bigl(V^{\pi_T}(s)-Q^{\pi_T}(s,a)\bigr),
\label{eq:av-bias-witness}
\end{equation}
where the inequality follows from the choice of \((s,a)\) and \(\ell(s')<N\).
For any policy \(\pi_{t'}\) in this sequence with \(\pi_{t'}(s)=a\), we have
\(V^{\pi_{t'}}(s)=r(s,a)-J^{\pi_T}(s)+V^{\pi_{t'}}(f(s,a))\le Q^{\pi_T}(s,a)\)
by~\eqref{eq:av-bias} and~\eqref{eq:av-bias-monotonicity}.
Therefore, by~\eqref{eq:av-bias-witness}, we have
\[
\sum_{s'\in\mathcal S}\bigl(V^{\pi_T}(s')-V^{\pi_{t'}}(s')\bigr)
\ge V^{\pi_T}(s)-Q^{\pi_T}(s,a)
\ge \frac1{N^2}\sum_{s'\in\mathcal S}\bigl(V^{\pi_T}(s')-V^{\pi_t}(s')\bigr).
\]
For every \(t'\) with \(t+2N^2\log N<t'\le T\),
\eqref{eq:av-transient-contraction} makes the left-hand side strictly smaller
than the right-hand side, so \(\pi_{t'}(s)\ne a\).
Each application of this argument excludes a state--action pair after
\(O(N^2\log(2N))\) updates. Repeating the argument from the resulting policy
excludes a different pair because previously excluded pairs are no longer used.
There are only \(M\) state--action pairs, so we obtain
\begin{equation}
\#\{\text{consecutive bias updates with unchanged cycles}\}
=O(N^2M\log(2N)).
\label{eq:av-bias-update-count}
\end{equation}

\vspace{5pt}\noindent\textbf{Number of new cycles.}
We count cycle creations separately in each strongly connected component
of the transition graph.
Every cycle creation occurs in a bias update, so no gain update is available
immediately beforehand. By~\eqref{eq:av-single-switch}, we have
\(J^{\pi_t}(f(s,a))\le J^{\pi_t}(s)\) for every state--action pair.
Applying this inequality along paths in both directions shows that
\(J^{\pi_t}\) is constant within each component whenever no gain update is available.
Within a fixed component, let \(\mu_n\) denote the mean of the \(n\)-th cycle
created by the algorithm.
We consider the update from \(\pi_t\) to \(\pi_{t+1}\)
that creates the \(n\)-th cycle \(C\) in this component.
We denote the common gain
under \(\pi_t\) in the component containing \(C\) by \(J^\dagger\).
Thus we have \(J^{\pi_t}(f(s,a))=J^{\pi_t}(s)=J^\dagger\)
for every edge \((s,a)\) within this component.
We write \(C^\star\) for a cycle of maximum mean \(\mu^\star\)
in the component containing \(C\).
Since this is a bias update, we have
\(Q^{\pi_t}(s_t,a_t)-V^{\pi_t}(s_t)>0\).
Summing the second identity in~\eqref{eq:av-bias} around \(C\) and \(C^\star\),
we obtain
\begin{align}
\mu_n-J^\dagger
&=\frac1{|C|}\sum_{(s,a)\in C}
\bigl(Q^{\pi_t}(s,a)-V^{\pi_t}(s)\bigr)
=\frac{Q^{\pi_t}(s_t,a_t)-V^{\pi_t}(s_t)}{|C|}\notag\\
&\ge\frac1{|C|\,|C^\star|}\sum_{(s,a)\in C^\star}
\bigl(Q^{\pi_t}(s,a)-V^{\pi_t}(s)\bigr)
=\frac{\mu^\star-J^\dagger}{|C|}\notag\\
&\ge\frac{\mu^\star-J^\dagger}{N}>0,
\label{eq:av-cycle-gain-bound}
\end{align}
where the second equality uses \(Q^{\pi_t}(s,a)-V^{\pi_t}(s)=0\)
for every \((s,a)\in C\setminus\{(s_t,a_t)\}\), since \(a=\pi_t(s)\).
The first inequality follows from~\eqref{eq:av-single-switch}
because \(J^{\pi_t}(f(s,a))=J^{\pi_t}(s)\)
for every \((s,a)\in C^\star\).
The second uses \(|C|\le N\), and the strict inequality follows from
\(\mu^\star\ge\mu_n>J^\dagger\).
Consequently, we have
\(\mu^\star-\mu_n\le(1-1/N)(\mu^\star-J^\dagger)\).
Every state of \(C\) has gain \(\mu_n\) under \(\pi_{t+1}\).
If the \((n+1)\)-th cycle is created in this component,
the common gain in the component immediately before that update
is at least \(\mu_n\) by~\eqref{eq:av-gain-monotonicity}.
Applying~\eqref{eq:av-cycle-gain-bound} to that update, we obtain
\begin{equation}
\mu_{n+1}>\mu_n,\qquad
\mu^\star-\mu_{n+1}\le\left(1-\frac1N\right)(\mu^\star-\mu_n).
\label{eq:av-cyclic-contraction}
\end{equation}
By~\eqref{eq:av-cyclic-contraction}, no further cycle can be created
in this component when \(\mu_n=\mu^\star\).
For \(\mu_n<\mu^\star\), we show that some edge of \(C\) belongs to
no cycle of mean greater than
\(\mu^\star-(\mu^\star-\mu_n)/N\).
Suppose instead that every edge \(e\in C\) belongs to a cycle
\(C_e\) in this component with mean greater than this bound.
For every cycle in this component, the sum of \(r(s,a)-\mu^\star\)
over its edges is nonpositive.
By the choice of \(C_e\) and \(|C_e|\le N\), we have
\begin{equation}
\sum_{e\in C}\sum_{(s,a)\in C_e}\bigl(r(s,a)-\mu^\star\bigr)
>-\frac{\mu^\star-\mu_n}{N}\sum_{e\in C}|C_e|
\ge-|C|(\mu^\star-\mu_n)
=\sum_{(s,a)\in C}\bigl(r(s,a)-\mu^\star\bigr).
\label{eq:av-cycle-sum-comparison}
\end{equation}
By~\eqref{eq:av-cycle-sum-comparison}, we have
\(\sum_{e\in C}\sum_{(s,a)\in C_e\setminus\{e\}}
\bigl(r(s,a)-\mu^\star\bigr)>0\).
Counting repeated edges separately, the remaining edges have equal numbers
of incoming and outgoing edges at every state.
They can therefore be partitioned into directed cycles.
At least one of these cycles has mean greater than \(\mu^\star\),
contradicting the definition of \(\mu^\star\).
Meanwhile, if the \((n+\lceil N\log(2N)\rceil)\)-th cycle is created
in this component, by~\eqref{eq:av-cyclic-contraction}, we have
\[
\mu^\star-\mu_{n+\lceil N\log(2N)\rceil}
\le\left(1-\frac1N\right)^{\lceil N\log(2N)\rceil}(\mu^\star-\mu_n)
\le\frac{\mu^\star-\mu_n}{2N}
<\frac{\mu^\star-\mu_n}{N},
\]
where the second inequality uses \(1-1/N\le e^{-1/N}\).
By~\eqref{eq:av-cyclic-contraction}, every cycle created in this component
with index at least \(n+\lceil N\log(2N)\rceil\) omits the identified edge.
Repeating this argument every \(\lceil N\log(2N)\rceil\) further cycle creations
within each component excludes a different edge each time, with at most
\(M\) exclusions in total.
Counting also the first creation and fewer than \(\lceil N\log(2N)\rceil\)
remaining creations in each of the at most \(N\) components, we obtain
\begin{equation}
\#\{\text{cycle creations}\}
\le M\lceil N\log(2N)\rceil
+N\bigl(1+\lceil N\log(2N)\rceil\bigr)
=O(NM\log(2N)).
\label{eq:av-cycle-creation-count}
\end{equation}

\vspace{5pt}\noindent\textbf{Completing the proof.}
Since at most \(N\) cycles are present initially and at least one remains, we have
\[
\begin{aligned}
\#\{\text{cycle destructions}\}
&=\#\{\text{cycle creations}\}
+\#\{\text{initial cycles}\}-\#\{\text{current cycles}\}\\
&\le\#\{\text{cycle creations}\}+N-1.
\end{aligned}
\]
By~\eqref{eq:av-cycle-creation-count}, we therefore have
\(O(NM\log(2N))\) updates that create or destroy a cycle.
Gain updates cannot resume after a bias update until a cycle is created,
so there are also \(O(NM\log(2N))\) sequences of consecutive bias updates
with unchanged cycles.
Combining~\eqref{eq:av-gain-update-count},
\eqref{eq:av-bias-update-count}, and~\eqref{eq:av-cycle-creation-count},
we conclude that the algorithm terminates within
\(O(N^3M^2\log^2(2N))\) iterations.
By Lemma~\ref{lem:av-monotonicity}, the final policy is optimal.
Policy evaluation and maximization in~\eqref{eq:av-single-switch} require
\(O(M)\) arithmetic operations per update, with all intermediate numbers
having encoding lengths polynomial in the input size.
Thus the implementation is strongly polynomial.
\end{proof}

\subsection{Strong polynomiality of policy iteration with small integer rewards}
\label{sec:av-bounded-rewards}

The instances in Theorem~\ref{thm:average} use \(O(N)\) bits per reward.
We next give an iteration bound for policy iteration with nonnegative
integer rewards in terms of their bit length.

\begin{theorem}[Bounded integer rewards]
\label{thm:av-bounded-rewards}
On any deterministic average-reward MDP with \(N\) states
and nonnegative integer rewards smaller than \(2^b\),
Howard's policy iteration with update rule~\eqref{eq:av-update} terminates with an optimal policy
within \(O(N^3 2^b)\) iterations from every initial policy.
\end{theorem}

For Howard's policy iteration with update rule~\eqref{eq:av-update}, we give a direct, self-contained
proof of the \(O(N^3 2^b)\) iteration bound stated by \citet{AsadiEtAl2025}.
\citet{MukherjeeKalyanakrishnan2025} give an \(O(N^5 4^b)\) bound
for a different policy iteration scheme that performs bias updates only
when no gain update is available at any state.
For nonnegative integer rewards of at most logarithmic bit length
(\(b=O(\log N)\)), Howard's policy iteration is strongly polynomial
on deterministic average-reward MDPs.
For deterministic discounted MDPs, Theorem~\ref{thm:constant-rewards}
gives a stretched-exponential lower bound even when each reward
has constant bit length.
To understand this difference, consider cycle creation in the proof below.
Every \(N\) consecutive updates before termination create a new cycle,
and each update creating a cycle increases the sum of state gains
by at least \(1/N\). This sum is nondecreasing and smaller than \(N2^b\).
In the discounted construction of Section~\ref{sec:bounded-rewards},
cycles with the same mean reward have strictly increasing discounted
values along the counting sequence. The positive increments depend
on reward timing and tend to zero as \(\gamma\to1\) with the rewards fixed,
so logarithmic reward bit length does not provide a comparable minimum increase.

\begin{proof}[Proof of Theorem~\ref{thm:av-bounded-rewards}]
We write \((x_1,x_2)\ge_{\mathrm{lex}}(y_1,y_2)\) if \(x_1>y_1\),
or if \(x_1=y_1\) and \(x_2\ge y_2\).
For \(P_t(s)=(J^{\pi_t}(s),V^{\pi_t}(s))\),
Lemma~\ref{lem:av-monotonicity} gives
\(P_{t+1}(s)\ge_{\mathrm{lex}}P_t(s)\) for every \(s\in\mathcal S\),
with strict inequality if \(\pi_{t+1}(s)\ne\pi_t(s)\).

\vspace{5pt}\noindent\textbf{Updates without a new cycle.}
Combining~\eqref{eq:av-bias} and~\eqref{eq:av-update}, we have
\[
\pi_{t+1}(s)\in\operatorname*{arg\,max}_{a\in\mathcal A_s}
\bigl(J^{\pi_t}(f(s,a)),\
r(s,a)-J^{\pi_t}(f(s,a))+V^{\pi_t}(f(s,a))\bigr).
\]
For the selected action \(a=\pi_{t+1}(s)\),
by~\eqref{eq:av-gain-monotonicity}, we have
\(J^{\pi_{t+1}}(s)=J^{\pi_{t+1}}(f(s,a))\ge J^{\pi_t}(f(s,a))\).
If equality holds, \eqref{eq:av-bias-monotonicity} gives
\(V^{\pi_{t+1}}(f(s,a))\ge V^{\pi_t}(f(s,a))\),
so we obtain
\(V^{\pi_{t+1}}(s)\ge
r(s,a)-J^{\pi_t}(f(s,a))+V^{\pi_t}(f(s,a))\)
by the first identity in~\eqref{eq:av-bias}.
Since \(\pi_{t+1}(s)\) attains the maximum above, we conclude that
\begin{equation}
P_{t+1}(s)\ge_{\mathrm{lex}}
\bigl(J^{\pi_t}(f(s,a)),\
r(s,a)-J^{\pi_t}(f(s,a))+V^{\pi_t}(f(s,a))\bigr),
\label{eq:av-bounded-propagation}
\end{equation}
for every \(s\in\mathcal S\) and \(a\in\mathcal A_s\).

Suppose \(N\) consecutive updates create no new cycle,
and reindex their policies as \(\pi_0,\ldots,\pi_N\).
For \(0\le t<N\), every cycle contained in \(\pi_{t+1}\)
is also contained in \(\pi_t\).
Thus \(\pi_0,\ldots,\pi_N\) use the same actions on every cycle \(C\)
contained in \(\pi_N\), and their gains on \(C\) equal its mean reward.
Their biases satisfy the same recurrence in the first identity
of~\eqref{eq:av-bias} and are all zero at \(s_C\), giving
\(P_t(s)=P_N(s)\) for every \(s\in C\) and \(0\le t\le N\).
For \(0\le t<N\), suppose \(P_t(s)=P_N(s)\) for every state
that reaches a cycle under \(\pi_N\) within \(t\) transitions.
If \(s\) reaches a cycle under \(\pi_N\) within \(t+1\) transitions,
its successor \(f(s,\pi_N(s))\) reaches a cycle within \(t\) transitions,
so the induction hypothesis gives
\(P_t(f(s,\pi_N(s)))=P_N(f(s,\pi_N(s)))\).
By~\eqref{eq:av-bias} and~\eqref{eq:av-bounded-propagation}
with \(a=\pi_N(s)\), we have
\(P_{t+1}(s)\ge_{\mathrm{lex}}P_N(s)\), while
Lemma~\ref{lem:av-monotonicity} gives \(P_{t+1}(s)\le_{\mathrm{lex}}P_N(s)\).
Hence \(P_{t+1}(s)=P_N(s)\), completing the induction.
Every state reaches a cycle under \(\pi_N\) within \(N-1\) transitions,
so \(P_{N-1}(s)=P_N(s)\) for every \(s\in\mathcal S\),
contradicting strict improvement at a state with
\(\pi_N(s)\ne\pi_{N-1}(s)\).
Thus every \(N\) consecutive updates create a new cycle.

\vspace{5pt}\noindent\textbf{Cycle creations and termination.}
Consider a cycle \(C\) created in an update from \(\pi_t\) to \(\pi_{t+1}\).
Applying~\eqref{eq:av-edge-gain} around \(C\), we obtain
\(J^{\pi_t}(s)=J^{\pi_t}(s_C)\) for every \(s\in C\).
Hence~\eqref{eq:av-update} gives
\(Q^{\pi_t}(s,\pi_{t+1}(s))-V^{\pi_t}(s)\ge0\) for \(s\in C\),
with strict inequality if \(\pi_{t+1}(s)\ne\pi_t(s)\).
At least one action on \(C\) changes because \(C\) is new.
Summing the second identity in~\eqref{eq:av-bias} around \(C\), we obtain
\begin{equation}
\sum_{s\in C}\bigl(J^{\pi_{t+1}}(s)-J^{\pi_t}(s)\bigr)
=\sum_{s\in C}\bigl(Q^{\pi_t}(s,\pi_{t+1}(s))-V^{\pi_t}(s)\bigr)
=\sum_{s\in C}r(s,\pi_{t+1}(s))-|C|J^{\pi_t}(s_C)
\ge\frac1N.
\label{eq:av-bounded-cycle-increase}
\end{equation}
The final inequality holds because the residual sum is positive,
the reward sum is an integer, and \(J^{\pi_t}(s_C)\) is a cycle mean
with denominator at most \(N\).
By~\eqref{eq:av-gain-monotonicity} and~\eqref{eq:av-bounded-cycle-increase},
we have \(\sum_{s\in\mathcal S}
\bigl(J^{\pi_{t+1}}(s)-J^{\pi_t}(s)\bigr)\ge1/N\)
whenever the update creates a cycle.
Among the first \(t\) updates, at least \(\lfloor t/N\rfloor\)
create a cycle, so we obtain
\[
\frac{\lfloor t/N\rfloor}{N}
\le\sum_{s\in\mathcal S}\bigl(J^{\pi_t}(s)-J^{\pi_0}(s)\bigr)
<N2^b.
\]
The last inequality uses \(0\le J^{\pi_0}(s)\le J^{\pi_t}(s)<2^b\)
for every \(s\in\mathcal S\).
Combining this bound with Lemma~\ref{lem:av-monotonicity}, we conclude
that policy iteration finds an optimal policy within \(O(N^3 2^b)\) iterations.
With \(s_C\) chosen as the first state of \(C\) in a fixed ordering,
policy evaluation and maximization in~\eqref{eq:av-update} require
\(O(M)\) arithmetic operations per iteration on
\(O(b+\log(2N))\)-bit numbers.
The total arithmetic cost is therefore \(O(MN^3 2^b)\), so the
implementation is strongly polynomial when \(b=O(\log N)\).
\end{proof}

\section{Conclusion}
\label{sec:conclusion}

This paper studies Howard's policy iteration and Dantzig's rule for
deterministic discounted and average-reward MDPs.
We show that policy iteration can require exponentially many iterations
in the number of states on instances of polynomial encoding length,
so it is not a polynomial-time algorithm in either setting.
For average-reward MDPs, policy iteration is strongly polynomial when
nonnegative integer rewards have at most logarithmic bit length.
For discounted MDPs, it can require superpolynomially many iterations
even when each reward has constant bit length.
Together with the strongly polynomial bounds for Dantzig's rule,
our lower bounds establish an exponential separation between the two algorithms.
This separation reveals a price of algorithmic anarchy: decentralized
and simultaneous selfish improvements based on the same policy values
can be exponentially slower than Dantzig's coordinated single-state updates
with policy re-evaluation after each switch.

%% file: appendix.tex
\renewcommand*{\theHsection}{appendix.\Alph{section}}
\section{\texorpdfstring{Proof of Lemma~\ref{prop:counter}}{Proof of Lemma \getrefnumber{prop:counter}}}
\label{app:discounted-counting}

\newcounter{updatestep}[subsection]
\renewcommand{\theupdatestep}{\arabic{subsection}.\arabic{updatestep}}

Suppose \(\pi_t=\pi[\mathbf b]\), where \(\mathbf b\) encodes
an integer in \(0,\ldots,2^d-3\).
Let
\begin{equation}
\kappa=\min\{i:[\mathbf b]_i=0\}.
\label{eq:first-zero-bit}
\end{equation}
We show that the next five iterations produce \(\pi_{t+5}=\pi[\mathbf b^+]\).
In Appendices~\ref{sec:update-one}--\ref{sec:update-five}, we determine
the actions selected at each iteration by comparing action values
as \(\gamma\to1^-\).
Appendix~\ref{sec:complete-counter-proof} verifies these comparisons
for the discount in~\eqref{proof:32}.
Table~\ref{tab:auxiliary-updates} records the choices at all states with
two actions throughout the increment.

We divide the unshifted rewards by \(4\gamma F_0(\gamma)\), which is positive
for \(\gamma<1\) sufficiently close to one.
For a path starting with action \(a\) at \(s\) and then following \(\pi\),
suppose its first visit to \(s'\) at a positive time
occurs after \(m\) transitions.
When \(s'=s\), this is the first return.
Let \(r_0,\ldots,r_{m-1}\) be the unshifted rewards before this visit and define
\begin{equation}
\widetilde R^\pi((s,a)\leadsto s')
:=\frac{\sum_{\tau=0}^{m-1}\gamma^\tau r_\tau}{4\gamma F_0(\gamma)},
\qquad
\widetilde R^\pi(s\leadsto s')
=\widetilde R^\pi((s,\pi(s))\leadsto s').
\label{eq:path-reward}
\end{equation}
\noeqref{eq:path-reward}
For \(0\le i<d\), the selected \(6L\)-transition paths under \(\pi[\mathbf b]\)
have scaled rewards given by~\eqref{eq:rewards-s1} and~\eqref{proof:31}:
\begin{equation}
\widetilde R^{\pi[\mathbf b]}(s_i\leadsto s_{i+1})
=\widetilde R^{\pi[\mathbf b]}(\bar s_i\leadsto\bar s_{i+1})
=[\mathbf b]_i\gamma^{4L}\mathsf R_i(\gamma),\qquad
\mathsf R_i(\gamma):=(24d+4)\mathsf g_i(\gamma)(1-\gamma^L).
\label{eq:scaled-rewards}
\end{equation}
Here \(\mathsf R_i(\gamma)\) is the scaled reward along the
\(2L\)-transition \(a_0\) path from \(\bar s_{i,4}\) to \(\bar s_{i+1}\).
For any policy \(\pi\), define the scaled values
\begin{equation}
\widetilde V^\pi(s):=\frac{V^\pi(s)}{4\gamma F_0(\gamma)},\qquad
\widetilde Q^\pi(s,a):=\frac{Q^\pi(s,a)}{4\gamma F_0(\gamma)},\qquad
\widetilde Q^\pi(s,a)=\widetilde R^\pi((s,a)\leadsto s')+\gamma^m\widetilde V^\pi(s'),
\label{eq:path-decomposition}
\end{equation}
where the decomposition in~\eqref{eq:path-decomposition}
follows by applying~\eqref{eq:policy-value}
along the path in~\eqref{eq:path-reward}.
Using~\eqref{eq:scaled-rewards} and~\eqref{eq:path-decomposition}, we obtain
\begin{equation}
\widetilde V^{\pi[\mathbf b]}(s_i)=\gamma^{6L}\widetilde V^{\pi[\mathbf b]}(s_{i+1})+[\mathbf b]_i\gamma^{4L}\mathsf R_i(\gamma),
\quad \forall\,0\le i<d,\qquad
\widetilde V^{\pi[\mathbf b]}(s_d)=\gamma^L\widetilde V^{\pi[\mathbf b]}(s_0).
\label{proof:6}
\end{equation}
We also obtain~\eqref{proof:6} with each \(s_i\) replaced by
\(\bar s_i\), since the corresponding selected paths have the same
lengths and reward sequences (see Section~\ref{sec:counter-interpretation}).
We obtain the common limit at \(s_0\) and \(\bar s_0\) by dividing
\eqref{eq:encoded-cycle-value} by \(4\gamma F_0(\gamma)\) and taking
\(\gamma\to1^-\) as in the derivation of~\eqref{eq:increment-limit}.
Since \(\gamma^L\to1\) and \(\mathsf R_i(\gamma)\to0\),
\eqref{proof:6} gives this limit first at \(s_d,\bar s_d\) and then
at \(s_i,\bar s_i\) for \(i=d-1,\ldots,1\). Hence, we have
\begin{equation}
\lim_{\gamma\to1^{-}}\widetilde V^{\pi[\mathbf b]}(s_i)
=4\sum_{j=0}^{d-1}\mathsf g_j[\mathbf b]_j
=\lim_{\gamma\to1^{-}}\widetilde V^{\pi[\mathbf b]}(\bar s_i),
\qquad \forall\,0\le i\le d.
\label{eq:encoding-value-limit}
\end{equation}
We express the policy iteration rule at states with two actions in terms of
\begin{equation}
\Delta^\pi(s):=\widetilde Q^\pi(s,a_1)-\widetilde Q^\pi(s,a_0).
\label{eq:action-difference}
\end{equation}
By~\eqref{eq:policy-update}, the next policy selects \(a_1\) if
\(\Delta^\pi(s)>0\), selects \(a_0\) if \(\Delta^\pi(s)<0\),
and keeps \(\pi(s)\) otherwise.

\subsection{Iteration 1: identify the zero positions of \texorpdfstring{\(\mathbf b^+\)}{b+}}
\label{sec:update-one}
Before this iteration, \(\pi_t=\pi[\mathbf b]\).
We show that this iteration sets all bits at \(s_i\) to one, leaves
the bits at \(\bar s_i\) unchanged, and selects paths through
\(\mathcal S_2\) at the zero positions of \(\mathbf b^+\).

\vspace{5pt}\noindent\refstepcounter{updatestep}%
\textbf{\boldmath Step~\theupdatestep. Determine the actions at \(s_i,s_d\) and \(\{s_{i,j}\}_{j=1}^{3}\).}\label{step:set-first-bits}
For \(0\le i<d\), \eqref{proof:32} and~\eqref{eq:encoding-policy} give
\[
\pi_t(s_{i,j})=\pi_0(s_{i,j})=a_1,\quad \forall\,1\le j\le3,\qquad
\pi_t(s_{i,j})=\pi_0(s_{i,j})=a_0,\quad \forall\,j\in\{4,5\}.
\]
Both actions at \(s_i\), followed by \(\pi_t\), reach
\(s_{i+1}\) after \(6L\) transitions.
At \(s_d\), the two \(L\)-transition paths end at \(s_0\) and
\(\bar s_0\), whose values are equal by~\eqref{eq:encoded-cycle-value}.
By~\eqref{eq:reward-sd-a0}, \eqref{proof:31},
and~\eqref{eq:action-difference}, we obtain
\begin{equation*}
\Delta^{\pi_t}(s_i)=\bigl[\gamma^{4L}\mathsf R_i(\gamma)
+\gamma^{6L}\widetilde V^{\pi_t}(s_{i+1})\bigr]
-\gamma^{6L}\widetilde V^{\pi_t}(s_{i+1})
=\gamma^{4L}\mathsf R_i(\gamma)>0,\qquad
\Delta^{\pi_t}(s_d)=-4\mathsf f_{d-1}(\gamma)<0.
\end{equation*}
At \(s_{i,j}\), both actions followed by \(\pi_t\) reach
\(s_{i+1}\) in \((6-j)L\) transitions,
so the same comparison gives
\(\Delta^{\pi_t}(s_{i,j})=\gamma^{(4-j)L}\mathsf R_i(\gamma)>0\)
for \(1\le j\le3\).
We conclude that \(s_i\) and \(\{s_{i,j}\}_{j=1}^{3}\) select \(a_1\)
for \(0\le i<d\), while \(s_d\) switches to \(a_0\).

\vspace{5pt}\noindent\refstepcounter{updatestep}%
\textbf{\boldmath Step~\theupdatestep. Determine the actions at \(\bar s_i,\{\bar s_{i,j}\}_{j=1}^{4}\) and in \(\mathcal S_3\).}\label{step:retain-old-bits}
For \(s\in\{\bar s_{i,4}:0\le i<d\}\cup\mathcal S_3\),
\eqref{proof:32} and~\eqref{eq:encoding-policy} give \(\pi_t(s)=a_0\).
We compare the two action values after subtracting
\(\widetilde V^{\pi_t}(\bar s_0)\).
\begin{itemize}[leftmargin=*,topsep=3pt,itemsep=2pt,parsep=0pt]
\item \emph{Initial action \(a_0\).}
By~\eqref{eq:reward-bar-s14-a0}, we have
\(\widetilde Q^{\pi_t}(\bar s_{i,4},a_0)=\mathsf R_i(\gamma)
+\gamma^{2L}\widetilde V^{\pi_t}(\bar s_{i+1})\).
By~\eqref{eq:encoding-value-limit} and \(\mathsf R_i(\gamma)\to0\), we have
\(\widetilde Q^{\pi_t}(\bar s_{i,4},a_0)-\widetilde V^{\pi_t}(\bar s_0)\to0\).
By Table~\ref{tab:paths-s1}, the selected path from
\(s_{i,j}^{\mathrm r}\in\mathcal S_3\) reaches \(\bar s_0\)
in \((2i+2)L\) transitions through either \(s_{-1}^{\mathrm c}\) or
\(s_{-1}^{\mathrm z}\).
In the first case, the path has zero rewards until \(s_{-1}^{\mathrm c}\),
whose path to \(\bar s_0\) contributes \(4\mathsf f_{d-1}\) in the limit
by~\eqref{eq:rewards-s2}.
In the second case, each visited \(s_{h,0}^{\mathrm r}\) contributes
\(\mathsf f_h\) in the limit by~\eqref{eq:reward-sr0-a0}, giving a nonnegative
sum bounded by \(\sum_{h=0}^i\mathsf f_h<2\mathsf f_{d-1}\)
by~\eqref{proof:1}.
Combining these two cases with the preceding comparison at \(\bar s_{i,4}\),
we obtain from~\eqref{eq:path-decomposition} and~\eqref{eq:encoding-value-limit}
\begin{equation}
0\le\lim_{\gamma\to1^{-}}
\bigl(\widetilde Q^{\pi_t}(s,a_0)-\widetilde V^{\pi_t}(\bar s_0)\bigr)
\le4\mathsf f_{d-1},\qquad
\forall\,s\in\{\bar s_{i,4}:0\le i<d\}\cup\mathcal S_3.
\label{eq:initial-a0-value-bound}
\end{equation}
\item \emph{Initial action \(a_1\).} By Table~\ref{tab:paths-s1},
the path from \(\bar s_{i,4}\) ends at \(s_i\).
For the states in \(\mathcal S_3\), the path from \(s_{i,0}^{\mathrm r}\)
ends at \(s_{i+1}\), while the paths from \(s_{i,1}^{\mathrm r}\)
and \(s_{i,2}^{\mathrm r}\) both end at \(s_{i,1}\).
By~\eqref{eq:encoding-value-limit}, the values at \(s_i,s_{i+1}\)
have the same limit as \(\widetilde V^{\pi_t}(\bar s_0)\).
The identity
\(\widetilde V^{\pi_t}(s_{i,1})=\gamma^{3L}\mathsf R_i(\gamma)
+\gamma^{5L}\widetilde V^{\pi_t}(s_{i+1})\)
and \(\mathsf R_i(\gamma)\to0\) give the same limit at \(s_{i,1}\).
Using the rewards in~\eqref{eq:reward-bar-s14-a1} and
\mbox{\eqref{eq:reward-sr0-a1}--\eqref{eq:reward-sr12-a1}},
we obtain for \(0\le i<d\)
\begin{equation}
\lim_{\gamma\to1^{-}}
\bigl(\widetilde Q^{\pi_t}(s,a_1)-\widetilde V^{\pi_t}(\bar s_0)\bigr)
=\begin{cases}
-8\mathsf f_{d-1},&s=\bar s_{i,4},\\
-8\mathsf f_{d-1}-\mathsf f_i,&s=s_{i,j}^{\mathrm r},\quad j\in\{0,1,2\},
\end{cases}
\le-8\mathsf f_{d-1}.
\label{eq:initial-a1-value-bound}
\end{equation}
\end{itemize}
Using~\eqref{eq:action-difference}, \eqref{eq:initial-a0-value-bound},
and~\eqref{eq:initial-a1-value-bound}, we obtain
\[
\lim_{\gamma\to1^{-}}\Delta^{\pi_t}(s)
\le-8\mathsf f_{d-1}<0,\qquad
\forall\,s\in\{\bar s_{i,4}:0\le i<d\}\cup\mathcal S_3.
\label{eq:reference-retention}
\]
We conclude that \(\pi_{t+1}(s)=a_0\) for
\(s\in\{\bar s_{i,4}:0\le i<d\}\cup\mathcal S_3\).
We next show that each \(\bar s_i\) keeps \(a_{[\mathbf b]_i}\)
and determine the actions selected at
\(\{\bar s_{i,j}\}_{j=1}^{3}\).
By~\eqref{eq:encoding-policy}, \(\pi_t\) chooses
\(a_{[\mathbf b]_i}\) at \(\bar s_i\) and \(\{\bar s_{i,j}\}_{j=1}^{3}\).
\begin{itemize}[leftmargin=*,topsep=3pt,itemsep=2pt,parsep=0pt,beginpenalty=10000]
\item If \([\mathbf b]_i=0\), action \(a_0\) at \(\bar s_i\)
reaches \(\bar s_{i+1}\) in \(6L\) transitions. Action \(a_1\)
reaches \(\bar s_{i,1}\) in \(L\) transitions, after which its
current action \(a_0\) reaches \(\bar s_{i+1}\) in \(5L\) transitions.
Both paths have zero rewards by~\eqref{eq:rewards-s1}, and thus
\[
\widetilde Q^{\pi_t}(\bar s_i,a_1)=\gamma^L\widetilde V^{\pi_t}(\bar s_{i,1})
=\gamma^{6L}\widetilde V^{\pi_t}(\bar s_{i+1})
=\widetilde Q^{\pi_t}(\bar s_i,a_0),\qquad \Delta^{\pi_t}(\bar s_i)=0.
\]
At \(\bar s_{i,3}\), both actions followed by \(\pi_t\) reach
\(\bar s_{i+1}\) in \(3L\) transitions.
The \(a_0\) path has zero rewards, while the \(a_1\) path passes through
\(\bar s_{i,4}\).
For \(j\in\{1,2\}\), both actions at \(\bar s_{i,j}\) followed by \(\pi_t\)
reach \(\bar s_{i+1}\) in \((6-j)L\) transitions with zero rewards.
Using~\eqref{eq:reward-bar-s14-a0} and~\eqref{eq:scaled-rewards}, we obtain
\[
\Delta^{\pi_t}(\bar s_{i,3})=\gamma^L\mathsf R_i(\gamma)>0,\qquad
\Delta^{\pi_t}(\bar s_{i,j})=0,\quad \forall\,j\in\{1,2\}.
\]
We therefore find that \(\bar s_{i,3}\) switches to \(a_1\), while
\(\bar s_i,\bar s_{i,1},\bar s_{i,2}\) keep \(a_0\).
\item If \([\mathbf b]_i=1\), the \(a_1\) path from \(\bar s_i\) passes through
\(\bar s_{i,1},\ldots,\bar s_{i,4}\) and reaches \(\bar s_{i+1}\)
in \(6L\) transitions. The \(a_0\) path has the same length and endpoint
but zero rewards. By~\eqref{proof:31} and~\eqref{eq:scaled-rewards},
we obtain the action differences at
\(\bar s_i\) and \(\{\bar s_{i,j}\}_{j=1}^{3}\):
\[
\Delta^{\pi_t}(\bar s_i)=\gamma^{4L}\mathsf R_i(\gamma)>0,\qquad
\Delta^{\pi_t}(\bar s_{i,j})=\gamma^{(4-j)L}\mathsf R_i(\gamma)>0,
\quad \forall\,1\le j\le3.
\]
We conclude that \(\bar s_i\) and \(\{\bar s_{i,j}\}_{j=1}^{3}\)
continue to choose \(a_1\).
\end{itemize}
If \([\mathbf b]_i=0\), switching the action at
\(\bar s_{i,3}\) to \(a_1\) leaves the selected \(a_0\) path from
\(\bar s_i\) to \(\bar s_{i+1}\) unchanged, because this path bypasses
\(\bar s_{i,3}\). Since the selected paths for \([\mathbf b]_i=1\)
also remain unchanged, we conclude that \(\pi_t\) and \(\pi_{t+1}\)
generate the same trajectory and rewards from \(\bar s_0\).

\vspace{5pt}\noindent\refstepcounter{updatestep}%
\textbf{\boldmath Step~\theupdatestep. Determine the actions at \(s_{i,4}\) and \(s_{i,5}\).}\label{step:select-reset-paths}
By Table~\ref{tab:paths-s1} and Figure~\ref{fig:bit-paths}, the actions
at \(s_{i,4}\) and \(s_{i,5}\) determine which of
\(s_{i+1},s_i^{\mathrm z},s_i^{\mathrm c}\) is reached from \(s_{i,4}\)
after \(2L\) transitions.
By~\eqref{proof:32} and~\eqref{eq:encoding-policy}, we have
\(\pi_t(s_j^{\mathrm c})=\pi_t(s_j^{\mathrm z})=a_{[\mathbf b]_j}\)
for \(0\le j<d\) and \(\pi_t(s)=a_0\) for \(s\in\mathcal S_3\).
Under these choices, the paths from \(s_i^{\mathrm z}\) and
\(s_i^{\mathrm c}\) reach \(\bar s_0\) in \((2i+3)L\) transitions.

\vspace{5pt}\noindent\emph{Determine \(\pi_{t+1}(s_{i,5})\).}
By Table~\ref{tab:paths-s1}, actions \(a_0\) and \(a_1\) at \(s_{i,5}\)
reach \(s_{i+1}\) and \(s_i^{\mathrm z}\), respectively, in \(L\) transitions
under any policy.
By~\eqref{eq:rewards-s2} and~\eqref{eq:reward-sr0-a0}, we have
\(\widetilde R^{\pi_t}(s_j^{\mathrm z}\leadsto s_{j-1}^{\mathrm z})
=\gamma^L\mathsf f_j(\gamma)(1-[\mathbf b]_j)\) for \(0\le j\le i\).
The final path from \(s_{-1}^{\mathrm z}\) to \(\bar s_0\) has zero rewards.
Using~\eqref{eq:reward-s15-a0} and~\eqref{eq:reward-s15-a1}, we obtain
\begin{subequations}\label{eq:si5-action-difference}
\begin{align}
\lim_{\gamma\to1^-}\Delta^{\pi_t}(s_{i,5})
&=\lim_{\gamma\to1^-}\Bigl[
\gamma^L\bigl(\widetilde V^{\pi_t}(s_i^{\mathrm z})-\widetilde V^{\pi_t}(s_{i+1})\bigr)
-\mathsf f_i(\gamma)-\tfrac14
\Bigr]\label{eq:si5-action-comparison}\\
&=\lim_{\gamma\to1^-}\Bigl[
\gamma^L\bigl(
\widetilde R^{\pi_t}(s_i^{\mathrm z}\leadsto\bar s_0)
+\gamma^{(2i+3)L}\widetilde V^{\pi_t}(\bar s_0)
-\widetilde V^{\pi_t}(s_{i+1})
\bigr)-\mathsf f_i(\gamma)-\tfrac14
\Bigr]\label{eq:si5-path-decomposition}\\
&=\sum_{j=0}^i\mathsf f_j(1-[\mathbf b]_j)-\mathsf f_i-\tfrac14.\label{eq:si5-action-limit}
\end{align}
\end{subequations}
Here~\eqref{eq:si5-path-decomposition} follows from~\eqref{eq:path-decomposition},
and~\eqref{eq:si5-action-limit} uses the limits
in~\eqref{proof:29} and~\eqref{eq:encoding-value-limit}.
With \(\kappa\) as defined in~\eqref{eq:first-zero-bit},
the limit in~\eqref{eq:si5-action-limit} equals \(-1/4\) when \(i=\kappa\).
If \(i>\kappa\) and \([\mathbf b]_i=0\), it is at least
\(\mathsf f_\kappa-1/4>0\).
If \([\mathbf b]_i=1\), it is at most \(-5/4\) because
\(\sum\nolimits_{j=0}^{i-1}\mathsf f_j\le\mathsf f_i-1\) by~\eqref{proof:1}.
We conclude that \(\pi_{t+1}(s_{i,5})=a_1\) exactly when
\(i>\kappa\) and \([\mathbf b]_i=0\).

\vspace{5pt}\noindent\emph{Determine \(\pi_{t+1}(s_{i,4})\).}
By~\eqref{proof:32} and~\eqref{eq:encoding-policy},
the current policy satisfies \(\pi_t(s_{i,5})=a_0\).
For any policy \(\pi\) with \(\pi(s_{i,5})=a_0\), action \(a_0\)
at \(s_{i,4}\) followed by \(\pi\) reaches \(s_{i+1}\) in \(2L\)
transitions with
\(\widetilde R^\pi((s_{i,4},a_0)\leadsto s_{i+1})=\mathsf R_i(\gamma)\).
Action \(a_1\) reaches \(s_i^{\mathrm c}\) in the same number of transitions.
For \(i<\kappa\), the path from \(s_i^{\mathrm c}\) under \(\pi_t\) reaches
\(s_{-1}^{\mathrm c}\) with zero rewards because \([\mathbf b]_j=1\)
for \(0\le j\le i\).
By~\eqref{eq:rewards-s2}, we obtain
\(\lim_{\gamma\to1^-}\widetilde R^{\pi_t}(s_i^{\mathrm c}\leadsto\bar s_0)
=4\mathsf f_{d-1}\).
For \(i\ge \kappa\), let \(j\) be the largest index at most \(i\) with
\([\mathbf b]_j=0\).
The paths from \(s_i^{\mathrm c}\) and \(s_i^{\mathrm z}\) reach
\(s_{j,0}^{\mathrm r}\) after the same number of zero-reward transitions
and then follow the same path to \(\bar s_0\).
Their discounted path rewards are therefore equal.
Using~\eqref{eq:reward-s14-a0}--\eqref{eq:reward-s15-a0}, we obtain
\begin{subequations}\label{eq:si4-action-difference}
\begin{align}
\lim_{\gamma\to1^-}\Delta^{\pi_t}(s_{i,4})
&=\lim_{\gamma\to1^-}\Bigl[
\gamma^{2L}\bigl(\widetilde V^{\pi_t}(s_i^{\mathrm c})-\widetilde V^{\pi_t}(s_{i+1})\bigr)
+\mathsf f_i(\gamma)-4\mathsf f_{d-1}(\gamma)-\tfrac12-\mathsf R_i(\gamma)
\Bigr]\label{eq:si4-action-comparison}\\
&=\lim_{\gamma\to1^-}\widetilde R^{\pi_t}(s_i^{\mathrm c}\leadsto\bar s_0)
+\mathsf f_i-4\mathsf f_{d-1}-\tfrac12\label{eq:si4-path-reward-limit}\\
&=\begin{cases}
\mathsf f_i-\tfrac12,&i<\kappa,\\
\sum\nolimits_{j=0}^i\mathsf f_j(1-[\mathbf b]_j)
+\mathsf f_i-4\mathsf f_{d-1}-\tfrac12,&i\ge \kappa.
\end{cases}\label{eq:si4-action-limit}
\end{align}
\end{subequations}
Using~\eqref{proof:29}, \eqref{eq:path-decomposition}, and~\eqref{eq:encoding-value-limit}
with \(\mathsf R_i(\gamma)\to0\), we obtain~\eqref{eq:si4-path-reward-limit}.
The limit in~\eqref{eq:si4-action-limit} is positive for \(i<\kappa\).
For \(i\ge \kappa\), it is less than
\(\mathsf f_i-2\mathsf f_{d-1}-1/2<0\), since
\(\sum_{j=0}^i\mathsf f_j(1-[\mathbf b]_j)<2\mathsf f_{d-1}\)
by~\eqref{proof:1}.
We conclude that \(\pi_{t+1}(s_{i,4})=a_1\) for \(i<\kappa\)
and \(a_0\) otherwise.

Under \(\pi_{t+1}\), the selected paths from \(s_{i,4}\) reach
\(s_i^{\mathrm c}\) for \(i<\kappa\), \(s_i^{\mathrm z}\) for
\(i>\kappa\) with \([\mathbf b]_i=0\), and \(s_{i+1}\) otherwise.
They enter \(\mathcal S_2\) exactly where \([\mathbf b^+]_i=0\).

\vspace{5pt}\noindent\refstepcounter{updatestep}%
\textbf{\boldmath Step~\theupdatestep. Determine the actions at \(s_i^{\mathrm c}\) and \(s_i^{\mathrm z}\).}\label{step:update-carry-states}
We have \(\pi_t(s)=a_0\) for all \(s\in\mathcal S_3\).
We derive the comparison for any policy \(\pi\) with this property
so that it can also be used in later iterations.
At \(s_i^{\mathrm z}\), both actions followed by \(\pi\) reach
\(s_{i-1}^{\mathrm z}\) in \(2L\) transitions.
By~\eqref{eq:rewards-s2}, \eqref{eq:reward-sr0-a0}, and~\eqref{eq:path-decomposition},
we obtain \(\Delta^\pi(s_i^{\mathrm z})=0-\gamma^L\mathsf f_i(\gamma)<0\).
By Table~\ref{tab:paths-s1}, either action at \(s_i^{\mathrm c}\) followed
by \(\pi\) reaches \(\bar s_0\) in \((2i+3)L\) transitions through
\(s_{-1}^{\mathrm c}\) or \(s_{-1}^{\mathrm z}\), so~\eqref{eq:path-decomposition} gives
\begin{equation}
\Delta^\pi(s_i^{\mathrm c})
=\widetilde R^\pi((s_i^{\mathrm c},a_1)\leadsto\bar s_0)
-\widetilde R^\pi((s_i^{\mathrm c},a_0)\leadsto\bar s_0).
\label{eq:carry-path-comparison}
\end{equation}
\begin{itemize}[leftmargin=*,topsep=3pt,itemsep=3pt,parsep=0pt]
\item \emph{Initial action \(a_0\) at \(s_i^{\mathrm c}\).}
By~\eqref{eq:rewards-s2} and~\eqref{eq:reward-sr0-a0}, the first
\(2L\)-transition path reaches \(s_{i-1}^{\mathrm z}\) with
\(\widetilde R^\pi((s_i^{\mathrm c},a_0)\leadsto s_{i-1}^{\mathrm z})
=\gamma^L\mathsf f_i(\gamma)\).
The path then passes through \(s_{i-1}^{\mathrm z},\ldots,s_{-1}^{\mathrm z}\)
to \(\bar s_0\), collecting \(\gamma^L\mathsf f_j(\gamma)\) on the path
from \(s_j^{\mathrm z}\) to \(s_{j-1}^{\mathrm z}\) whenever
\(\pi(s_j^{\mathrm z})=a_0\), and zero otherwise.
Summing these discounted rewards and using~\eqref{proof:29}, we obtain
\begin{equation}
\lim_{\gamma\to1^-}\widetilde R^\pi((s_i^{\mathrm c},a_0)\leadsto\bar s_0)
=\mathsf f_i+\sum_{0\le j<i:\,\pi(s_j^{\mathrm z})=a_0}\mathsf f_j
\in[\mathsf f_i,2\mathsf f_i-1],
\label{eq:carry-a0-reward-limit}
\end{equation}
where the upper bound follows from
\(\sum\nolimits_{j=0}^{i-1}\mathsf f_j\le\mathsf f_i-1\) in~\eqref{proof:1}.
\item \emph{Initial action \(a_1\) at \(s_i^{\mathrm c}\).}
If \(\pi(s_j^{\mathrm c})=a_1\) for all \(0\le j<i\), the path
reaches \(s_{-1}^{\mathrm c}\) in \(2(i+1)L\) zero-reward transitions,
followed by the \(L\)-transition path to \(\bar s_0\).
By~\eqref{eq:rewards-s2}--\eqref{eq:rewards-s3}, we have
\(\widetilde R^\pi((s_i^{\mathrm c},a_1)\leadsto\bar s_0)\to4\mathsf f_{d-1}\).
Otherwise, let \(j\) be the first index encountered in
\(i-1,\ldots,0\) with \(\pi(s_j^{\mathrm c})=a_0\).
The path has zero rewards until \(s_j^{\mathrm c}\) and then continues
through \(s_{j-1}^{\mathrm z},\ldots,s_{-1}^{\mathrm z}\) to \(\bar s_0\).
Using~\eqref{proof:1} and applying~\eqref{eq:carry-a0-reward-limit}
at index \(j\), we obtain
\begin{equation}
\lim_{\gamma\to1^-}\widetilde R^\pi((s_i^{\mathrm c},a_1)\leadsto\bar s_0)
\le\sum_{h=0}^{j}\mathsf f_h
\le\sum_{h=0}^{i-1}\mathsf f_h
\le\mathsf f_i-1.
\label{eq:carry-a1-reward-bound}
\end{equation}
\end{itemize}
Combining the two cases for initial action \(a_1\)
with~\eqref{eq:carry-path-comparison}--\eqref{eq:carry-a1-reward-bound},
we obtain
\begin{equation}
\left\{\begin{aligned}
\lim_{\gamma\to1^-}\Delta^\pi(s_i^{\mathrm c})
&\ge4\mathsf f_{d-1}-(2\mathsf f_i-1)>0,
&&\text{if }\pi(s_j^{\mathrm c})=a_1\text{ for all }0\le j<i,\\
\lim_{\gamma\to1^-}\Delta^\pi(s_i^{\mathrm c})
&\le(\mathsf f_i-1)-\mathsf f_i=-1<0,
&&\text{otherwise}.
\end{aligned}\right.
\label{eq:carry-action-comparison}
\end{equation}
By~\eqref{eq:encoding-policy} and the definition of \(\kappa\),
all \(s_j^{\mathrm c}\) with \(0\le j<i\) choose \(a_1\) under \(\pi_t\)
exactly when \(i\le \kappa\).
Using~\eqref{eq:carry-action-comparison} and
\(\Delta^{\pi_t}(s_i^{\mathrm z})=-\gamma^L\mathsf f_i(\gamma)<0\), we obtain
\begin{equation*}
\pi_{t+1}(s_i^{\mathrm c})=
\begin{cases}
a_1,&i\le \kappa,\\
a_0,&i>\kappa,
\end{cases}
\qquad
\pi_{t+1}(s_i^{\mathrm z})=a_0,\quad \forall\,0\le i<d.
\end{equation*}

\subsection{Iteration 2: encode \texorpdfstring{\(\mathbf b^+\)}{b+} in the first group}
\label{sec:update-two}
Under \(\pi_{t+1}\), every \(s_i\) chooses \(a_1\), and the selected
paths from \(s_{i,4}\) enter \(\mathcal S_2\) exactly where
\([\mathbf b^+]_i=0\).
By Step~\ref*{step:retain-old-bits} in Appendix~\ref{sec:update-one},
the paths from \(\bar s_i\) remain those of \(\pi[\mathbf b]\), so
\(\widetilde V^{\pi_{t+1}}(\bar s_i)=\widetilde V^{\pi[\mathbf b]}(\bar s_i)\)
for \(0\le i\le d\).

\vspace{5pt}\noindent\refstepcounter{updatestep}%
\textbf{\boldmath Step~\theupdatestep. Determine the actions at \(s_i\) and \(\{s_{i,j}\}_{j=1}^{3}\).}\label{step:reset-first-bits}
For any policy \(\pi\) satisfying \(\pi(s_{i,j})=a_1\) for \(1\le j\le3\),
action \(a_0\) at \(s_i\) reaches \(s_{i+1}\) in \(6L\) transitions,
while action \(a_1\) followed by \(\pi\) reaches \(s_{i,4}\) in \(4L\) transitions.
Both paths have zero rewards by~\eqref{eq:rewards-s1}, so we obtain
\begin{equation}
\Delta^{\pi}(s_i)=\gamma^{4L}\bigl(\widetilde V^{\pi}(s_{i,4})-\gamma^{2L}\widetilde V^{\pi}(s_{i+1})\bigr),\qquad
\Delta^{\pi}(s_{i,j})=\gamma^{-jL}\Delta^{\pi}(s_i),
\quad \forall\,1\le j\le3.
\label{proof:7}
\end{equation}
The same identities hold for \(\bar s_i,\bar s_{i,j}\)
when \(\pi(\bar s_{i,j})=a_1\) for \(1\le j\le3\).
By Step~\ref*{step:set-first-bits} in Appendix~\ref{sec:update-one},
we have \(\pi_{t+1}(s_{i,j})=a_1\) for \(1\le j\le3\),
so~\eqref{proof:7} applies.
For \([\mathbf b^+]_i=0\), we first compute
\(\widetilde V^{\pi_{t+1}}(s_{i,4})\) from the selected paths.
(I)~If \(i<\kappa\), then \(\pi_{t+1}(s_j^{\mathrm c})=a_1\) for \(0\le j\le i\),
so the path continues through \(s_i^{\mathrm c},\ldots,s_{-1}^{\mathrm c}\)
to \(\bar s_0\).
By~\eqref{eq:reward-s14-a1} and~\eqref{eq:rewards-s2}, we have
\(\widetilde R^{\pi_{t+1}}(s_{i,4}\leadsto\bar s_0)
\to\mathsf f_i-4\mathsf f_{d-1}-1/2+4\mathsf f_{d-1}
=\mathsf f_i-1/2\).
(II)~If \(i>\kappa\), then \([\mathbf b]_i=0\), and the path enters \(s_i^{\mathrm z}\).
Since every \(s_j^{\mathrm z}\) chooses \(a_0\),
we obtain from~\eqref{eq:reward-s14-a0}, \eqref{eq:reward-s15-a1},
and~\eqref{eq:reward-sr0-a0} with \(\mathsf R_i(\gamma)\to0\) that
\(\widetilde R^{\pi_{t+1}}(s_{i,4}\leadsto\bar s_0)
\to-\mathsf f_i-1/4+\sum\nolimits_{j=0}^i\mathsf f_j
=\sum\nolimits_{j=0}^{i-1}\mathsf f_j-1/4\).
The zero-reward path from \(s_i\) to \(s_{i,4}\) gives
\(\widetilde V^{\pi_{t+1}}(s_i)=\gamma^{4L}\widetilde V^{\pi_{t+1}}(s_{i,4})\).
By~\eqref{eq:path-decomposition} and~\eqref{eq:encoding-value-limit},
we therefore obtain
\begin{equation}
\lim_{\gamma\to1^-}\bigl(\widetilde V^{\pi_{t+1}}(s_{i,4})-\widetilde V^{\pi_{t+1}}(\bar s_0)\bigr)
=\lim_{\gamma\to1^-}\bigl(\widetilde V^{\pi_{t+1}}(s_i)-\widetilde V^{\pi_{t+1}}(\bar s_0)\bigr)
=\begin{cases}
\mathsf f_i-\tfrac12,&i<\kappa,\\
\sum\nolimits_{j=0}^{i-1}\mathsf f_j-\tfrac14,&i>\kappa,\ [\mathbf b]_i=0.
\end{cases}
\label{proof:14}
\end{equation}
To determine the sign of \(\Delta^{\pi_{t+1}}(s_i)\)
in~\eqref{proof:7}, we next bound the limit of
\(\widetilde V^{\pi_{t+1}}(s_{i+1})
-\widetilde V^{\pi_{t+1}}(\bar s_0)\).
Let \(j\) be the smallest index with \(i<j<d\) and
\([\mathbf b^+]_j=0\), taking \(j=d\) if no such index exists.
By~\eqref{eq:scaled-rewards}, the selected path from \(s_{i+1}\) to \(s_j\)
under \(\pi_{t+1}\) has \(6(j-i-1)L\) transitions and scaled reward
\(\sum_{h=i+1}^{j-1}\gamma^{(6(h-i-1)+4)L}\mathsf R_h(\gamma)\to0\).
If \(j<d\), the limit of
\(\widetilde V^{\pi_{t+1}}(s_j)-\widetilde V^{\pi_{t+1}}(\bar s_0)\)
is at least \(\sum_{h=0}^{j-1}\mathsf f_h-1/4\ge\mathsf f_i-1/4\)
and at most \(4\mathsf f_{d-1}\) by~\eqref{proof:1} and~\eqref{proof:14}.
If \(j=d\), the difference tends to \(4\mathsf f_{d-1}\) by~\eqref{eq:reward-sd-a0}.
Thus, we obtain
\begin{equation}
\mathsf f_i-\tfrac14
\le \lim_{\gamma\to1^-}\bigl(\widetilde V^{\pi_{t+1}}(s_{i+1})-\widetilde V^{\pi_{t+1}}(\bar s_0)\bigr)
\le4\mathsf f_{d-1},\qquad \forall\,0\le i<d.
\label{eq:reset-successor-bound}
\end{equation}
For \([\mathbf b^+]_i=0\), both cases in~\eqref{proof:14} are at most
\(\mathsf f_i-1/2\) by~\eqref{proof:1}.
Combining~\eqref{proof:7} with~\eqref{eq:reset-successor-bound}, we obtain
\begin{equation*}
\lim_{\gamma\to1^-}\Delta^{\pi_{t+1}}(s_i)
\le\bigl(\mathsf f_i-\tfrac12\bigr)-\bigl(\mathsf f_i-\tfrac14\bigr)
=-\tfrac14<0,\qquad \pi_{t+2}(s_i)=a_0.
\end{equation*}
For \([\mathbf b^+]_i=1\), the selected path from \(s_{i,4}\)
reaches \(s_{i+1}\), so we obtain
\begin{equation*}
\Delta^{\pi_{t+1}}(s_i)=\gamma^{4L}\mathsf R_i(\gamma)>0,
\qquad \pi_{t+2}(s_i)=a_1.
\end{equation*}
By~\eqref{proof:7}, we also obtain
\(\pi_{t+2}(s_{i,j})=a_{[\mathbf b^+]_i}\) for \(1\le j\le3\).

\vspace{5pt}\noindent\refstepcounter{updatestep}%
\textbf{\boldmath Step~\theupdatestep. Restore \(a_0\) at \(s_{i,4},s_{i,5}\).}\label{step:restore-first-paths}
By Steps~\ref*{step:retain-old-bits} and~\ref*{step:update-carry-states}
in Appendix~\ref{sec:update-one}, \(\pi_{t+1}\) selects \(a_0\) at every
\(s_j^{\mathrm z}\) and every state in \(\mathcal S_3\).
By Table~\ref{tab:paths-s1} and~\eqref{eq:carry-a0-reward-limit}, we have
\(\widetilde R^{\pi_{t+1}}(s_i^{\mathrm z}\leadsto\bar s_0)
\to\sum_{j=0}^i\mathsf f_j\), and the path estimates in the proof
of~\eqref{eq:carry-action-comparison} give
\(\lim_{\gamma\to1^-}\widetilde R^{\pi_{t+1}}(s_i^{\mathrm c}\leadsto\bar s_0)
\le4\mathsf f_{d-1}\).
Both paths reach \(\bar s_0\) in \((2i+3)L\) transitions, so we obtain
from~\eqref{proof:1}, \eqref{eq:path-decomposition}, and~\eqref{eq:encoding-value-limit}
\begin{equation}
\lim_{\gamma\to1^-}\bigl(\widetilde V^{\pi_{t+1}}(s_i^{\mathrm c})
-\widetilde V^{\pi_{t+1}}(\bar s_0)\bigr)\le4\mathsf f_{d-1},\qquad
\lim_{\gamma\to1^-}\bigl(\widetilde V^{\pi_{t+1}}(s_i^{\mathrm z})
-\widetilde V^{\pi_{t+1}}(\bar s_0)\bigr)=\sum_{j=0}^i\mathsf f_j\le2\mathsf f_i-1.
\label{eq:auxiliary-value-bounds}
\end{equation}
Both actions at \(s_{i,5}\) follow fixed \(L\)-transition paths,
so the comparison in~\eqref{eq:si5-action-comparison} also applies under \(\pi_{t+1}\).
By~\eqref{eq:reset-successor-bound} and~\eqref{eq:auxiliary-value-bounds}, we obtain
\begin{equation*}
\begin{aligned}
\lim_{\gamma\to1^-}\Delta^{\pi_{t+1}}(s_{i,5})
&=\lim_{\gamma\to1^-}\bigl(\widetilde V^{\pi_{t+1}}(s_i^{\mathrm z})
-\widetilde V^{\pi_{t+1}}(\bar s_0)\bigr)
-\lim_{\gamma\to1^-}\bigl(\widetilde V^{\pi_{t+1}}(s_{i+1})
-\widetilde V^{\pi_{t+1}}(\bar s_0)\bigr)-\mathsf f_i-\tfrac14\\
&\le(2\mathsf f_i-1)-(\mathsf f_i-\tfrac14)-\mathsf f_i-\tfrac14=-1<0.
\end{aligned}
\end{equation*}
If \(\pi_{t+1}(s_{i,5})=a_0\),
the path comparison in~\eqref{eq:si4-action-comparison} applies under \(\pi_{t+1}\).
Using~\eqref{eq:reset-successor-bound} and~\eqref{eq:auxiliary-value-bounds}
with \(\mathsf R_i(\gamma)\to0\), we obtain
\begin{equation*}
\begin{aligned}
\lim_{\gamma\to1^-}\Delta^{\pi_{t+1}}(s_{i,4})
&=\lim_{\gamma\to1^-}\bigl(\widetilde V^{\pi_{t+1}}(s_i^{\mathrm c})
-\widetilde V^{\pi_{t+1}}(\bar s_0)\bigr)
-\lim_{\gamma\to1^-}\bigl(\widetilde V^{\pi_{t+1}}(s_{i+1})
-\widetilde V^{\pi_{t+1}}(\bar s_0)\bigr)
+\mathsf f_i-4\mathsf f_{d-1}-\tfrac12\\
&\le4\mathsf f_{d-1}-(\mathsf f_i-\tfrac14)
+\mathsf f_i-4\mathsf f_{d-1}-\tfrac12=-\tfrac14<0.
\end{aligned}
\end{equation*}
If \(\pi_{t+1}(s_{i,5})=a_1\),
Step~\ref*{step:select-reset-paths} in Appendix~\ref{sec:update-one} gives \(i>\kappa\),
and Step~\ref*{step:update-carry-states} gives
\(\pi_{t+1}(s_i^{\mathrm c})=\pi_{t+1}(s_i^{\mathrm z})=a_0\).
By Table~\ref{tab:paths-s1} and~\eqref{eq:rewards-s2}, the selected paths from
\(s_i^{\mathrm c}\) and \(s_i^{\mathrm z}\) reach \(s_{i,0}^{\mathrm r}\)
in \(L\) transitions with zero rewards, giving
\(\widetilde V^{\pi_{t+1}}(s_i^{\mathrm c})
=\gamma^L\widetilde V^{\pi_{t+1}}(s_{i,0}^{\mathrm r})
=\widetilde V^{\pi_{t+1}}(s_i^{\mathrm z})\).
From \(s_{i,4}\), action \(a_1\) reaches \(s_i^{\mathrm c}\)
in \(2L\) transitions, whereas action \(a_0\) followed by \(\pi_{t+1}\)
reaches \(s_i^{\mathrm z}\) through \(s_{i,5}\) in \(2L\) transitions.
By~\eqref{eq:reward-s14-a0}, \eqref{eq:reward-s14-a1}, and~\eqref{eq:reward-s15-a1}, we obtain
\begin{equation*}
\lim_{\gamma\to1^-}\Delta^{\pi_{t+1}}(s_{i,4})
=\lim_{\gamma\to1^-}\Bigl\{
\bigl[\mathsf f_i(\gamma)-4\mathsf f_{d-1}(\gamma)-\tfrac12\bigr]
-\bigl[\mathsf R_i(\gamma)-\gamma^L(\mathsf f_i(\gamma)+\tfrac14)\bigr]\Bigr\}
=-4\mathsf f_{d-1}+2\mathsf f_i-\tfrac14<0.
\end{equation*}
We conclude that \(\pi_{t+2}(s_{i,4})=\pi_{t+2}(s_{i,5})=a_0\)
for \(0\le i<d\).

\vspace{5pt}\noindent\refstepcounter{updatestep}%
\textbf{\boldmath Step~\theupdatestep. Determine the actions at \(s_d,\bar s_{i,4}\) and in \(\mathcal S_3\).}\label{step:second-reference-actions}
We first show that \(\pi_{t+2}(s)=a_0\) for
\(s\in\{\bar s_{i,4}:0\le i<d\}\cup\mathcal S_3\).
By Table~\ref{tab:paths-s1}, the \(a_1\) paths from
\(\bar s_{i,4}\) and the states in \(\mathcal S_3\) end at
\(s_i,s_{i,1},s_{i+1}\).
By Step~\ref*{step:set-first-bits} in Appendix~\ref{sec:update-one}, we have
\(\widetilde V^{\pi_{t+1}}(s_i)
=\gamma^L\widetilde V^{\pi_{t+1}}(s_{i,1})
=\gamma^{4L}\widetilde V^{\pi_{t+1}}(s_{i,4})\).
For \([\mathbf b^+]_i=0\), \eqref{proof:1} and~\eqref{proof:14} bound the limit of
\(\widetilde V^{\pi_{t+1}}(s_{i,4})-\widetilde V^{\pi_{t+1}}(\bar s_0)\)
by \(\mathsf f_i-1/2<4\mathsf f_{d-1}\).
For \([\mathbf b^+]_i=1\), the selected path gives
\(\widetilde V^{\pi_{t+1}}(s_{i,4})=\mathsf R_i(\gamma)
+\gamma^{2L}\widetilde V^{\pi_{t+1}}(s_{i+1})\).
Using~\eqref{eq:reset-successor-bound} and \(\mathsf R_i(\gamma)\to0\), we obtain
\begin{equation}
\lim_{\gamma\to1^-}\bigl(\widetilde V^{\pi_{t+1}}(s)
-\widetilde V^{\pi_{t+1}}(\bar s_0)\bigr)\le4\mathsf f_{d-1},
\qquad \forall\,s\in\{s_i,s_{i,1},s_{i+1}\},\quad \forall\,0\le i<d.
\label{eq:second-endpoint-value-bound}
\end{equation}
By Step~\ref*{step:retain-old-bits} in Appendix~\ref{sec:update-one},
we have \(\pi_{t+1}(s)=a_0\) for
\(s\in\{\bar s_{i,4}:0\le i<d\}\cup\mathcal S_3\),
and the trajectory from \(\bar s_0\) is unchanged.
The path reward estimates in the proofs of~\eqref{eq:initial-a0-value-bound}
and~\eqref{eq:initial-a1-value-bound} therefore give
\(\lim_{\gamma\to1^-}\widetilde R^{\pi_{t+1}}(s\leadsto\bar s_0)\ge0\) and
\(\lim_{\gamma\to1^-}\widetilde R^{\pi_{t+1}}((s,a_1)\leadsto s')\le-8\mathsf f_{d-1}\),
where \(s'\in\{s_i,s_{i,1},s_{i+1}\}\) for some \(0\le i<d\).
Using~\eqref{eq:path-decomposition} and~\eqref{eq:second-endpoint-value-bound}, we obtain
\begin{equation}
\lim_{\gamma\to1^-}\Delta^{\pi_{t+1}}(s)
\le(-8\mathsf f_{d-1}+4\mathsf f_{d-1})-0=-4\mathsf f_{d-1}<0,
\qquad \forall\,s\in\{\bar s_{i,4}:0\le i<d\}\cup\mathcal S_3.
\label{eq:second-reference-comparison}
\end{equation}
To determine \(\pi_{t+2}(s_d)\), let
\(j=\min\{0\le i<d:[\mathbf b^+]_i=0\}\).
Under \(\pi_{t+1}\), the selected path from \(s_0\) to \(s_j\)
takes \(6jL\) transitions.
By~\eqref{eq:scaled-rewards} and~\eqref{eq:path-decomposition}, we have
\begin{equation*}
\widetilde V^{\pi_{t+1}}(s_0)
-\widetilde V^{\pi_{t+1}}(s_j)
=\sum_{i=0}^{j-1}\gamma^{(6i+4)L}\mathsf R_i(\gamma)
+(\gamma^{6jL}-1)\widetilde V^{\pi_{t+1}}(s_j)
\longrightarrow0,
\end{equation*}
where \(\mathsf R_i(\gamma)\to0\) by~\eqref{eq:scaled-rewards},
and \(\widetilde V^{\pi_{t+1}}(s_j)\) has a finite limit
by~\eqref{eq:encoding-value-limit} and~\eqref{proof:14}.
Combining this limit with~\eqref{proof:1}, \eqref{eq:reward-sd-a0},
and~\eqref{proof:14}, we obtain
\begin{equation*}
\lim_{\gamma\to1^-}\Delta^{\pi_{t+1}}(s_d)
=\lim_{\gamma\to1^-}\bigl(\widetilde V^{\pi_{t+1}}(s_j)
-\widetilde V^{\pi_{t+1}}(\bar s_0)\bigr)-4\mathsf f_{d-1}
\le\mathsf f_j-\tfrac12-4\mathsf f_{d-1}<0.
\end{equation*}
We conclude that \(\pi_{t+2}(s)=a_0\) for
\(s\in\{s_d\}\cup\{\bar s_{i,4}:0\le i<d\}\cup\mathcal S_3\).

\vspace{5pt}\noindent\refstepcounter{updatestep}%
\begingroup\emergencystretch=2em
\textbf{\boldmath Step~\theupdatestep. Determine the actions at \(\bar s_i,\{\bar s_{i,j}\}_{j=1}^{3},s_i^{\mathrm c},s_i^{\mathrm z}\).}\label{step:second-delayed-bits}
For \([\mathbf b]_i=0\),
Step~\ref*{step:retain-old-bits} in Appendix~\ref{sec:update-one} gives
\(\pi_{t+1}(\bar s_i)=\pi_{t+1}(\bar s_{i,1})
=\pi_{t+1}(\bar s_{i,2})=a_0\) and
\(\pi_{t+1}(\bar s_{i,3})=a_1\).
Both actions at \(\bar s_i\), followed by \(\pi_{t+1}\), reach
\(\bar s_{i+1}\) in \(6L\) zero-reward transitions, and both actions at
\(\bar s_{i,1}\) do so in \(5L\) transitions.
At \(\bar s_{i,2}\), the \(a_1\) path now passes through
\(\bar s_{i,3},\bar s_{i,4}\), so~\eqref{eq:reward-bar-s14-a0} gives
\begin{equation*}
\Delta^{\pi_{t+1}}(\bar s_i)
=\Delta^{\pi_{t+1}}(\bar s_{i,1})=0,\qquad
\Delta^{\pi_{t+1}}(\bar s_{i,2})=\gamma^{2L}\mathsf R_i(\gamma)>0.
\end{equation*}
The same path comparison gives
\(\Delta^{\pi_{t+1}}(\bar s_{i,3})=\gamma^L\mathsf R_i(\gamma)>0\)
for every \(i\). When \([\mathbf b]_i=1\), the positive action differences
at \(\bar s_i,\bar s_{i,1},\bar s_{i,2}\) from
Step~\ref*{step:retain-old-bits} are unchanged.
We therefore obtain
\begin{equation*}
\pi_{t+2}(\bar s_i)=\pi_{t+2}(\bar s_{i,1})=a_{[\mathbf b]_i},
\qquad \pi_{t+2}(\bar s_{i,2})=\pi_{t+2}(\bar s_{i,3})=a_1,
\quad \forall\,0\le i<d.
\end{equation*}
If \([\mathbf b]_i=0\), the unchanged \(a_0\) path from
\(\bar s_i\) bypasses \(\bar s_{i,2}\).
By Step~\ref*{step:second-reference-actions}, \(\bar s_{i,4}\) also keeps \(a_0\),
so \(\pi_{t+1}\) and \(\pi_{t+2}\) generate the same trajectory and rewards
from \(\bar s_0\).

By Steps~\ref*{step:retain-old-bits} and~\ref*{step:update-carry-states}
in Appendix~\ref{sec:update-one}, \(\pi_{t+1}\) selects \(a_0\)
at every state in \(\mathcal S_3\), and all \(s_j^{\mathrm c}\)
with \(0\le j<i\) select \(a_1\) exactly when \(i\le \kappa+1\).
Using~\eqref{eq:carry-action-comparison} and
\(\Delta^{\pi_{t+1}}(s_i^{\mathrm z})=-\gamma^L\mathsf f_i(\gamma)<0\),
we obtain \(\pi_{t+2}(s_i^{\mathrm c})=a_1\) exactly when \(i\le \kappa+1\),
and \(\pi_{t+2}(s_i^{\mathrm z})=a_0\) for \(0\le i<d\).
\par\endgroup

\subsection{Iteration 3: form the new cycle}
\label{sec:update-three}
Under \(\pi_{t+2}\), the selected paths from \(s_i\) to \(s_{i+1}\)
encode \(\mathbf b^+\) and lead from \(s_0\) to \(s_d\), where
\(a_0\) still leads to \(\bar s_0\).
By Step~\ref*{step:second-delayed-bits} in Appendix~\ref{sec:update-two},
the paths from \(\bar s_i\) remain those of \(\pi[\mathbf b]\), so
\(\widetilde V^{\pi_{t+2}}(\bar s_0)=\widetilde V^{\pi[\mathbf b]}(\bar s_0)\).

\vspace{5pt}\noindent\refstepcounter{updatestep}%
\textbf{\boldmath Step~\theupdatestep. Select \(a_1\) at \(s_d\).}
Starting from \(s_d\), action \(a_1\) followed by \(\pi_{t+2}\)
returns to \(s_d\) after \((6d+1)L\) transitions.
This path has the same length and reward sequence as one cycle starting
from \(s_d\) under \(\pi[\mathbf b^+]\), so
\(\widetilde R^{\pi_{t+2}}((s_d,a_1)\leadsto s_d)
=(1-\gamma^{(6d+1)L})\widetilde V^{\pi[\mathbf b^+]}(s_d)\).
Since \(\pi_{t+2}(s_d)=a_0\), we apply~\eqref{eq:path-decomposition} to obtain
\begin{equation}
\begin{aligned}
\Delta^{\pi_{t+2}}(s_d)
&=\bigl[(1-\gamma^{(6d+1)L})\widetilde V^{\pi[\mathbf b^+]}(s_d)
+\gamma^{(6d+1)L}\widetilde V^{\pi_{t+2}}(s_d)\bigr]
-\widetilde V^{\pi_{t+2}}(s_d)\\
&=(1-\gamma^{(6d+1)L})
\bigl(\widetilde V^{\pi[\mathbf b^+]}(s_d)-\widetilde V^{\pi_{t+2}}(s_d)\bigr).
\end{aligned}
\label{proof:17}
\end{equation}
By~\eqref{eq:reward-sd-a0},
\(\widetilde V^{\pi_{t+2}}(s_d)=4\mathsf f_{d-1}(\gamma)+\gamma^L\widetilde V^{\pi[\mathbf b]}(\bar s_0)\).
We combine this identity with~\eqref{eq:increment-limit}
and~\eqref{eq:encoding-value-limit} to obtain
\begin{equation}
\lim_{\gamma\to1^-}
\bigl(\widetilde V^{\pi[\mathbf b^+]}(s_d)-\widetilde V^{\pi_{t+2}}(s_d)\bigr)
=4\Bigl(\mathsf g_\kappa-\sum_{j=0}^{\kappa-1}\mathsf g_j\Bigr)-4\mathsf f_{d-1}
\ge12\mathsf f_{d-1}>0,
\label{eq:third-cycle-value-gap}
\end{equation}
where we use
\(\sum\nolimits_{j=0}^{\kappa-1}\mathsf g_j\le\mathsf g_\kappa-\mathsf g_0\)
and \(\mathsf g_0\ge4\mathsf f_{d-1}\) in~\eqref{proof:1}.
Combining~\eqref{proof:17} and~\eqref{eq:third-cycle-value-gap},
we obtain \mbox{\(\pi_{t+3}(s_d)=a_1\)}.

\vspace{5pt}\noindent\refstepcounter{updatestep}%
\textbf{\boldmath Step~\theupdatestep. Keep \(a_0\) at \(s_{i,4},s_{i,5},\bar s_{i,4}\) and in \(\mathcal S_3\).}\label{step:third-fixed-actions}
By~\eqref{eq:scaled-rewards}, the selected path from \(s_i\)
to \(s_d\) satisfies
\begin{equation*}
\widetilde R^{\pi_{t+2}}(s_i\leadsto s_d)
=\sum_{h=i}^{d-1}[\mathbf b^+]_h\gamma^{(6(h-i)+4)L}\mathsf R_h(\gamma)
\longrightarrow0,\qquad \forall\,0\le i<d.
\end{equation*}
We also have
\(\widetilde V^{\pi_{t+2}}(s_i)=\gamma^L\widetilde V^{\pi_{t+2}}(s_{i,1})\).
When \([\mathbf b^+]_i=0\), this follows from the zero-reward paths of lengths
\(6L\) and \(5L\) from \(s_i\) and \(s_{i,1}\) to \(s_{i+1}\).
Using~\eqref{eq:path-decomposition} and~\eqref{eq:encoding-value-limit}
with the value at \(s_d\) computed above, we obtain
\begin{equation}
\lim_{\gamma\to1^-}\bigl(\widetilde V^{\pi_{t+2}}(s)
-\widetilde V^{\pi_{t+2}}(\bar s_0)\bigr)=4\mathsf f_{d-1},
\qquad \forall\,s\in\{s_i,s_{i,1},s_{i+1}\},\quad \forall\,0\le i<d.
\label{eq:third-endpoint-value-limit}
\end{equation}
By Steps~\ref*{step:second-reference-actions} and~\ref*{step:second-delayed-bits}
in Appendix~\ref{sec:update-two}, \(\pi_{t+2}\) selects \(a_0\)
at \(\bar s_{i,4},s_i^{\mathrm z}\) and every state in \(\mathcal S_3\).
The path reward bounds used in~\eqref{eq:auxiliary-value-bounds}
and~\eqref{eq:second-reference-comparison} therefore remain valid.
Applying the proof of~\eqref{eq:second-reference-comparison}
with~\eqref{eq:third-endpoint-value-limit}, we obtain
\begin{equation*}
\lim_{\gamma\to1^-}\Delta^{\pi_{t+2}}(s)
\le-4\mathsf f_{d-1}<0,
\qquad \forall\,s\in\{\bar s_{i,4}:0\le i<d\}\cup\mathcal S_3.
\end{equation*}
We also use the bounds in~\eqref{eq:auxiliary-value-bounds}
under \(\pi_{t+2}\) to compare the actions at \(s_{i,5}\) and \(s_{i,4}\).
Since \(\pi_{t+2}(s_{i,5})=a_0\) by Step~\ref*{step:restore-first-paths}
in Appendix~\ref{sec:update-two}, the comparisons in~\eqref{eq:si5-action-comparison}
and~\eqref{eq:si4-action-comparison} apply under \(\pi_{t+2}\).
Using~\eqref{eq:third-endpoint-value-limit}, we obtain
\begin{equation*}
\begin{aligned}
\lim_{\gamma\to1^-}\Delta^{\pi_{t+2}}(s_{i,5})
&\le(2\mathsf f_i-1)-4\mathsf f_{d-1}-\mathsf f_i-\tfrac14
=\mathsf f_i-4\mathsf f_{d-1}-\tfrac54<0,\\
\lim_{\gamma\to1^-}\Delta^{\pi_{t+2}}(s_{i,4})
&\le4\mathsf f_{d-1}-4\mathsf f_{d-1}+\mathsf f_i-4\mathsf f_{d-1}-\tfrac12
=\mathsf f_i-4\mathsf f_{d-1}-\tfrac12<0.
\end{aligned}
\end{equation*}
We conclude that
\(\pi_{t+3}(s_{i,4})=\pi_{t+3}(s_{i,5})=\pi_{t+3}(\bar s_{i,4})=a_0\)
for \(0\le i<d\), and \(\pi_{t+3}(s)=a_0\) for \(s\in\mathcal S_3\).

\vspace{5pt}\noindent\refstepcounter{updatestep}%
\begingroup\emergencystretch=2em
\textbf{\boldmath Step~\theupdatestep. Determine the actions at \(s_i,\bar s_i,\{s_{i,j},\bar s_{i,j}\}_{j=1}^{3},s_i^{\mathrm c},s_i^{\mathrm z}\).}\label{step:third-delayed-bits}
The current policy satisfies
\(\pi_{t+2}(s_{i,4})=\pi_{t+2}(s_{i,5})=\pi_{t+2}(\bar s_{i,4})=a_0\).
Thus the selected \(2L\)-transition paths from \(s_{i,4}\) to \(s_{i+1}\)
and from \(\bar s_{i,4}\) to \(\bar s_{i+1}\) have the same reward sequence,
and the equal-length path comparison in Step~\ref*{step:retain-old-bits} in Appendix~\ref{sec:update-one}
applies to both groups.
For \([\mathbf b^+]_i=0\), we obtain
\begin{equation*}
\Delta^{\pi_{t+2}}(s_i)=\Delta^{\pi_{t+2}}(s_{i,1})
=\Delta^{\pi_{t+2}}(s_{i,2})=0,\qquad
\Delta^{\pi_{t+2}}(s_{i,3})=\gamma^L\mathsf R_i(\gamma)>0.
\end{equation*}
We therefore obtain \(\pi_{t+3}(s_{i,3})=a_1\), while
\(s_i,s_{i,1},s_{i,2}\) keep \(a_0\).
For \([\mathbf b^+]_i=1\), the positive action differences keep
\(a_1\) at \(s_i\) and \(\{s_{i,j}\}_{j=1}^{3}\).
Using the choices from Step~\ref*{step:second-delayed-bits}
in Appendix~\ref{sec:update-two}, we obtain for \([\mathbf b]_i=0\)
\begin{equation*}
\Delta^{\pi_{t+2}}(\bar s_i)=0,\qquad
\Delta^{\pi_{t+2}}(\bar s_{i,1})=\gamma^{3L}\mathsf R_i(\gamma)>0.
\end{equation*}
We obtain \(\pi_{t+3}(\bar s_i)=a_0\) and \(\pi_{t+3}(\bar s_{i,1})=a_1\).
The positive action differences keep \(a_1\) at \(\bar s_{i,2},\bar s_{i,3}\)
for every \(i\), and at \(\bar s_i,\bar s_{i,1}\) when \([\mathbf b]_i=1\).
The selected \(a_0\) path from \(s_i\) when \([\mathbf b^+]_i=0\)
bypasses \(s_{i,3}\), and that from \(\bar s_i\) when \([\mathbf b]_i=0\)
bypasses \(\bar s_{i,1}\), so these intermediate changes do not alter either encoded path.

By Step~\ref*{step:second-delayed-bits} in Appendix~\ref{sec:update-two},
all \(s_j^{\mathrm c}\) with \(0\le j<i\) select \(a_1\) under \(\pi_{t+2}\)
exactly when \(i\le \kappa+2\).
Applying~\eqref{eq:carry-action-comparison} to \(\pi_{t+2}\) and using
\(\Delta^{\pi_{t+2}}(s_i^{\mathrm z})=-\gamma^L\mathsf f_i(\gamma)<0\),
we obtain \(\pi_{t+3}(s_i^{\mathrm c})=a_1\) exactly when \(i\le \kappa+2\),
and \(\pi_{t+3}(s_i^{\mathrm z})=a_0\) for \(0\le i<d\).
\par\endgroup

\subsection{Iteration 4: connect to the new cycle}
\label{sec:update-four}
Under \(\pi_{t+3}\), the cycles through \(s_0\) and \(\bar s_0\)
encode \(\mathbf b^+\) and \(\mathbf b\), respectively.
The selected paths therefore give
\begin{equation*}
\widetilde V^{\pi_{t+3}}(s_i)=\widetilde V^{\pi[\mathbf b^+]}(s_i),\qquad
\widetilde V^{\pi_{t+3}}(\bar s_i)=\widetilde V^{\pi[\mathbf b]}(\bar s_i),
\qquad \forall\,0\le i\le d.
\end{equation*}
By~\eqref{proof:1}, \eqref{eq:increment-limit}, and~\eqref{eq:encoding-value-limit},
we obtain
\begin{equation}
\lim_{\gamma\to1^-}\bigl(\widetilde V^{\pi_{t+3}}(s_i)-\widetilde V^{\pi_{t+3}}(\bar s_0)\bigr)
=4\Bigl(\mathsf g_\kappa-\sum_{j=0}^{\kappa-1}\mathsf g_j\Bigr)
\ge16\mathsf f_{d-1},\qquad \forall\,0\le i\le d.
\label{eq:copy-cycle-gap}
\end{equation}

\vspace{5pt}\noindent\refstepcounter{updatestep}%
\textbf{\boldmath Step~\theupdatestep. Select \(a_1\) at \(\bar s_i,\{\bar s_{i,j}\}_{j=1}^{4}\) and in \(\mathcal S_3\).}\label{step:fourth-copying-actions}
For \(0\le i<d\) and
\(s\in\{\bar s_{i,4},s_{i,0}^{\mathrm r},s_{i,1}^{\mathrm r},s_{i,2}^{\mathrm r}\}\),
we have \(\pi_{t+3}(s)=a_0\) by Step~\ref*{step:third-fixed-actions}
in Appendix~\ref{sec:update-three}.
By Table~\ref{tab:paths-s1}, action \(a_0\) followed by \(\pi_{t+3}\)
reaches \(\bar s_0\), whereas the \(a_1\) path ends at
\(s_i,s_{i,1}\), or \(s_{i+1}\).
The selected path from \(s_{i,1}\) gives
\(\widetilde V^{\pi_{t+3}}(s_{i,1})
=[\mathbf b^+]_i\gamma^{3L}\mathsf R_i(\gamma)
+\gamma^{5L}\widetilde V^{\pi_{t+3}}(s_{i+1})\).
Using \(\mathsf R_i(\gamma)\to0\) and the path reward bounds from the proofs
of~\eqref{eq:initial-a0-value-bound} and~\eqref{eq:initial-a1-value-bound},
we obtain from~\eqref{eq:path-decomposition} and~\eqref{eq:copy-cycle-gap}
\begin{equation}
\left\{\begin{aligned}
\lim_{\gamma\to1^-}\bigl(\widetilde Q^{\pi_{t+3}}(s,a_1)
-\widetilde V^{\pi_{t+3}}(\bar s_0)\bigr)
&\ge16\mathsf f_{d-1}-8\mathsf f_{d-1}-\mathsf f_i,\\
\lim_{\gamma\to1^-}\bigl(\widetilde Q^{\pi_{t+3}}(s,a_0)
-\widetilde V^{\pi_{t+3}}(\bar s_0)\bigr)
&\le4\mathsf f_{d-1}.
\end{aligned}\right.
\label{eq:fourth-action-value-bounds}
\end{equation}
Using~\eqref{eq:action-difference}
and~\eqref{eq:fourth-action-value-bounds}, we obtain
\begin{equation*}
\lim_{\gamma\to1^-}\Delta^{\pi_{t+3}}(s)
\ge16\mathsf f_{d-1}-(8\mathsf f_{d-1}+\mathsf f_i)-4\mathsf f_{d-1}
=4\mathsf f_{d-1}-\mathsf f_i>0.
\end{equation*}
By Steps~\ref*{step:third-fixed-actions} and~\ref*{step:third-delayed-bits}
in Appendix~\ref{sec:update-three}, we have \(\pi_{t+3}(\bar s_{i,4})=a_0\)
and \(\pi_{t+3}(\bar s_{i,j})=a_1\) for \(1\le j\le3\).
The path rewards in~\eqref{eq:scaled-rewards} and the comparisons in~\eqref{proof:7} give
\begin{equation*}
\Delta^{\pi_{t+3}}(\bar s_i)=\gamma^{4L}\mathsf R_i(\gamma)>0,\qquad
\Delta^{\pi_{t+3}}(\bar s_{i,j})=\gamma^{(4-j)L}\mathsf R_i(\gamma)>0,
\quad \forall\,1\le j\le3.
\end{equation*}
We conclude that \(\pi_{t+4}\) selects \(a_1\) at
\(\bar s_i\), \(\{\bar s_{i,j}\}_{j=1}^{4}\), and every state in \(\mathcal S_3\).

\vspace{5pt}\noindent\refstepcounter{updatestep}%
\textbf{\boldmath Step~\theupdatestep. Determine the actions at \(s_d,s_i,\{s_{i,j}\}_{j=1}^{5},s_i^{\mathrm c},s_i^{\mathrm z}\).}\label{step:fourth-first-actions}
By Steps~\ref*{step:third-fixed-actions} and~\ref*{step:third-delayed-bits}
in Appendix~\ref{sec:update-three}, \(\pi_{t+3}\) selects \(a_0\) at
\(s_{i,4},s_{i,5},s_i^{\mathrm z}\) and every state in \(\mathcal S_3\).
The bounds in~\eqref{eq:auxiliary-value-bounds} therefore remain valid,
and the comparisons in~\eqref{eq:si5-action-comparison} and~\eqref{eq:si4-action-comparison}
apply under \(\pi_{t+3}\).
Using~\eqref{eq:copy-cycle-gap}, we obtain for \(0\le i<d\)
\begin{equation*}
\lim_{\gamma\to1^-}\Delta^{\pi_{t+3}}(s)
\le\mathsf f_i-\tfrac12-16\mathsf f_{d-1}<0,
\qquad \pi_{t+4}(s)=a_0,
\quad \forall\,s\in\{s_{i,4},s_{i,5}\}.
\end{equation*}
At \(s_d\), we obtain from~\eqref{eq:reward-sd-a0}
and~\eqref{eq:copy-cycle-gap}
\begin{equation*}
\lim_{\gamma\to1^-}\Delta^{\pi_{t+3}}(s_d)
=\lim_{\gamma\to1^-}\bigl(\widetilde V^{\pi_{t+3}}(s_0)
-\widetilde V^{\pi_{t+3}}(\bar s_0)\bigr)-4\mathsf f_{d-1}
\ge12\mathsf f_{d-1}>0,\qquad \pi_{t+4}(s_d)=a_1.
\end{equation*}
For \([\mathbf b^+]_i=0\),
Step~\ref*{step:third-delayed-bits} in Appendix~\ref{sec:update-three}
gives \(\pi_{t+3}(s_i)=\pi_{t+3}(s_{i,1})=\pi_{t+3}(s_{i,2})=a_0\)
and \(\pi_{t+3}(s_{i,3})=a_1\).
The path comparison at \(\bar s_i,\{\bar s_{i,j}\}_{j=1}^{3}\)
in Step~\ref*{step:second-delayed-bits} in Appendix~\ref{sec:update-two}
therefore gives
\begin{equation*}
\Delta^{\pi_{t+3}}(s_i)=\Delta^{\pi_{t+3}}(s_{i,1})=0,
\qquad \Delta^{\pi_{t+3}}(s_{i,2})=\gamma^{2L}\mathsf R_i(\gamma)>0.
\end{equation*}
We conclude that \(s_{i,2}\) switches to \(a_1\),
while \(s_i,s_{i,1}\) keep \(a_0\).
The same comparison gives \(\Delta^{\pi_{t+3}}(s_{i,3})=\gamma^L\mathsf R_i(\gamma)>0\)
for every \(i\), and keeps \(a_1\) at \(s_i\) and \(\{s_{i,j}\}_{j=1}^{3}\)
when \([\mathbf b^+]_i=1\).
When \([\mathbf b^+]_i=0\), the selected \(a_0\) path from \(s_i\)
bypasses \(s_{i,2}\), so the switch at \(s_{i,2}\) does not alter the cycle through \(s_0\).

By Step~\ref*{step:third-delayed-bits} in Appendix~\ref{sec:update-three},
all \(s_j^{\mathrm c}\) with \(0\le j<i\) select \(a_1\) under \(\pi_{t+3}\)
exactly when \(i\le \kappa+3\).
Applying~\eqref{eq:carry-action-comparison} and using
\(\Delta^{\pi_{t+3}}(s_i^{\mathrm z})=-\gamma^L\mathsf f_i(\gamma)<0\),
we obtain \(\pi_{t+4}(s_i^{\mathrm c})=a_1\) exactly when \(i\le \kappa+3\),
and \(\pi_{t+4}(s_i^{\mathrm z})=a_0\) for \(0\le i<d\).

\subsection{Iteration 5: encode \texorpdfstring{\(\mathbf b^+\)}{b+} in the second group}
\label{sec:update-five}
By Steps~\ref*{step:fourth-copying-actions} and~\ref*{step:fourth-first-actions}
in Appendix~\ref{sec:update-four}, the cycle through \(s_0\) under \(\pi_{t+4}\)
still encodes \(\mathbf b^+\), and every \(\bar s_{i,4}\) and every state
in \(\mathcal S_3\) chooses \(a_1\).

\vspace{5pt}\noindent\refstepcounter{updatestep}%
\textbf{\boldmath Step~\theupdatestep. Copy \(\mathbf b^+\) to \(\bar s_i,\{\bar s_{i,j}\}_{j=1}^{3},s_i^{\mathrm c},s_i^{\mathrm z}\).}\label{step:copy-new-bits}
We express the action differences at \(\bar s_i,s_i^{\mathrm c},s_i^{\mathrm z}\)
in terms of \(\Delta^{\pi_{t+4}}(s_i)\).
The unchanged cycle gives
\(\widetilde V^{\pi_{t+4}}(s_i)=\widetilde V^{\pi[\mathbf b^+]}(s_i)\)
for \(0\le i\le d\).
For \([\mathbf b^+]_i=0\),
Step~\ref*{step:fourth-first-actions} in Appendix~\ref{sec:update-four}
gives \(\pi_{t+4}(s_{i,1})=a_0\), so both actions at \(s_i\)
followed by \(\pi_{t+4}\) reach \(s_{i+1}\) in \(6L\) transitions with zero rewards.
For \([\mathbf b^+]_i=1\), action \(a_1\) follows the full horizontal path.
Since \(\widetilde Q^{\pi_{t+4}}(s_i,a_1)=\widetilde V^{\pi_{t+4}}(s_i)\)
in both cases, we obtain from~\eqref{eq:scaled-rewards}
\begin{equation}
\widetilde V^{\pi_{t+4}}(s_i)-\gamma^{6L}\widetilde V^{\pi_{t+4}}(s_{i+1})
=\Delta^{\pi_{t+4}}(s_i)
=[\mathbf b^+]_i\gamma^{4L}\mathsf R_i(\gamma),
\qquad \forall\,0\le i<d.
\label{eq:copy-first-bit-comparison}
\end{equation}
By Step~\ref*{step:fourth-copying-actions} in Appendix~\ref{sec:update-four},
\(\pi_{t+4}\) selects \(a_1\) at \(\bar s_i\) and \(\{\bar s_{i,j}\}_{j=1}^{4}\).
The path from \(\bar s_i\) to \(s_i\) therefore has \(4L\) zero-reward transitions
followed by the \(L\)-transition path in~\eqref{eq:reward-bar-s14-a1}, giving
\begin{equation}
\widetilde V^{\pi_{t+4}}(\bar s_i)=\gamma^{4L}\bigl(\gamma^L\widetilde V^{\pi_{t+4}}(s_i)-8\mathsf f_{d-1}(\gamma)\bigr),
\quad \forall\,0\le i<d,\qquad
\widetilde V^{\pi_{t+4}}(\bar s_d)=\gamma^L\widetilde V^{\pi_{t+4}}(\bar s_0).
\label{proof:20}
\end{equation}
Using~\eqref{eq:copy-first-bit-comparison} and~\eqref{proof:20}, we subtract
the \(a_0\) value \(\gamma^{6L}\widetilde V^{\pi_{t+4}}(\bar s_{i+1})\)
from \(\widetilde V^{\pi_{t+4}}(\bar s_i)\) to obtain
\begin{equation}
\gamma^{-4L}\Delta^{\pi_{t+4}}(\bar s_i)
=\gamma^L\Delta^{\pi_{t+4}}(s_i)-
\begin{cases}
8\mathsf f_{d-1}(\gamma)(1-\gamma^{6L}),&0\le i<d-1,\\
8\mathsf f_{d-1}(\gamma)(1-\gamma^{7L}),&i=d-1.
\end{cases}
\label{eq:copy-return}
\end{equation}
For \(i=d-1\), the additional \(L\) transitions from \(\bar s_d\)
to \(\bar s_0\) account for the exponent \(7L\).
By~\eqref{proof:7}, we also have
\(\Delta^{\pi_{t+4}}(\bar s_{i,j})=\gamma^{-jL}\Delta^{\pi_{t+4}}(\bar s_i)\)
for \(1\le j\le3\).

For \(s\in\{s_i^{\mathrm c},s_i^{\mathrm z}\}\), the \(a_0\) path followed
by \(\pi_{t+4}\) reaches \(s_{i+1}\) in \(7L\) transitions,
and the \(a_1\) path reaches \(s_{i,1}\) in \(2L\) transitions.
The identities
\(\widetilde Q^{\pi_{t+4}}(s_i,a_0)=\gamma^{6L}\widetilde V^{\pi_{t+4}}(s_{i+1})\)
and \(\widetilde Q^{\pi_{t+4}}(s_i,a_1)=\gamma^L\widetilde V^{\pi_{t+4}}(s_{i,1})\)
let us express these action values in terms of those at \(s_i\).
By~\eqref{eq:reward-sr0-a1} and~\eqref{eq:reward-sr12-a1}, we obtain
\begin{equation}
\begin{aligned}
\widetilde Q^{\pi_{t+4}}(s,a_0)
&=\gamma^L\bigl[\widetilde Q^{\pi_{t+4}}(s_i,a_0)
-8\mathsf f_{d-1}(\gamma)-\mathsf f_i(\gamma)+1-\gamma^L\bigr],\\
\widetilde Q^{\pi_{t+4}}(s,a_1)
&=\gamma^L\bigl[\widetilde Q^{\pi_{t+4}}(s_i,a_1)
-8\mathsf f_{d-1}(\gamma)-\mathsf f_i(\gamma)\bigr].
\end{aligned}
\label{eq:copy-auxiliary-action-values}
\end{equation}
Subtracting the action values in~\eqref{eq:copy-auxiliary-action-values}, we obtain
\begin{equation}
\Delta^{\pi_{t+4}}(s)=\gamma^L\bigl(\Delta^{\pi_{t+4}}(s_i)-(1-\gamma^L)\bigr),
\qquad \forall\,s\in\{s_i^{\mathrm c},s_i^{\mathrm z}\}.
\label{eq:copy-choice}
\end{equation}
For \([\mathbf b^+]_i=0\), \eqref{eq:copy-first-bit-comparison}
gives \(\Delta^{\pi_{t+4}}(s_i)=0\).
Substituting this into \eqref{eq:copy-return} and~\eqref{eq:copy-choice},
and using \(\mathsf f_{d-1}(\gamma)>0\) for \(\gamma<1\)
sufficiently close to one, we obtain
\(\Delta^{\pi_{t+4}}(\bar s_i),\Delta^{\pi_{t+4}}(s_i^{\mathrm c}),
\Delta^{\pi_{t+4}}(s_i^{\mathrm z})<0\).
For \([\mathbf b^+]_i=1\), these action differences tend to zero.
To determine their signs, we divide~\eqref{eq:copy-return}
and~\eqref{eq:copy-choice} by \(1-\gamma^L>0\).
Since \((1-\gamma^{7L})/(1-\gamma^L)\to7\), the subtracted terms
have limits at most \(56\mathsf f_{d-1}\) and \(1\), respectively.
Using~\eqref{proof:1} and~\eqref{eq:scaled-rewards}, we obtain
\begin{equation*}
\lim_{\gamma\to1^-}\frac{\Delta^{\pi_{t+4}}(s)}{1-\gamma^L}
\ge(24d+4)\mathsf g_i-56\mathsf f_{d-1}>0,
\qquad \forall\,s\in\{\bar s_i,s_i^{\mathrm c},s_i^{\mathrm z}\}.
\end{equation*}
We conclude that \(\pi_{t+5}(s)=a_{[\mathbf b^+]_i}\)
for \(s\in\{\bar s_i,s_i^{\mathrm c},s_i^{\mathrm z}\}
\cup\{\bar s_{i,j}\}_{j=1}^{3}\), \(0\le i<d\).

\vspace{5pt}\noindent\refstepcounter{updatestep}%
\textbf{\boldmath Step~\theupdatestep. Determine the actions at \(s_i,s_d,\{s_{i,j}\}_{j=1}^{5},\bar s_{i,4}\) and in \(\mathcal S_3\).}\label{step:recover-policy}
At \(\bar s_{i,4}\), action \(a_0\) has value
\(\mathsf R_i(\gamma)+\gamma^{2L}\widetilde V^{\pi_{t+4}}(\bar s_{i+1})\)
by~\eqref{eq:reward-bar-s14-a0}.
Combining~\eqref{eq:copy-first-bit-comparison} and~\eqref{eq:copy-return}, we obtain
\begin{equation*}
\Delta^{\pi_{t+4}}(\bar s_{i,4})
=\gamma^{-4L}\Delta^{\pi_{t+4}}(\bar s_i)-\mathsf R_i(\gamma)
\le(\gamma^{5L}-1)\mathsf R_i(\gamma)<0,
\qquad \pi_{t+5}(\bar s_{i,4})=a_0.
\end{equation*}
We next determine the actions in \(\mathcal S_3\).
By~\eqref{eq:copy-first-bit-comparison}, we have
\(\widetilde Q^{\pi_{t+4}}(s_i,a_1)=\widetilde V^{\pi_{t+4}}(s_i)\)
and \(\Delta^{\pi_{t+4}}(s_i)\to0\).
Using~\eqref{eq:reward-sr0-a1}--\eqref{eq:reward-sr12-a1}
along the selected \(a_1\) paths in \(\mathcal S_3\), together with
\eqref{eq:encoding-value-limit} and the action values at
\(s_i^{\mathrm c},s_i^{\mathrm z}\) in~\eqref{eq:copy-auxiliary-action-values}, we obtain
\begin{equation}
\lim_{\gamma\to1^-}\bigl(\widetilde V^{\pi_{t+4}}(s)-\widetilde V^{\pi_{t+4}}(s_i)\bigr)
=-8\mathsf f_{d-1}-\mathsf f_i,
\qquad \forall\,s\in\{s_i^{\mathrm c},s_i^{\mathrm z},s_{i,0}^{\mathrm r},s_{i,1}^{\mathrm r},s_{i,2}^{\mathrm r}\}.
\label{eq:copy-state-offset}
\end{equation}
For \(0\le i<d\), Table~\ref{tab:paths-s1} gives the endpoints
\(s_{i-1}^{\mathrm z}\) for the \(a_0\) paths from
\(s_{i,0}^{\mathrm r},s_{i,2}^{\mathrm r}\), and \(s_{i-1}^{\mathrm c}\)
for the \(a_0\) path from \(s_{i,1}^{\mathrm r}\).
By~\eqref{eq:reward-sr0-a0}, we have
\(\widetilde R^{\pi_{t+4}}((s_{i,0}^{\mathrm r},a_0)\leadsto s_{i-1}^{\mathrm z})
=\mathsf f_i(\gamma)\).
The \(a_0\) paths from \(s_{i,1}^{\mathrm r},s_{i,2}^{\mathrm r}\) have zero rewards.
For \(0<i<d\), we subtract the \(a_0\) value from the current \(a_1\) value
and use~\eqref{eq:encoding-value-limit} and~\eqref{eq:copy-state-offset} to obtain
\begin{equation*}
\lim_{\gamma\to1^-}\Delta^{\pi_{t+4}}(s_{i,j}^{\mathrm r})
\le(-8\mathsf f_{d-1}-\mathsf f_i)-(-8\mathsf f_{d-1}-\mathsf f_{i-1})
=\mathsf f_{i-1}-\mathsf f_i\le-1,
\qquad \forall\,0\le j\le2.
\end{equation*}
For \(i=0\), we obtain from~\eqref{eq:rewards-s2} and~\eqref{proof:20}
that \(\widetilde V^{\pi_{t+4}}(s_{-1}^{\mathrm c})
-\widetilde V^{\pi_{t+4}}(s_0)\to-4\mathsf f_{d-1}\) and
\(\widetilde V^{\pi_{t+4}}(s_{-1}^{\mathrm z})
-\widetilde V^{\pi_{t+4}}(s_0)\to-8\mathsf f_{d-1}\).
Using~\eqref{eq:copy-state-offset}, we therefore obtain
\begin{equation*}
\lim_{\gamma\to1^-}\Delta^{\pi_{t+4}}(s_{0,j}^{\mathrm r})
\le(-8\mathsf f_{d-1}-\mathsf f_0)-(-8\mathsf f_{d-1})
=-\mathsf f_0=-1,\qquad \forall\,0\le j\le2.
\end{equation*}
We conclude that \(\pi_{t+5}(s_{i,j}^{\mathrm r})=a_0\)
for \(0\le i<d\) and \(0\le j\le2\).
By Step~\ref*{step:fourth-first-actions} in Appendix~\ref{sec:update-four},
we have \(\pi_{t+4}(s_{i,4})=\pi_{t+4}(s_{i,5})=a_0\).
We use~\eqref{eq:si5-action-comparison},
\eqref{eq:si4-action-comparison}, and~\eqref{eq:copy-state-offset}
under \(\pi_{t+4}\) to obtain
\begin{equation*}
\lim_{\gamma\to1^-}\Delta^{\pi_{t+4}}(s_{i,4})
=-12\mathsf f_{d-1}-\tfrac12<0,\qquad
\lim_{\gamma\to1^-}\Delta^{\pi_{t+4}}(s_{i,5})
=-8\mathsf f_{d-1}-2\mathsf f_i-\tfrac14<0.
\end{equation*}
We conclude that \(s_{i,4},s_{i,5}\) keep \(a_0\).
At \(s_d\), we use~\eqref{eq:reward-sd-a0} and~\eqref{proof:20} to obtain
\begin{equation*}
\lim_{\gamma\to1^-}\Delta^{\pi_{t+4}}(s_d)
=8\mathsf f_{d-1}-4\mathsf f_{d-1}=4\mathsf f_{d-1}>0,
\qquad \pi_{t+5}(s_d)=a_1.
\end{equation*}
For \([\mathbf b^+]_i=0\),
Step~\ref*{step:fourth-first-actions} in Appendix~\ref{sec:update-four}
gives \(\pi_{t+4}(s_{i,1})=a_0\) and
\(\pi_{t+4}(s_{i,2})=\pi_{t+4}(s_{i,3})=a_1\).
Using~\eqref{eq:copy-first-bit-comparison} and the path comparison
at \(\bar s_{i,1}\) in Step~\ref*{step:third-delayed-bits}
in Appendix~\ref{sec:update-three}, we obtain
\begin{equation*}
\Delta^{\pi_{t+4}}(s_i)=0,\qquad
\Delta^{\pi_{t+4}}(s_{i,1})=\gamma^{3L}\mathsf R_i(\gamma)>0.
\end{equation*}
We conclude that \(s_i\) keeps \(a_0\) and \(s_{i,1}\) switches to \(a_1\).
The same comparison keeps \(a_1\) at \(s_{i,2},s_{i,3}\).
For \([\mathbf b^+]_i=1\), \(s_i\) and \(\{s_{i,j}\}_{j=1}^{3}\) keep \(a_1\).
All choices now agree with~\eqref{eq:encoding-policy}, so
\(\pi_{t+5}=\pi[\mathbf b^+]\).

\subsection{\texorpdfstring{Completing the proof of Lemma~\ref{prop:counter}}{Completing the proof of Lemma \getrefnumber{prop:counter}}}
\label{sec:complete-counter-proof}
We verify that the action comparisons in
Appendices~\ref{sec:update-one}--\ref{sec:update-five} hold at the discount
in~\eqref{proof:32}.
Fix one of the five policies \(\pi\in\{\pi_t,\ldots,\pi_{t+4}\}\)
specified in Table~\ref{tab:auxiliary-updates}, and use the unscaled
action values \(Q^\pi\) for the unshifted rewards.
By Tables~\ref{tab:paths-s1} and~\ref{tab:auxiliary-updates},
every trajectory under \(\pi\) eventually reaches a cycle of length
\((6d+1)L\).
Thus, for any state \(s\) with two actions, we have
\begin{equation}
Q^\pi(s,a_1)-Q^\pi(s,a_0)=\frac{H(\gamma)}{1-\gamma^{(6d+1)L}},\qquad
H\in\mathbb Z[x].
\label{proof:35}
\end{equation}
The transient path from either successor state and its eventual cycle
contain at most \(N\) distinct states, so \(\deg H\le N\).
Canceling the periodic terms in each action value leaves coefficients
that are rewards or differences of two rewards, so every coefficient
of \(H\) has magnitude at most \(4r_{\max}\).
For \(H\ne0\), we can write
\begin{equation*}
H(\gamma)=(1-\gamma)^{\ell_0}\Bigl(c_0+\sum_{\ell\ge1}c_\ell(1-\gamma)^\ell\Bigr),
\qquad
\sum_{\ell\ge0}|c_\ell|\le4r_{\max}\sum_{j=0}^N2^j<2^{N+3}r_{\max},
\end{equation*}
where \(\ell_0\) is a nonnegative integer and \(c_\ell\in\mathbb Z\)
with \(c_0\ne0\).
The coefficient bound follows by expanding each
\(\gamma^j=(1-(1-\gamma))^j\), whose coefficients have absolute values
summing to \(2^j\).
The nonzero integer \(c_0\) satisfies \(|c_0|\ge1\).
At the discount in~\eqref{proof:32}, we have
\begin{equation}
\Bigl|\sum_{\ell\ge1}c_\ell(1-\gamma)^\ell\Bigr|
<(1-\gamma)2^{N+3}r_{\max}=2^{3-N}<1\le|c_0|.
\label{eq:sign-preservation}
\end{equation}
Thus \(H(\gamma)\) has the sign of \(c_0\) both near one and at the chosen discount.
Since the denominator in~\eqref{proof:35} is positive, all strict action
comparisons are preserved, while \(H=0\) gives a tie for every discount.
The bounds depend only on \(N\) and \(r_{\max}\), so the same discount
applies to every binary increment.

Thus \(\pi_t=\pi[\mathbf b]\) implies
\(\pi_{t+5}=\pi[\mathbf b^+]\) for the specified discount.
Induction from \(\pi_0=\pi[\mathbf0]\) proves the lemma and gives
at least \(5(2^d-2)\ge2^d\) iterations.
\qed

For \(\mathbf b\ne\mathbf1\) and \(s\in\{s_0,\bar s_0\}\),
\eqref{eq:encoded-cycle-value} expresses
\(V^{\pi[\mathbf b^+]}(s)-V^{\pi[\mathbf b]}(s)\)
with denominator \(1-\gamma^{(6d+1)L}\) and an integer numerator
of degree below \(N\), with coefficients of magnitude at most \(2r_{\max}\).
Applying~\eqref{eq:sign-preservation} to this numerator shows that
the inequalities in Lemma~\ref{lem:encoding-values}
also hold at the discount in~\eqref{proof:32}.
The discount requires \(O(N+\log r_{\max})\) bits.

\section{\texorpdfstring{Proof of Lemma~\ref{lem:av-counting}}{Proof of Lemma \getrefnumber{lem:av-counting}}}
\label{app:average-counting}

\looseness=-1 We proceed by induction on \(n\), with \(t=3n\) and
\(0\le n\le2^{d+1}-3\).
We set \(C=C_n\), \(C^+=C_{n+1}\), \(\ell=\ell(n)\),
and \(\ell^+=\ell(n+1)\), so \(C\subseteq\mathcal S_\ell\)
and \(C^+\subseteq\mathcal S_{\ell^+}\).
Let \(\mathbf b\) be the binary representation of
\(\lfloor n/2\rfloor\) and \(\mu\) the mean reward of \(C\).
Our induction hypothesis is that \(\pi_t\) has \(C\) as its
unique cycle and satisfies
\begin{equation}
\left\{
\begin{aligned}
\pi_t(s)&=a_0,
&&s\in\{s_{i,j},s'_{i,j},\bar s_{k,j},\bar s'_{k,j}\}
\cap(\mathcal S_\ell\setminus C),\quad j=0,1,\\
\pi_t(s)&\in\smash[b]{\operatorname*{arg\,max}_{a\in\mathcal A_s\cap\{a_2,a_3\}}}Q^{\pi_t}(s,a),
&&s\in\mathcal S_{\ell^+}.
\end{aligned}
\right.
\label{eq:av-induction}
\end{equation}
The initial policy \(\pi_0\) in~\eqref{eq:av-initial-policy}
has \(C_0\) as its unique cycle.
Since \(a_2\) is the only action leading from \(\mathcal S_2\)
to \(\mathcal S_1\), \(\pi_0(s)=a_2\) maximizes
\(Q^{\pi_0}(s,a)\) over \(a\in\mathcal A_s\cap\{a_2,a_3\}\)
for every \(s\in\mathcal S_2\).
Thus \(\pi_0\) satisfies~\eqref{eq:av-induction} for \(n=0\).

\vspace{5pt}\noindent\textbf{Overview.} The next three iterations proceed as follows.
After the first iteration, \(\pi_{t+1}\) selects all but one edge of \(C^+\)
and still has \(C\) as its unique cycle.
After the second iteration, \(\pi_{t+2}\) selects the remaining edge of \(C^+\)
and has exactly the two cycles \(C,C^+\).
For \(n<2^{d+1}-3\), the policy \(\pi_{t+3}\) obtained after the third iteration
has \(C^+\) as its unique cycle and satisfies the induction hypothesis for \(n+1\).

\vspace{5pt}\noindent\textbf{Iteration 1: select the next paths.}
Since \(\pi_t\) has \(C\) as its unique cycle, every state reaches \(C\),
so \(J^{\pi_t}(s)=\mu\) for all \(s\in\mathcal S\).
By~\eqref{eq:av-update}, policy iteration therefore selects an action
maximizing \(Q^{\pi_t}(s,a)\) at each state.

\vspace{5pt}\noindent\textit{Step 1. Determine the actions in \(\mathcal S_\ell\).}
For \(s\in C\), we take \(a'=\pi_t(s)\), whose successor lies on \(C\)
and whose reward is at least \(-16\).
For \(s\in\mathcal S_\ell\setminus C\), we take \(a'=a_0\), which has
reward \(-24\) and leads directly to the corresponding \(s_{i,5}\) or
\(\bar s_{k,5}\) visited by \(C\) regardless of the bit at that position.
In both cases, \(f(s,a')\in C\) and \(r(s,a')\ge-24\).
For \(a\in\mathcal A_s\cap\{a_2,a_3\}\), the path under \(\pi_t\)
from \(f(s,a)\) to its first visit to \(f(s,a')\) has no repeated states.
By~\eqref{eq:av-reward-j4-s1}, \eqref{eq:av-reward-j4-s2-positive},
and~\eqref{eq:av-reward-j4-s2-bit}, its positive rewards sum to less
than \(16d+4\).
Since \(\mu\ge0\), summing the first identity in~\eqref{eq:av-bias}
along this path gives
\(V^{\pi_t}(f(s,a))-V^{\pi_t}(f(s,a'))<16d+4\).
By~\eqref{eq:av-rewards-s1} and~\eqref{eq:av-exit-rewards}, we have
\(r(s,a)\le-64d\).
The second identity in~\eqref{eq:av-bias} and \(r(s,a')\ge-24\)
therefore give
\begin{equation}
Q^{\pi_t}(s,a)-Q^{\pi_t}(s,a')
<-64d+24+(16d+4)=-48d+28<0.
\label{eq:av-toll-bias}
\end{equation}
Hence, we have
\[
\pi_{t+1}(s)\in\{a_0,a_1\},\qquad\forall\,s\in\mathcal S_\ell.
\]
\begin{itemize}[leftmargin=*,topsep=3pt,itemsep=2pt,parsep=0pt]
\item \emph{States \(s\in\{s_{i,j},s'_{i,j},\bar s_{k,j},\bar s'_{k,j}\}\cap C\)
with \(0\le j\le4\).}
The selected \(a_1\) path reaches \(f(s,a_0)\)
in \(5-j\) transitions. Its total reward is at least \(-16\)
by~\eqref{eq:av-branch-rewards}.
Since \(\pi_t(s)=a_1\), summing the first identity in~\eqref{eq:av-bias}
along this path and using the second identity and~\eqref{eq:av-private-return},
we obtain
\begin{equation}
Q^{\pi_t}(s,a_0)-Q^{\pi_t}(s,a_1)
=Q^{\pi_t}(s,a_0)-V^{\pi_t}(s)
\le-8+(4-j)\mu<0.
\label{eq:av-selected-return}
\end{equation}
Hence, we have \(\pi_{t+1}(s)=\pi_t(s)=a_1\).
\item \emph{States \(s\in\{s_{i,5},\bar s_{k,5}\}\cap C\).}
For \(a\in\{a_0,a_1\}\setminus\{\pi_t(s)\}\), the successors
\(f(s,\pi_t(s))\in C\) and \(f(s,a)\notin C\)
are the two \(j=0\) states at the next position.
Their \(a_0\) actions have the same reward and successor,
and~\eqref{eq:av-induction} gives \(\pi_t(f(s,a))=a_0\).
Thus \(V^{\pi_t}(f(s,a))=Q^{\pi_t}(f(s,\pi_t(s)),a_0)\).
Since \(r(s,\pi_t(s))=r(s,a)=0\), applying~\eqref{eq:av-selected-return}
at \(f(s,\pi_t(s))\) and the second identity in~\eqref{eq:av-bias} gives
\begin{equation}
Q^{\pi_t}(s,\pi_t(s))-Q^{\pi_t}(s,a)
=V^{\pi_t}(f(s,\pi_t(s)))-V^{\pi_t}(f(s,a))
\ge8-4\mu>7.
\label{eq:av-paired-bias-gap}
\end{equation}
Hence, we have \(\pi_{t+1}(s)=\pi_t(s)\).
\end{itemize}
Combining the two cases, we have
\begin{equation}
\pi_{t+1}(s)=\pi_t(s),\qquad\forall\,s\in C.
\label{eq:av-first-retained-cycle}
\end{equation}
For \(s\in\{s_i,s'_i,\bar s_k,\bar s'_k\}\cap(\mathcal S_\ell\setminus C)\),
the induction condition~\eqref{eq:av-induction} gives
\(\pi_t(s)=\pi_t(f(s,a_1))=a_0\).
Since the \(a_0\) actions at \(s\) and \(f(s,a_1)\) have the same
reward and successor and \(r(s,a_1)=0\), we obtain
\(Q^{\pi_t}(s,a_1)-Q^{\pi_t}(s,a_0)=-\mu\le0\)
by~\eqref{eq:av-bias}, so we have
\(\pi_{t+1}(s)=\pi_t(s)=a_0\).
The cycle \(C\) is unchanged by~\eqref{eq:av-first-retained-cycle},
and every state in \(\mathcal S_\ell\) reaches \(C\) under \(\pi_{t+1}\).
Each available \(a_2,a_3\) action in \(\mathcal S_{\ell^+}\)
leads either to \(C\) or to a \(j=0\) state in
\(\mathcal S_\ell\setminus C\) that enters \(C\) directly
through \(a_0\) under both policies, so we have
\begin{equation}
V^{\pi_{t+1}}(f(s,a))=V^{\pi_t}(f(s,a)),\qquad
\forall\,s\in\mathcal S_{\ell^+},\quad
a\in\mathcal A_s\cap\{a_2,a_3\}.
\label{eq:av-preserved-successor-bias}
\end{equation}

\vspace{5pt}\noindent\textit{Step 2. Determine the actions at \(\{s_{i,j},s'_{i,j},\bar s_{k,j},\bar s'_{k,j}\}_{j=0}^{4}\) in \(\mathcal S_{\ell^+}\).}
At each of these states \(s\), \eqref{eq:av-induction} gives
\(\pi_t(s)\in\mathcal A_s\cap\{a_2,a_3\}\) and
\(Q^{\pi_t}(s,a)\le Q^{\pi_t}(s,\pi_t(s))=V^{\pi_t}(s)\) for
\(a\in\mathcal A_s\cap\{a_2,a_3\}\).
When \(C^+\subseteq\mathcal S_1\), \eqref{eq:av-path-exits}
and~\eqref{eq:av-reward-s1-a23} give
\(Q^{\pi_t}(s_{i,j},a)=Q^{\pi_t}(s_i,a)+j\) and
\(Q^{\pi_t}(s'_{i,j},a)=Q^{\pi_t}(s'_i,a)+j\) for
\(a\in\{a_2,a_3\}\).
Taking the maximum over these two actions in~\eqref{eq:av-induction}, we obtain
\begin{equation}
V^{\pi_t}(s_{i,j})=V^{\pi_t}(s_i)+j,\qquad
V^{\pi_t}(s'_{i,j})=V^{\pi_t}(s'_i)+j,\qquad 0\le j\le4.
\label{eq:av-depth-profile}
\end{equation}
For \(C^+\subseteq\mathcal S_2\), \eqref{eq:av-induction} gives
\(\pi_t(s)=a_2\) throughout \(\mathcal S_2\).
Using~\eqref{eq:av-exit-rewards}, we obtain, for
\(k\in\{i,d+i,2d+i\}\) and \(0\le j\le4\),
\begin{equation}
\left\{
\begin{aligned}
V^{\pi_t}(\bar s_{k,j})&=64(k-4d)+j-\mu+V^{\pi_t}(s_i),\\
V^{\pi_t}(\bar s'_{k,j})&=64(k-4d)+j-\mu+V^{\pi_t}(s'_i).
\end{aligned}
\right.
\label{eq:av-s2-bias-profile}
\end{equation}
These identities also give~\eqref{eq:av-depth-profile} with
\(s_{i,j},s'_{i,j}\) replaced by \(\bar s_{k,j},\bar s'_{k,j}\), respectively.
\begin{itemize}[leftmargin=*,topsep=3pt,itemsep=2pt,parsep=0pt]
\item \emph{\(n\) even (\(C^+\subseteq\mathcal S_2\)).}
Taking the difference in~\eqref{eq:av-s2-bias-profile} at \(j=0\), we obtain
\begin{equation}
V^{\pi_t}(\bar s_k)-V^{\pi_t}(\bar s'_k)
=V^{\pi_t}(s_i)-V^{\pi_t}(s'_i).
\label{eq:av-copy-bias-gap}
\end{equation}
By~\eqref{eq:av-paired-bias-gap}, we have
\(V^{\pi_t}(\bar s_k)-V^{\pi_t}(\bar s'_k)>7\) when \([\mathbf b]_i=0\), and
\(V^{\pi_t}(\bar s'_k)-V^{\pi_t}(\bar s_k)>7\) when \([\mathbf b]_i=1\).
\item \emph{\(n\) odd (\(C^+\subseteq\mathcal S_1\)).}
For \([\mathbf b]_i=0\), summing the first identity in~\eqref{eq:av-bias}
over the \(6d\) transitions on \(C\) from \(\bar s_i\) to \(\bar s_{d+i}\)
and using~\eqref{eq:av-path-exits}, \eqref{eq:av-reward-j4-s2-negative},
and~\eqref{eq:av-reward-s1-a23}, we obtain
\begin{equation}
Q^{\pi_t}(s'_i,a_3)-Q^{\pi_t}(s_i,a_2)
=8+V^{\pi_t}(\bar s_i)-V^{\pi_t}(\bar s_{d+i})
=8-16\sum_{p=0}^{i-1}(1-[\mathbf b]_p)-6d\mu.
\label{eq:av-carry}
\end{equation}
By~\eqref{eq:av-cycle-means}, we have \(0<6d\mu<1\).
Thus the last expression in~\eqref{eq:av-carry} is greater than \(7\)
when \([\mathbf b]_p=1\) for every \(0\le p<i\), and less than \(-8\)
when \([\mathbf b]_p=0\) for some \(0\le p<i\).
Applying~\eqref{eq:av-paired-bias-gap} at \(\bar s_{i-1,5}\) and
\(\bar s_{d+i-1,5}\), where \(\pi_t\) chooses \(a_0\), and
using~\eqref{eq:av-path-exits} and~\eqref{eq:av-reward-s1-a23}, we obtain
\begin{subequations}
\begin{empheq}[left=\empheqlbrace]{align}
Q^{\pi_t}(s'_i,a_2)-Q^{\pi_t}(s_i,a_2)
&=V^{\pi_t}(\bar s'_{d+i})-V^{\pi_t}(\bar s_{d+i})<-7,
\label{eq:av-carry-a2-gap}\\
Q^{\pi_t}(s_i,a_3)-Q^{\pi_t}(s'_i,a_3)
&=V^{\pi_t}(\bar s'_i)-V^{\pi_t}(\bar s_i)<-7.
\label{eq:av-carry-a3-gap}
\end{empheq}
\end{subequations}
Together with~\eqref{eq:av-cycle-means} and~\eqref{eq:av-carry},
\eqref{eq:av-carry-a2-gap} gives
\(Q^{\pi_t}(s_i,a_2)>\max\{Q^{\pi_t}(s'_i,a_2),Q^{\pi_t}(s'_i,a_3)\}+7\)
when \([\mathbf b^+]_i=[\mathbf b]_i\), while~\eqref{eq:av-carry-a3-gap} gives
\(Q^{\pi_t}(s'_i,a_3)>\max\{Q^{\pi_t}(s_i,a_2),Q^{\pi_t}(s_i,a_3)\}+7\)
when \([\mathbf b^+]_i\ne[\mathbf b]_i\).
For \([\mathbf b]_i=1\), repeating the argument along \(C\) from
\(\bar s'_i\) to \(\bar s'_{d+i}\) gives the two inequalities involving
\(\max\) with \(s_i,s'_i\) interchanged.
By~\eqref{eq:av-induction}, we have
\(V^{\pi_t}(s)=\max\{Q^{\pi_t}(s,a_2),Q^{\pi_t}(s,a_3)\}\) for
\(s\in\{s_i,s'_i\}\), so we obtain
\begin{equation}
V^{\pi_t}(s_i)-V^{\pi_t}(s'_i)>7\quad([\mathbf b^+]_i=0),\qquad
V^{\pi_t}(s'_i)-V^{\pi_t}(s_i)>7\quad([\mathbf b^+]_i=1).
\label{eq:av-carry-bias-gap}
\end{equation}
\end{itemize}
For \(s\in\{s_{i,5},\bar s_{k,5}\}\cap C\), we have
\(r(s,a_0)=r(s,a_1)=0\), so
\(V^{\pi_t}(s)=\max\{V^{\pi_t}(f(s,a_0)),V^{\pi_t}(f(s,a_1))\}-\mu\)
by~\eqref{eq:av-paired-bias-gap}.
When \(C^+\subseteq\mathcal S_1\), applying this identity at the successors
in~\eqref{eq:av-shared-exits} and using~\eqref{eq:av-reward-s1-a23}
and~\eqref{eq:av-reward-s1-shared}, we obtain
\begin{equation*}
Q^{\pi_t}(s_{i,5},a)
=\max\{Q^{\pi_t}(s_i,a),Q^{\pi_t}(s'_i,a)\}+24-\mu,
\qquad a\in\{a_2,a_3\}.
\end{equation*}
When \(C^+\subseteq\mathcal S_2\), for \(k\in\{i,d+i,2d+i\}\),
the transition \(f(\bar s_{k,5},a_2)=s_{i-1,5}\)
and~\eqref{eq:av-exit-rewards}
similarly give
\begin{equation*}
Q^{\pi_t}(\bar s_{k,5},a_2)
=\max\{Q^{\pi_t}(\bar s_k,a_2),Q^{\pi_t}(\bar s'_k,a_2)\}+24-\mu.
\end{equation*}
Taking maxima in~\eqref{eq:av-induction}, we obtain
\begin{subequations}
\label{eq:av-shared-bias}
\begin{empheq}[left=\empheqlbrace]{align}
V^{\pi_t}(s_{i,5})-\max\{V^{\pi_t}(s_i),V^{\pi_t}(s'_i)\}
&=24-\mu,\quad C^+\subseteq\mathcal S_1,
\label{eq:av-shared-bias-s1}\\
V^{\pi_t}(\bar s_{k,5})-\max\{V^{\pi_t}(\bar s_k),V^{\pi_t}(\bar s'_k)\}
&=24-\mu,\quad C^+\subseteq\mathcal S_2.
\label{eq:av-shared-bias-s2}
\end{empheq}
\end{subequations}
\begin{samepage}%
For \(C^+\subseteq\mathcal S_1\) and \(0\le j<4\),
\eqref{eq:av-depth-profile} gives
\(Q^{\pi_t}(s_{i,j},a_1)=V^{\pi_t}(s_i)+j+1-\mu\).
Combining this with~\eqref{eq:av-private-return} and~\eqref{eq:av-shared-bias-s1},
we obtain
\begin{equation}
Q^{\pi_t}(s_{i,j},a_0)-Q^{\pi_t}(s_{i,j},a_1)
=\max\{V^{\pi_t}(s_i),V^{\pi_t}(s'_i)\}
-V^{\pi_t}(s_i)-(j+1)-\mu.
\label{eq:av-path-action-difference}
\end{equation}
By~\eqref{eq:av-private-return} and~\eqref{eq:av-shared-bias-s1},
the identity also holds for \(s'_{i,j}\), with \(s_i,s'_i\) interchanged.
When \(C^+\subseteq\mathcal S_2\), \eqref{eq:av-private-return}
and~\eqref{eq:av-shared-bias-s2} give the corresponding identities
for \(\bar s_{k,j},\bar s'_{k,j}\), with \(s_i,s'_i\) replaced by
\(\bar s_k,\bar s'_k\), respectively.
For each \(0\le j\le4\), we consider
\(s\in\{s_{i,j},s'_{i,j},\bar s_{k,j},\bar s'_{k,j}\}
\cap\mathcal S_{\ell^+}\).
\begin{itemize}[leftmargin=*,topsep=3pt,itemsep=2pt,parsep=0pt]
\item \emph{\(0\le j<4\).}
By~\eqref{eq:av-induction} and~\eqref{eq:av-depth-profile}, we have
\(Q^{\pi_t}(s,a_1)=V^{\pi_t}(s)+1-\mu
>V^{\pi_t}(s)\ge Q^{\pi_t}(s,a)\)
for \(a\in\mathcal A_s\cap\{a_2,a_3\}\).
Combining~\eqref{eq:av-copy-bias-gap}
and~\eqref{eq:av-carry-bias-gap}
with~\eqref{eq:av-path-action-difference}, we obtain
\(Q^{\pi_t}(s,a_1)-Q^{\pi_t}(s,a_0)=j+1+\mu>0\)
for \(s\in C^+\), and
\(Q^{\pi_t}(s,a_0)-Q^{\pi_t}(s,a_1)
>7-(j+1)-\mu\ge3-\mu>0\)
for \(s\notin C^+\).
\item \emph{\(j=4\).}
Since \(r(s,a_1)\ge-16\), \eqref{eq:av-induction},
\eqref{eq:av-depth-profile}, and~\eqref{eq:av-shared-bias} give
\(Q^{\pi_t}(s,a_1)\ge V^{\pi_t}(s)+4-2\mu
>V^{\pi_t}(s)\ge Q^{\pi_t}(s,a)\)
for \(a\in\mathcal A_s\cap\{a_2,a_3\}\).
Since \(f(s,a_0)=f(s,a_1)\), \eqref{eq:av-private-return} gives
\(Q^{\pi_t}(s,a_1)-Q^{\pi_t}(s,a_0)=24+r(s,a_1)\ge8>0\).
\end{itemize}
Combining these comparisons with~\eqref{eq:av-induction}, we obtain
\begin{equation}
\pi_{t+1}(s)=
\begin{cases}
a_1,&s\in C^+\text{ or }j=4,\\
a_0,&s\notin C^+,\quad 0\le j<4.
\end{cases}
\label{eq:av-first-path-actions}
\end{equation}
\end{samepage}

\vspace{5pt}\noindent\textit{Step 3. Determine the actions at \(s_{i,5}\) and \(\bar s_{k,5}\) in \(\mathcal S_{\ell^+}\).}
For \((s,a)\in C^+\) with \(s\in\{s_{i,5},\bar s_{k,5}\}\),
we have \(r(s,a_0)=r(s,a_1)=0\).
For \(a'\in\{a_0,a_1\}\setminus\{a\}\),
\eqref{eq:av-copy-bias-gap} and~\eqref{eq:av-carry-bias-gap} therefore give
\begin{equation}
Q^{\pi_t}(s,a)-Q^{\pi_t}(s,a')
=V^{\pi_t}(f(s,a))-V^{\pi_t}(f(s,a'))>0.
\label{eq:av-shared-action-comparison}
\end{equation}

\begin{itemize}[leftmargin=*,topsep=3pt,itemsep=2pt,parsep=0pt]
\item \emph{\(n\) even (\(C^+\subseteq\mathcal S_2\)).}
By~\eqref{eq:av-paired-bias-gap},
\(\max\{V^{\pi_t}(s_i),V^{\pi_t}(s'_i)\}\) equals
\(V^{\pi_t}(s_i)\) when \([\mathbf b]_i=0\) and
\(V^{\pi_t}(s'_i)\) when \([\mathbf b]_i=1\).
For \(k\in\{i,d+i,2d+i\}\) with \(k<3d-1\), we have
\begin{equation}
\begin{aligned}
&\max\{V^{\pi_t}(\bar s_{k+1}),V^{\pi_t}(\bar s'_{k+1})\}
-\max\{V^{\pi_t}(\bar s_k),V^{\pi_t}(\bar s'_k)\}
\\
&\qquad=64+\max\{V^{\pi_t}(s_{i+1}),V^{\pi_t}(s'_{i+1})\}
-\max\{V^{\pi_t}(s_i),V^{\pi_t}(s'_i)\}
=64+6\mu-2^{i-d}[\mathbf b]_i>24,
\end{aligned}
\label{eq:av-even-bias-increment}
\end{equation}
where the first equality follows from~\eqref{eq:av-s2-bias-profile},
and the second follows by summing the first identity
in~\eqref{eq:av-bias} over six transitions on \(C\) from
position \(i\) to \(i+1\), using~\eqref{eq:av-reward-j4-s1}
and~\eqref{eq:av-reward-j4-zero}.
Here \(i+1\) is taken modulo \(d\).
\item \emph{\(n\) odd (\(C^+\subseteq\mathcal S_1\)).}
By~\eqref{eq:av-paired-bias-gap}, for \(k\in\{i,d+i,2d+i\}\),
\(\max\{V^{\pi_t}(\bar s_k),V^{\pi_t}(\bar s'_k)\}\) equals
\(V^{\pi_t}(\bar s_k)\) when \([\mathbf b]_i=0\) and
\(V^{\pi_t}(\bar s'_k)\) when \([\mathbf b]_i=1\).
For \(0\le i<d-1\) and \(a'\in\{a_2,a_3\}\), summing the first
identity in~\eqref{eq:av-bias} over six transitions on \(C\)
and using~\eqref{eq:av-path-exits} and~\eqref{eq:av-reward-s1-a23},
we obtain
\begin{equation*}
\begin{aligned}
&\max\{Q^{\pi_t}(s_{i+1},a'),Q^{\pi_t}(s'_{i+1},a')\}
-\max\{Q^{\pi_t}(s_i,a'),Q^{\pi_t}(s'_i,a')\}
\\
&\qquad=64+6\mu+
\begin{cases}
16(1-[\mathbf b]_i),&a'=a_2,\\
0,&a'=a_3,
\end{cases}
\ge64+6\mu,
\end{aligned}
\end{equation*}
where (i) for \(a'=a_2\), the rewards from position \(d+i\) to
\(d+i+1\) sum to \(-16(1-[\mathbf b]_i)\) by
\eqref{eq:av-reward-j4-s2-negative} and~\eqref{eq:av-reward-j4-zero},
and (ii) for \(a'=a_3\), the rewards from position \(i\) to \(i+1\)
sum to \(0\) by~\eqref{eq:av-reward-j4-zero}.
Taking the maximum over \(a_2,a_3\) in~\eqref{eq:av-induction}, we obtain
\begin{equation}
\max\{V^{\pi_t}(s_{i+1}),V^{\pi_t}(s'_{i+1})\}
-\max\{V^{\pi_t}(s_i),V^{\pi_t}(s'_i)\}\ge64+6\mu>24,
\qquad 0\le i<d-1.
\label{eq:av-odd-bias-increment}
\end{equation}
\end{itemize}
\begin{samepage}%
For \((s,a)\in C^+\) with
\(s\in\{s_{i,5},\bar s_{k,5}\}
\setminus\{s_{d-1,5},\bar s_{3d-1,5}\}\),
combining~\eqref{eq:av-shared-bias}
with~\eqref{eq:av-even-bias-increment}
and~\eqref{eq:av-odd-bias-increment} gives
\begin{equation*}
Q^{\pi_t}(s,a)>V^{\pi_t}(s)
=\max_{a'\in\mathcal A_s\cap\{a_2,a_3\}}Q^{\pi_t}(s,a'),
\end{equation*}
where the equality follows from~\eqref{eq:av-induction}.
Combining this with~\eqref{eq:av-shared-action-comparison},
we obtain \(\pi_{t+1}(s)=a\).
Together with the updates for \(0\le j\le4\)
in~\eqref{eq:av-first-path-actions}, we obtain
\begin{equation}
\pi_{t+1}(s)=a,\qquad
\forall\,(s,a)\in C^+\text{ with }
 s\notin\{s_{d-1,5},\bar s_{3d-1,5}\}.
\label{eq:av-first-next-actions}
\end{equation}
\end{samepage}

\noindent For the unique state
\(s\in C^+\cap\{s_{d-1,5},\bar s_{3d-1,5}\}\),
the \(a_0,a_1\) successors are the two \(j=0\) states at the first
position. Combining~\eqref{eq:av-shared-bias}
with~\eqref{eq:av-even-bias-increment}
and~\eqref{eq:av-odd-bias-increment}, we obtain
\begin{equation*}
Q^{\pi_t}(s,a')-V^{\pi_t}(s)
=V^{\pi_t}(f(s,a'))-\mu-V^{\pi_t}(s)
\le-24<0,\qquad a'\in\{a_0,a_1\}.
\end{equation*}
By~\eqref{eq:av-induction}, \(\pi_t(s)\) already maximizes
\(Q^{\pi_t}(s,a')\) over
\(a'\in\mathcal A_s\cap\{a_2,a_3\}\).
We therefore have \(\pi_{t+1}(s)=\pi_t(s)\).
By~\eqref{eq:av-first-path-actions}, every state in
\(\mathcal S_{\ell^+}\setminus C^+\) reaches a state on \(C^+\),
and by~\eqref{eq:av-first-next-actions}, the actions of \(\pi_{t+1}\)
on \(C^+\setminus\{s\}\) form a path ending at \(s\).
Since \(f(s,\pi_{t+1}(s))\in C\), every state
in \(\mathcal S_{\ell^+}\) reaches \(C\)
under \(\pi_{t+1}\).
Step~1 shows that the actions on \(C\) are unchanged and every state
in \(\mathcal S_\ell\) reaches \(C\), so we conclude that \(\pi_{t+1}\) has \(C\)
as its unique cycle.

\vspace{5pt}\noindent\textbf{Iteration 2: close the new cycle.}
Under \(\pi_{t+1}\), \(C\) is still the unique cycle and all gains
remain \(\mu\), so the update rule~\eqref{eq:av-update} again reduces to
maximizing \(Q^{\pi_{t+1}}(s,a)\) at each state.
For \(s\in\mathcal S_\ell\), the argument for~\eqref{eq:av-toll-bias}
applies under \(\pi_{t+1}\), so we have
\(\pi_{t+2}(s)\in\{a_0,a_1\}\).
For \(s\in\mathcal S_{\ell^+}\), \eqref{eq:av-bias-monotonicity},
\eqref{eq:av-induction}, and~\eqref{eq:av-preserved-successor-bias} give
\begin{equation}
Q^{\pi_{t+1}}(s,a)=Q^{\pi_t}(s,a)
\le V^{\pi_t}(s)\le V^{\pi_{t+1}}(s),
\qquad \forall\,s\in\mathcal S_{\ell^+},\quad
a\in\mathcal A_s\cap\{a_2,a_3\}.
\label{eq:av-fixed-exit-values}
\end{equation}
For \(s\in\mathcal S_{\ell^+}
\setminus\{s_{d-1,5},\bar s_{3d-1,5}\}\),
\eqref{eq:av-first-path-actions} and~\eqref{eq:av-first-next-actions}
give \(\pi_{t+1}(s)\in\{a_0,a_1\}\).
By~\eqref{eq:av-fixed-exit-values} and the tie-breaking rule
in~\eqref{eq:av-update}, we therefore have
\(\pi_{t+2}(s)\in\{a_0,a_1\}\).
For \((s,a)\in C^+\) with
\(s\in\{s_{d-1,5},\bar s_{3d-1,5}\}\),
\eqref{eq:av-first-next-actions} gives
\(Q^{\pi_{t+1}}(s',a')=V^{\pi_{t+1}}(s')\)
for every \((s',a')\in C^+\) with \(s'\ne s\).
Summing the second identity in~\eqref{eq:av-bias}
around \(C^+\) and using~\eqref{eq:av-cycle-means}, we obtain
\begin{equation}
Q^{\pi_{t+1}}(s,a)-V^{\pi_{t+1}}(s)
=\sum_{(s',a')\in C^+}r(s',a')-|C^+|\mu
=\frac{|C^+|2^{-d-1}}{6d}>0.
\label{eq:av-final-edge-residual}
\end{equation}
By~\eqref{eq:av-fixed-exit-values}
and~\eqref{eq:av-final-edge-residual}, we have
\(Q^{\pi_{t+1}}(s,a)>V^{\pi_{t+1}}(s)\ge Q^{\pi_{t+1}}(s,a')\)
for every \(a'\in\mathcal A_s\cap\{a_2,a_3\}\).
We therefore obtain \(\pi_{t+2}(s)\in\{a_0,a_1\}\)
for every \(s\in\mathcal S\).
By Step~1 of Iteration~1 and~\eqref{eq:av-first-path-actions},
\(\pi_{t+1}\) selects \(a_1\) along the paths of \(C\) and \(C^+\),
and \(a_0\) at the \(j=0\) states outside \(C\cup C^+\).
Applying the calculations for~\eqref{eq:av-selected-return}
and~\eqref{eq:av-paired-bias-gap} under \(\pi_{t+1}\), we obtain
\begin{equation}
Q^{\pi_{t+1}}(s,a)>Q^{\pi_{t+1}}(s,a'),\qquad
\forall\,(s,a)\in C\cup C^+,\quad
a'\in\{a_0,a_1\}\setminus\{a\}.
\label{eq:av-second-internal-comparison}
\end{equation}
Combining~\eqref{eq:av-second-internal-comparison} with
\(\pi_{t+2}(s)\in\{a_0,a_1\}\), we obtain
\[
\pi_{t+2}(s)=a,\qquad \forall\,(s,a)\in C\cup C^+.
\]
For \(s\in\{s_{i,j},s'_{i,j},\bar s_{k,j},\bar s'_{k,j}\}
\cap(\mathcal S_{\ell^+}\setminus C^+)\) with \(0\le j<3\),
\eqref{eq:av-first-path-actions} gives
\(\pi_{t+1}(s)=\pi_{t+1}(f(s,a_1))=a_0\).
Since the \(a_0\) actions at \(s\) and \(f(s,a_1)\) have the same successor,
\eqref{eq:av-bias} and~\eqref{eq:av-private-return} give
\(Q^{\pi_{t+1}}(s,a_1)-Q^{\pi_{t+1}}(s,a_0)=-\mu\le0\), so we obtain
\(\pi_{t+2}(s)=\pi_{t+1}(s)=a_0\).
Every state outside \(C\cup C^+\) chooses \(a_0\) or \(a_1\)
under \(\pi_{t+2}\) and therefore reaches the corresponding
\(s_{i,5}\) or \(\bar s_{k,5}\).
Thus \(\pi_{t+2}\) has exactly the two cycles \(C,C^+\).
Writing \(\mu^+=\mu+2^{-d-1}/(6d)\) for the mean reward of \(C^+\),
we obtain
\begin{equation}
J^{\pi_{t+2}}(s)=
\begin{cases}
\mu,&s\in\mathcal S_\ell,\\
\mu^+,&s\in\mathcal S_{\ell^+}.
\end{cases}
\label{eq:av-birth-gains}
\end{equation}

\vspace{5pt}\noindent\textbf{Iteration 3: make the new cycle unique.}
For \(n<2^{d+1}-3\), we verify the induction hypothesis for \(n+1\).
\begin{itemize}[leftmargin=*,topsep=3pt,itemsep=2pt,parsep=0pt]
\item \emph{States in \(\mathcal S_\ell\).}
The \(a_0,a_1\) successors have gain \(\mu\), whereas every available
\(a_2,a_3\) successor has gain \(\mu^+>\mu\).
By~\eqref{eq:av-update} and~\eqref{eq:av-birth-gains}, we obtain
\begin{equation}
\pi_{t+3}(s)\in\operatorname*{arg\,max}_{a\in\mathcal A_s\cap\{a_2,a_3\}}
Q^{\pi_{t+2}}(s,a),\qquad\forall\,s\in\mathcal S_\ell,
\label{eq:av-third-cross-actions}
\end{equation}
so every state in \(\mathcal S_\ell\) enters \(\mathcal S_{\ell^+}\).
\item \emph{States in \(\mathcal S_{\ell^+}\).}
By~\eqref{eq:av-update} and~\eqref{eq:av-birth-gains}, we have
\(\pi_{t+3}(s)\in\{a_0,a_1\}\).
By Iteration~2, \(\pi_{t+2}\) contains \(C^+\) and selects \(a_0\)
at the \(j=0\) states in \(\mathcal S_{\ell^+}\setminus C^+\).
Applying the calculations for~\eqref{eq:av-selected-return}
and~\eqref{eq:av-paired-bias-gap} under \(\pi_{t+2}\),
with gain \(\mu^+\in(0,1/(6d))\), we obtain
\(\pi_{t+3}(s)=\pi_{t+2}(s)\) for every \(s\in C^+\).
For \(s\in\{s_{i,j},s'_{i,j},\bar s_{k,j},\bar s'_{k,j}\}
\cap(\mathcal S_{\ell^+}\setminus C^+)\) with \(0\le j<2\), Iteration~2 gives
\(\pi_{t+2}(s)=\pi_{t+2}(f(s,a_1))=a_0\).
Hence \(Q^{\pi_{t+2}}(s,a_1)-Q^{\pi_{t+2}}(s,a_0)=-\mu^+<0\),
so we have \(\pi_{t+3}(s)=\pi_{t+2}(s)=a_0\).
Every state in \(\mathcal S_{\ell^+}\setminus C^+\) reaches the corresponding
\(s_{i,5}\) or \(\bar s_{k,5}\) on \(C^+\).
\end{itemize}
Consequently \(\pi_{t+3}\) has \(C^+\) as its unique cycle.
Since the actions on \(C^+\) and the \(a_0\) actions at the \(j=0\) states
in \(\mathcal S_{\ell^+}\setminus C^+\) are unchanged,
the argument for~\eqref{eq:av-preserved-successor-bias} gives
\(V^{\pi_{t+3}}(f(s,a))=V^{\pi_{t+2}}(f(s,a))\) for
\(s\in\mathcal S_\ell\) and \(a\in\mathcal A_s\cap\{a_2,a_3\}\).
Since \(J^{\pi_{t+2}}(s)=\mu\) and \(J^{\pi_{t+3}}(s)=\mu^+\)
on \(\mathcal S_\ell\), the second identity in~\eqref{eq:av-bias} gives
\begin{equation}
Q^{\pi_{t+3}}(s,a)=Q^{\pi_{t+2}}(s,a)-(\mu^+-\mu),
\qquad \forall\,s\in\mathcal S_\ell,\quad
 a\in\mathcal A_s\cap\{a_2,a_3\}.
\label{eq:av-induction-q-shift}
\end{equation}
The common shift in~\eqref{eq:av-induction-q-shift} preserves
the maximizers in~\eqref{eq:av-third-cross-actions}.
Together with the unique cycle \(C^+\) and the \(a_0\) actions
at states with \(j=0,1\) in \(\mathcal S_{\ell^+}\setminus C^+\),
this proves~\eqref{eq:av-induction} for \(n+1\).

The policy \(\pi_{t+1}\) has \(C_n\) as its unique cycle,
and \(\pi_{t+2}\) has exactly the two cycles \(C_n,C_{n+1}\).
Thus \(C_{n+1}\) first appears at \(t+2=3(n+1)-1\).
Induction from \(\pi_0\) gives these cycle-creation times, and every policy cycle
during these iterations is one of \(C_0,\ldots,C_{2^{d+1}-2}\).
This proves Lemma~\ref{lem:av-counting}.
\qed